\documentclass[twoside,11pt]{article}

\usepackage{thmtools, thm-restate}
\usepackage{subcaption}

\usepackage{url}
\usepackage[nohyperref]{jmlr2e} 
\usepackage{graphicx}  
\usepackage{amsfonts}       
\usepackage{amsmath}
\usepackage{nicefrac}       
\usepackage{microtype}      
\usepackage{algorithm}
\usepackage{algpseudocode}
\usepackage{graphicx}
\usepackage{xcolor}
\usepackage{xr}
\usepackage{natbib}
\usepackage{bm}
\usepackage{multirow}
\usepackage{tikz}
\usepackage{pgfplots}

\newtheorem{observation}[theorem]{Observation}

\DeclareMathAlphabet{\mymathbb}{U}{BOONDOX-ds}{m}{n}

\newcommand{\exc}[2]{{\mathbb E}\left[ #1 \,\middle \vert\, #2 \right]}
\newcommand{\exs}[2]{{\mathbb E_{#1}}\left[ #2 \right]}
\newcommand{\vars}[2]{{\mathbb V_{#1}}\left[ #2 \right]}
\newcommand{\jacobian}{{\textbf{Jb}^\theta }}
\newcommand{\excs}[3]{{\mathbb E_{#1}}\left[ #2  \,\middle \vert\, #3 \right]}
\algnewcommand{\LineComment}[1]{\State \(\triangleright\) #1}

\newcommand{\atpis}[1]{ \left[ #1 \right]_{a = \pi(s)}}
\newcommand{\atmus}[1]{ \left[ #1 \right]_{a = \mu_s}}
\newcommand{\atpist}[1]{ \left[ #1 \right]_{a = \pi(s_t)}}

\newcommand{\Qv}{ Q^{\pi_{\theta}} }
\newcommand{\Vv}{ V^{\pi_{\theta}} }
\DeclareMathOperator{\trace}{trace}

\newcommand{\est}{ \eta^s_\theta }
\newcommand{\Sigmasqr}[0]{\Sigma^{1/2}}
\DeclareMathOperator*{\argmax}{arg\,max}

\usepackage{lastpage}
\jmlrheading{21}{2020}{1-\pageref{LastPage}}{1/18; Revised
2/20}{2/20}{18-012}{Kamil Ciosek and Shimon Whiteson}
\ShortHeadings{Expected Policy Gradients}{Ciosek and Whiteson}

\firstpageno{1}

\definecolor{ffqqtt}{rgb}{1,0,0.2}
\definecolor{qqwuqq}{rgb}{0,0.39215686274509803,0}
\pgfplotsset{compat=1.14}
\usepackage{mathrsfs}
\usetikzlibrary{arrows}

\begin{document}

\title{Expected Policy Gradients for Reinforcement Learning}

\author{
    \\
    \name Kamil Ciosek\thanks{Work on this paper started when the first author was a post-doc at the Department of Computer Science, University of Oxford.} \email kamil.ciosek@microsoft.com \\
    \addr Microsoft Research Cambridge, \\
    21 Station Road,
    Cambridge CB1 2FB,
    United Kingdom
    \\ 
    \\
    \name Shimon Whiteson \email shimon.whiteson@cs.ox.ac.uk \\
    \addr Department of Computer Science, University of Oxford\\
    Wolfson Building,
    Parks Road,
    Oxford OX1 3QD
    United Kingdom
}

\editor{Jan Peters}

\maketitle

\begin{abstract}
We propose \emph{expected policy gradients} (EPG), which unify stochastic policy gradients (SPG) and deterministic policy gradients (DPG) for reinforcement learning. Inspired by \emph{expected sarsa}, EPG integrates (or sums) across actions when estimating the gradient, instead of relying only on the action in the sampled trajectory. For continuous action spaces, we first derive a practical result for Gaussian policies and quadratic critics and then extend it to a universal analytical method, covering a broad class of actors and critics, including Gaussian, exponential families, and policies with bounded support. For Gaussian policies, we introduce an exploration method that uses covariance proportional to $e^H$, where $H$ is the scaled Hessian of the critic with respect to the actions. For discrete action spaces, we derive a variant of EPG based on softmax policies. We also establish a new \emph{general policy gradient theorem}, of which the stochastic and deterministic policy gradient theorems are special cases. Furthermore, we prove that EPG reduces the variance of the gradient estimates without requiring deterministic policies and with little computational overhead. Finally, we provide an extensive experimental evaluation of EPG and show that it outperforms existing approaches on multiple challenging control domains.
\end{abstract}


\section{Introduction}
In reinforcement learning, an agent aims to learn an optimal behavior policy from trajectories sampled from the environment. In settings where it is feasible to explicitly represent the policy, \emph{policy gradient} methods \citep{sutton2000policy, peters2006policy, peters2008reinforcement, silver2014deterministic}, which optimize policies by gradient ascent, have enjoyed great success, especially with large or continuous action spaces. The archetypal algorithm optimizes an \emph{actor}, i.e., a policy, by  following a policy gradient that is estimated using a \emph{critic}, i.e., a value function.

The policy can be stochastic or deterministic, yielding \emph{stochastic policy gradients} (SPG) \citep{sutton2000policy} or \emph{deterministic policy gradients} (DPG) \citep{silver2014deterministic}. The theory underpinning these methods is quite fragmented, as each approach has a separate policy gradient theorem guaranteeing the policy gradient is unbiased under certain conditions.

Furthermore,  both approaches have significant shortcomings. For SPG,  variance in the gradient estimates means that many trajectories are usually needed for learning.  Since gathering trajectories is typically expensive, there is a great need for more sample efficient methods.

DPG's use of deterministic policies mitigates the problem of variance in the gradient but raises other difficulties.  The theoretical support for DPG is limited since it assumes a critic that approximates $\nabla_a Q$ when in practice it approximates $Q$ instead.  In addition, DPG learns \emph{off-policy},\footnote{We show in this article that, in certain settings, off-policy DPG is equivalent to EPG, our on-policy method.} which means that, unless specifically designed otherwise, it explores in a way that is oblivious to the reward signal. More importantly, learning off-policy necessitates designing a suitable \emph{exploration policy}, which is difficult in practice.  In fact, efficient exploration in DPG is an open problem and most applications simply use independent Gaussian noise or the Ornstein-Uhlenbeck heuristic \citep{uhlenbeck1930theory,lillicrap2015continuous}.

This article proposes a new approach called \emph{expected policy gradients} (EPG) that unifies policy gradients in a way that yields both theoretical and practical insights. Inspired by \emph{expected sarsa} \citep{sutton1998reinforcement,vanseijen:adprl09}, the main idea is to integrate across the action selected by the stochastic policy when estimating the gradient, instead of relying only on the action selected during the sampled trajectory. While the idea of summing over discrete actions and calculating analytic integrals has been proposed previously \citep{sutton2000comparing, bahdanau2016actor, kakade2002natural} and concurrently \citep{2017arXiv170900503A} in some specific settings, EPG is the first method to treat the technique in a unified way for both discrete and continuous action space on top of a single theoretical framework. The detailed differences between EPG and these approaches are given in Section \ref{me-sum-i}.

The contributions of this paper are threefold. First, EPG enables two general theoretical contributions (Section \ref{sec-gpgt}): 1) a new \emph{general policy gradient theorem}, of which the stochastic and deterministic policy gradient theorems are special cases, and 2) a proof that (Section \ref{sec-var-res}) EPG reduces the variance of the gradient estimates without requiring deterministic policies and, for the Gaussian case, with no computational overhead over SPG. Second, we define practical policy gradient methods. For the Gaussian case (Section \ref{sec-gaussian}), the EPG solution is not only analytically tractable but also leads to a successful exploration strategy (Section \ref{ss-hessian-exp}) for continuous problems, with an exploration covariance that is proportional to $e^H$, where $H$ is the scaled Hessian of the critic with respect to the actions. In Section \ref{sec-discrete-epg}, we derive a version of EPG for discrete control problems. We present empirical results (Section \ref{sec-gpg-exp}) confirming that this new approach to exploration substantially outperforms DPG with Ornstein-Uhlenbeck exploration in  MuJoCo continuous control tasks. Third, we provide a way of deriving tractable EPG methods for the general case of policies coming from a certain exponential family (Section \ref{sec-expfamily}) and for critics that can be reparameterized as polynomials, thus yielding analytic EPG solutions that are tractable for a broad class of problems and essentially making EPG a universal method.  Finally, in Section \ref{sec-rel}, we relate EPG to other RL approaches, including entropy-based methods and value gradient methods.

This paper is a revised and extended version of our AAAI conference submission \citep{epg-aaai}. On the theoretical side, we have added an analysis of softmax (Gibbs) policies. For continuous actions, we have analyzed a more general policy class (exponential families) and a critic class general enough to approximate any function (polynomials). We provide an analysis of the off-policy version of EPG. We also compare EPG with methods that adapt to the geometry of the policy space, entropy-based methods and value gradients. In addition, we have greatly expanded the experimental section, which now includes a comparison to reparameterized policy gradients and several ablations, in addition to results about EPG with numerical quadrature.

\section{Background}
\label{sec:bg}
A \emph{Markov decision process} \citep{puterman2014markov} is a tuple $(S,A,R_d,p,p_0,\gamma)$ where $S$ is a set of states, $A$ is a set of actions (in practice either $A=\mathbb{R}^d$ or $A$ is finite), $R_d(a, s)$ is a reward distribution (we introduce the notation $R(a,s) = \excs{R_d}{r}{a,s}$ for the mean state-action reward), $p(s' \mid a,s)$ is a transition kernel, $p_0$ is an initial state distribution, and $\gamma \in [0,1)$ is a discount factor. A policy $\pi(a \mid s)$ is a distribution over actions given a state. We denote trajectories as $\tau^\pi = (s_0,a_0,r_0,s_1,a_1,r_1,\dots)$, where $s_0 \sim p_0$, $a_t \sim \pi(\cdot \mid s_{t})$, $s_{t+1} \sim p(\cdot\mid s_t,a_t)$ and $r_t$ is the sampled reward. A policy $\pi$ induces a Markov process with transition kernel $p_\pi(s' \mid s) = \int_a d \pi (a \mid s) p(s' \mid a, s)$ where we use the symbol $d \pi (a \mid s)$ to denote Lebesgue integration against the measure $\pi (a \mid s)$ where $s$ is fixed. We assume the induced Markov process is ergodic with a single invariant measure defined for the whole state space.  The value function is $V^\pi(s) = \exs{\tau: s_0 = s}{\sum_{i=0}^\infty \gamma_i r_i}$ where actions are sampled from $\pi$. The $Q$-function is $Q^\pi(a, s) = \excs{R_d}{r}{a,s} + \gamma \excs{p_(s' \mid a, s)}{V^\pi(s')}{s}$ and the advantage function is $A^\pi(a \mid s) = Q^\pi(a, s) - V^\pi(s)$. An optimal policy maximizes the total return $J = \int_s dp_0(s) V^\pi(s)$. Since we consider only on-policy learning with just one current policy, we drop the $\pi$ super/subscript where it is redundant.

If $\pi$ is parametrized by $\theta$, then \emph{stochastic policy gradients} (SPG) \citep{sutton2000policy, peters2006policy, peters2008reinforcement} perform gradient ascent on $\nabla_\theta  J$, the gradient of $J$ with respect to $\theta$. For stochastic policies, we have
\begin{gather*}
\label{spg-update} \textstyle
\nabla_\theta  J = \int_s d \rho(s) \int_a d \pi(a \mid s) \nabla_\theta  \log \pi(a \mid s) (\Qv(a,s) + b(s)),
\end{gather*}
where $\rho$ is the discounted-ergodic occupancy measure, defined in the Appendix, and $b(s)$ is a baseline,\footnote{See also the work on action-dependent baselines by \citet{thomasPolicyGradientMethods2017} and \citet{wuVarianceReductionPolicy2018a}.} which can be any function that depends on the state but not the action, since $\int_a d \pi(a \mid s) \nabla_\theta  \log \pi(a \mid s) b(s) = 0$. Typically, because of ergodicity and Lemma \ref{lem-d-ergoidic} (see Appendix), we can approximate \eqref{spg-update} from samples from a trajectory $\tau$ of length $T$, giving
\begin{gather*}
\label{spg-samples}\textstyle
\hat{\nabla}_\theta   J = \sum_{t=0}^{T} \gamma^t \nabla_\theta  \log \pi(a_t \mid s_t) (\hat{Q}(a_t, s_t) + b(s_t)).
\end{gather*}
Here, $\hat{Q}(a_t, s_t)$ is a critic, discussed below. If the policy is deterministic (we denote it $\pi(s)$), we can use \emph{deterministic policy gradients} \citep{silver2014deterministic} instead, so that the gradient becomes
\begin{gather*}
    \label{dpg-update}\textstyle
    \nabla_\theta  J = \int_s d \rho(s) \nabla_\theta  \pi(s) \atpis{\nabla_a Q (a, s)}.
\end{gather*}
This update is then approximated using samples, giving the estimate
\begin{gather*}
    \label{dpg-samples}\textstyle
    \hat{\nabla}_\theta   J = \sum_{t=0}^{T} \gamma^t \nabla_\theta  \pi(s) \atpist{\nabla_a \hat{Q} (a, s_t)}.
\end{gather*}
Since the policy is deterministic, the problem of exploration is addressed using an external source of noise, typically modeled using a zero-mean Ornstein-Uhlenbeck (OU) process \citep{uhlenbeck1930theory, lillicrap2015continuous} parameterized by $\psi$ and $\sigma$ and generated as
\begin{gather*}
    \label{ou-noise}
    n_t \leftarrow - n_{t-1} \psi + \mathcal{N}(0,\sigma I) \quad \text{and} \quad a_t \sim \pi(s_t) + n_t.
\end{gather*}
In \eqref{spg-samples} and \eqref{dpg-samples}, $\hat{Q}$ is a \emph{critic} that approximates $Q$ and can be learned by \emph{sarsa} \citep{rummery1994line, sutton1996generalization}, using the update
\begin{align*}
\hat{Q}(a_t, s_t) \leftarrow &\hat{Q}(a_t, s_t) \; + \textstyle \; \alpha \big[ r_{t+1} + \gamma \hat{Q}(s_{t+1}, a_{t+1}) - \hat{Q}(a_t, s_t) \big].
\end{align*}

Alternatively, we can use \emph{expected sarsa} \citep{sutton1998reinforcement,vanseijen:adprl09}, which marginalizes out $a_{t+1}$, the distribution over which is specified by the known policy, to reduce the variance, using the update
\begin{align*}
\textstyle
\hat{Q}(&a_t, s_t) \leftarrow \hat{Q}(a_t, s_t) \; + \textstyle \; \alpha \big[ r_{t+1} + \gamma \int_a d \pi(a \mid s) \hat{Q}(a, s_{t+1}) - \hat{Q}(a_t, s_t) \big].
\end{align*}
Instead, we could also use advantage learning \citep{baird1995residual} or LSTDQ \citep{lagoudakis2003least}. This update also has the advantage that actions $a_t$ sampled from a policy different than $\pi$ can be used during learning, giving an off-policy algorithm. 

The theory behind policy gradient methods says that the actor follows the gradient of the total discounted return $J$ \citep{sutton2000policy} and therefore finds its local maximum if the critic's function approximator is \emph{compatible}.\footnote{This holds under the idealized setting where the critic is run to convergence and minimizes a weighted $L_2$-loss. In practice, the critic is not run until convergence and, even if it were, it is not guaranteed to minimize the weighted $L_2$ loss. Compatible critics are rarely used in practice.}

Instead of learning $\hat{Q}$, we can set $b(s) = - \Vv(s)$ so that $\Qv(a,s) + b(s) = A(a,s)$ and then use the TD error $\delta(r,s',s) = r + \gamma \Vv(s') - \Vv(s)$ as an estimate of $A(a,s)$ \citep{bhatnagar2008incremental}. The policy gradient estimate then becomes
\begin{gather}
\label{spg-samples-td} \textstyle
\hat{\nabla}_\theta   J = \sum_{t=0}^{T} \gamma^t \nabla_\theta  \log \pi(a_t \mid s_t) (r + \gamma \hat{V}(s_{t+1}) - \hat{V}(s_t)),
\end{gather}
where $\hat{V}(s)$ is an approximate value function learned using any policy evaluation algorithm. Equation \eqref{spg-samples-td} works because $\exc{\delta(r,s',s)}{a,s} = A(a,s)$, i.e., the TD error is an unbiased estimate of the advantage function. The benefit of this approach is that it is sometimes easier to approximate $V$ than $Q$ and that the return in the TD error is unprojected, i.e., it is not distorted by function approximation. However, for stochastic MDPs, the TD error is noisy, introducing variance in the gradient.

To cope with this variance, we can use an optimizer that adaptively reduces the learning rate when the variance of the gradient is high, using, e.g., \emph{Adam} \citep{kingma2014adam} or RMSprop \citep{tielemanH12}. However, this results in slow learning when the variance is high. We discuss other variance reduction techniques in Section \ref{sec-rel}.

\section{Expected Policy Gradients}
\label{sec-epg-main}
In this section, we propose \emph{expected policy gradients} (EPG). First, we introduce $I^Q_\pi(s)$ to denote the inner integral in \eqref{spg-update}, rewriting the policy gradient as
\begin{align}
\label{epg-main-eq}
\nabla_\theta  J & = \int_s d\rho(s) \underbrace{\int_a d \pi(a \mid s) \nabla_\theta  \log \pi(a \mid s) (\Qv(a,s) + b(s))}_{I^Q_\pi(s)} \nonumber \\
& = \int_s d\rho(s) \int_a d \pi(a \mid s) \nabla_\theta  \log \pi(a \mid s) \Qv(a,s) \nonumber \\
& = \int_s d \rho(s) I^Q_\pi(s).
\end{align}
This suggests a new way to write the approximate gradient, giving\footnote{The idea behind EPG was also independently and concurrently developed as Mean Actor Critic \citep{2017arXiv170900503A}, though only for discrete actions and without a supporting theoretical analysis.} 
\begin{gather*}
\label{epg-samples-i}
\hat{\nabla}_\theta  J = \sum_{t=0}^{T} \underbrace{\gamma^t I^{\hat{Q}}_\pi(s_t)}_{g_t}, \quad \text{where} \quad I^{\hat{Q}}_\pi(s) = \int_a d \pi(a \mid s) \nabla_\theta  \log \pi(a \mid s) \hat{Q}(a, s).
\end{gather*}
Here, we used Lemma \ref{lem-d-ergoidic} in the Appendix to sample from $\rho(s)$. This approach makes explicit that one step in estimating the gradient is to evaluate an integral included in the term $I^{\hat{Q}}_\pi(s)$. The main insight behind EPG is that, given a state, $I^{\hat{Q}}_\pi(s)$ is expressed fully in terms of known quantities. Hence we can manipulate it analytically to obtain a formula or we can just compute the integral using numerical quadrature if an analytical solution is impossible (in Section \ref{sec-universal} we show that this is rare). For a discrete action space, $I^{\hat{Q}}_\pi(s_t)$ becomes a sum over actions (see Section \ref{sec-discrete-epg} for more details).

SPG as given in \eqref{spg-samples} performs this quadrature using a simple one-sample Monte Carlo method, using the action $a_t \sim \pi(\cdot \mid s_t)$. It uses the update
\[
I^{\hat{Q}}_\pi(s) = \int_a d \pi(a \mid s) \nabla_\theta  \log \pi(a \mid s) \hat{Q}(a, s) \approx
\nabla_\theta  \log \pi(a_t \mid s_t) (\hat{Q}(a_t, s_t) + b(s_t)).
\]
Moreover, SPG assumes that the action $a_t$ used in the above estimation is the same action that is executed in the environment. However, relying on such a method is unnecessary. In fact, the actions used to interact with the environment need not be used at all in the evaluation of $\hat{I}_\pi^Q(s)$ since $a$ is a bound variable in the definition of $I_\pi^Q(s)$. The motivation is thus similar to that of expected sarsa but applied to the actor's gradient estimate instead of the critic's update rule. EPG, shown in Algorithm \ref{alg-epg}, uses \eqref{epg-samples-i} to form a policy gradient algorithm that repeatedly estimates $\hat{I}_\pi^Q(s)$ with an integration subroutine.

\begin{algorithm}[ht]
\begin{algorithmic}[1]
 \State $s \gets s_0$, $t \gets 0$
 \State initialize optimizer, initialize policy $\pi$ parameterized by $\theta$
\While{not converged}
 \label{ln-a-i} \State $g_t \gets \gamma^t$ \textsc{do-integral}($\hat{Q}, s, \pi_\theta $) \Comment{$g_t$ is the estimated policy gradient as per \eqref{epg-samples-i}}
 \State $\theta \gets  \theta \; + \;  $optimizer.\textsc{update}$(g_t) $
 \State $a \sim \pi(\cdot \mid s)$
 \State $s',r \gets $ simulator.\textsc{perform-action}(a)
 \State $\hat{Q}$.\textsc{update}($s,a,r,s'$)
 \State $t \gets t + 1$
 \State $s \gets s'$
\EndWhile

\end{algorithmic}
\caption{Expected policy gradients} \label{alg-epg}
\end{algorithm}

One of the motivations of DPG was precisely that the simple one-sample Monte-Carlo quadrature implicitly used by SPG often yields high variance gradient estimates, even with a good baseline. To see why, consider the setting in Figure \ref{fig-mc-p}, where we use the parametrization $\theta = \mu$. On the left, a simple Monte Carlo method evaluates the integral by sampling one or more times from $\pi(a \mid s)$ (blue) and evaluating $\nabla_\mu \log \pi(a \mid s) \Qv(a,s)$ (red) as a function of $a$. A baseline can decrease the variance by adding a multiple of $\nabla_\mu \log \pi(a \mid s)$ to the red curve. However, whatever the baseline, substantial variance persists, even with a simple linear $Q$-function, as shown in Figure \ref{fig-mc-p} (right).  DPG addressed this problem for deterministic policies but EPG extends it to stochastic ones. We show in Section \ref{sec-expfamily} that an analytical EPG solution, and thus the corresponding reduction in the variance, is possible for a wide array of critics. We also discuss  the rare case where numerical quadrature is necessary in Section \ref{ss-numq}.

\begin{figure}
    \centering
        \begin{subfigure}{0.42\textwidth}
            \includegraphics[width=\textwidth]{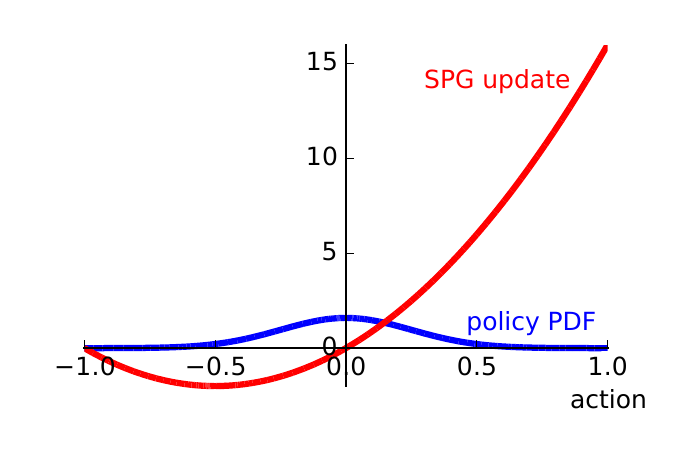}
        \end{subfigure}
        ~
        \begin{subfigure}{0.42\textwidth}
            \includegraphics[width=\textwidth]{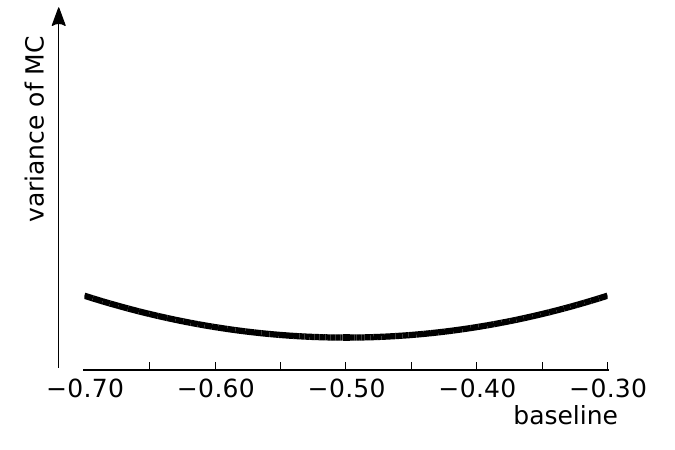}
        \end{subfigure}
        \caption{At left, $\pi(a \mid s)$ for a Gaussian policy with mean $\mu = \theta = 0$ at a given state and constant $\sigma^2$ (blue) and $\nabla_\theta \log \pi(a \mid s) \Qv(a,s)$ for $Q = \frac12 + \frac12 a$ (red). At right, the variance of a simple single-sample Monte Carlo estimator as a function of the baseline.  In a simple multi-sample Monte Carlo method, the variance would go down as the number of samples.}
        \label{fig-mc-p}
    \end{figure}

\subsection{General Policy Gradient Theorem}
\label{sec-gpgt}
We begin by stating our most general result, showing that EPG can be seen as a generalization of both SPG and DPG. To do this, we first state a new general policy gradient theorem. 
\begin{theorem}[General Policy Gradient Theorem]\ \\
If the value function $V(s)$ is bounded, continuously differentiable in the policy parameters $\theta$ and measurable in $s$ then
\begin{gather*}
\label{eq-gpgt-main}
\nabla_\theta  J = \int_s d\rho(s) \underbrace{ \left[ \nabla_\theta  \Vv(s) - \int_a d \pi(a \vert s) \nabla_\theta  \Qv(a,s) \right]}_{I_G(s)} = \int_s d\rho(s) I_G(s).
\end{gather*}
\label{th-gpgt}
\end{theorem}
\begin{proof}
We start with the expression on the left and begin expanding. We have
\begin{align*}
  \textstyle \int_s \textstyle d\rho(s) \int_a & d\pi(a \vert s)   \nabla_\theta  \Qv(a,s) \\
& \textstyle = \int_s d\rho(s) \int_a d\pi(a \vert s) \nabla_\theta  (R(a,s) + \gamma \int_{s'} d p(s' \mid a, s) \Vv(s'))  \\
&\textstyle = \int_s d \rho(s) \int_a d \pi(a \vert s)  (\underbrace{\textstyle \nabla_\theta  R(a,s)}_{\textstyle0} + \gamma \int_{s'} d p(s' \mid a, s) \nabla_\theta  \Vv(s')) \\
&\textstyle = \gamma \int_s d \rho(s) \int_{s'} d p_\pi(s' \mid s) \nabla_\theta  \Vv(s')) \\
&\textstyle = \int_s d \rho(s) \nabla_\theta  \Vv(s) -  \underbrace{\textstyle \int_s d p_0(s) \nabla_\theta  \Vv(s)} _{\nabla_\theta  J} \\
&\textstyle = \int_s d \rho(s) \nabla_\theta  \Vv(s) -  \nabla_\theta  J.
\end{align*}
In the above, we used the notation $p_\pi(s' \mid s) = \int_a d \pi (a \mid s) p(s' \mid a, s)$. The first equality follows by expanding the definition of $Q$ and the penultimate one follows from Lemma \ref{gep-lemma} in the Appendix. Then the theorem follows by rearranging terms.
\end{proof}
The crucial benefit of Theorem \ref{th-gpgt} is that it works for all policies, both stochastic and deterministic, unifying previously separate derivations for the two settings. To show this, in the following two corollaries, we use Theorem \ref{th-gpgt} to recover the \emph{stochastic policy gradient theorem} \citep{sutton2000policy}  and the \emph{deterministic policy gradient theorem} \citep{silver2014deterministic}, in each case by introducing additional assumptions to obtain a formula for $I_G(s)$ expressible in terms of known quantities.
\begin{corollary}[Stochastic Policy Gradient Theorem]
\label{cor-spg}
If $\pi(a \mid s)$, considered as a probability density function, is continuous in $s$ and continuously differentiable in $\theta$ and if $R(a,s)$ is continuous in $s$ and bounded then
\begin{align*} \textstyle
\nabla_\theta  J &= \textstyle \int_s d \rho(s) I_G(s) = \textstyle \int_s d \rho(s) \int_a d \pi(a \mid s) \nabla_\theta  \log \pi(a \mid s) \Qv(a,s).
\end{align*}
\end{corollary}
\begin{proof}
We expand $\nabla_\theta  V $, obtaining
\begin{align}
\nabla_\theta  V &= \nabla_\theta  \textstyle \int_a d \pi(a \vert s) \Qv(a,s) = \textstyle \int_a da (\nabla_\theta  \pi(a \vert s)) \Qv(a,s) +  \int_a d \pi(a \vert s) (\nabla_\theta  \Qv(a,s)).
\label{nabla-exp}
\end{align}
We obtain $I_G(s) = \int_a d \pi(a \mid s) \nabla_\theta  \log \pi(a \mid s) \Qv(a,s) = I^Q_\pi(s)$ by plugging \eqref{nabla-exp} into the definition of $I_G(s)$ as given in \eqref{eq-gpgt-main}. We obtain $\nabla_\theta  J$ by invoking  Theorem \ref{th-gpgt}, plugging in the above expression for $I_G(s)$ and observing that the regularity conditions on $V$ follow from the regularity conditions on $\pi(\cdot \mid s)$ and $R(a,s)$.
\end{proof}
 
\begin{corollary}[Deterministic Policy Gradient Theorem]
\label{cor-dpg}
If $\pi(\cdot \mid s)$ is a Dirac-delta measure, and $\pi(s)$ is continuously differentiable in $\theta$ and continuous in $s$, if $R(a,s)$ is continuous in $s$, differentiable in $a$ and bounded and the transition kernel $p(s' \mid a, s)$ is continuous in $s$ and differentiable in $a$, then
\[
\textstyle \nabla_\theta J = \int_s d \rho(s) I_G(s) =  \int_s d \rho(s) \nabla_\theta \pi(s) \atpis{\nabla_a \Qv(a,s)}.
\]
Here, we overload the notation of $\pi$ slightly. We denote by $\pi(s)$ the action taken at state $s$, i.e. $\pi(s) = \int_a a d\pi(a \mid s) $, where $\pi(\cdot \mid s)$ is the corresponding Dirac-delta measure.
\end{corollary}
\begin{proof}
We begin by expanding the term for $\nabla_\theta \Vv(s)$, which will be useful later on. We have
\begin{gather*}
\label{eq-dpg-chain}
\nabla_\theta \Vv(s) = \nabla_\theta Q(\pi(s), s) = \atpis{\nabla_\theta \Qv(a,s)} + \nabla_\theta \pi(s) \atpis{\nabla_a \Qv(a,s)}. 
\end{gather*}
The above results from applying the multivariate chain rule---observe that both $\pi(s)$ and $\Qv(a,s)$ depend on the policy parameters $\theta$; hence, the dependency appears twice in $Q(\pi(s), s)$.

We proceed to obtain an expression for $I_G(s)$. We have
\begin{align*}
I_G(s) &= \textstyle \nabla_\theta  \Vv(s) - \int_a d \pi(a \vert s) \nabla_\theta  \Qv(a,s) \\
&= \textstyle \nabla_\theta  \Vv(s) - \atpis{\nabla_\theta  \Qv(a,s)} \\
&= \textstyle \nabla_\theta  \pi(s) \atpis{\nabla_a \Qv(a,s)}.
\end{align*}
Here, the second equality follows by observing that the policy is a Dirac-delta and the third one follows from using \eqref{eq-dpg-chain}. We can then obtain $\nabla_\theta  J$ by invoking Theorem \ref{th-gpgt} and plugging in the above expression for $I_G(s)$. The regularity conditions on $V$ follow from the regularity conditions on $\pi(s)$, $R(a,s)$ and $p(s' \mid a, s) $. 
\end{proof}

These corollaries show that the choice between deterministic and stochastic policy gradients is fundamentally a choice of quadrature method. Hence, the empirical success of DPG relative to SPG \citep{silver2014deterministic, lillicrap2015continuous} can be understood in a new light. In particular, it can be attributed, not to a fundamental limitation of stochastic policies (indeed, stochastic policies are sometimes preferred), but instead to superior quadrature. DPG integrates over Dirac-delta measures, which is known to be easy, while SPG typically relies on simple Monte Carlo integration. Thanks to EPG, a deterministic approach is no longer required to obtain a method with low variance.

\subsection{Variance Analysis}
\label{sec-var-res}
In stochastic policy gradients, there are two reasons why the actor update can be inaccurate \citep{guInterpolatedPolicyGradient2017}. First, using an approximation in place of the original quadrature means that the integral estimate is stochastic and has nonzero variance. In, particular, the value of the integral  
\[ 
    \int_a d \pi(a \mid s) \nabla_\theta  \log \pi(a \mid s) (\hat{Q}(a,s) + b(s))
\]
is different from its one-sample Monte-Carlo approximation
\[
    \nabla_\theta  \log \pi(a_t \mid s_t) (\hat{Q}(a_t, s_t) + b(s_t)), \quad \text{where} \quad a_t \sim \pi(\cdot \mid s_t).
\]
Second, the learned critic value $\hat{Q}$ itself can be inaccurate -- if we use a neural net to learn $\hat{Q}$, it will have some approximation error, while if we use an advantage estimate as in  \eqref{spg-samples-td}, there will be additional variance which comes from this advantage estimate. 

We now prove that EPG completely eliminates the first kind of variance. In particular, for any policy, the EPG estimator of \eqref{epg-samples-i} has lower variance than the SPG estimator of \eqref{spg-samples}.

\begin{lemma}
If for all $s \in S$, the random variable $\nabla_\theta  \log \pi(a \mid s) \hat{Q}(a, s)$ where $a \sim \pi(\cdot \vert s)$ has nonzero variance, then
\begin{gather*} \textstyle
\textstyle \vars{\tau}{\textstyle \sum_{t=0}^{\infty} \gamma^t \nabla_\theta  \log \pi(a_t \mid s_t) (\hat{Q}(a_t, s_t) + b(s_t))} \textstyle >
\textstyle \vars{\tau}{  \sum_{t=0}^{\infty} \gamma^t I^{\hat{Q}}_\pi(s_t)  } .
\end{gather*}
\label{lem-var}
\end{lemma}
\begin{proof}
Both random variables have the same mean so we need only show that
\begin{gather*}
\textstyle \exs{\tau}{\left (\sum_{t=0}^{\infty} \gamma^t \nabla_\theta  \log \pi(a_t \mid s_t) (\hat{Q}(a_t, s_t) + b(s_t)) \right)^2} > \textstyle \exs{\tau}{  \left ( \sum_{t=0}^{\infty} \gamma^t I^{\hat{Q}}_\pi(s_t) \right)^2  }.
\end{gather*}
We start by applying Lemma \ref{smbe-lemma} to the left-hand side and setting
\[
X = X_1(s_t) = \nabla_\theta  \log \pi(a_t \mid s_t) (\hat{Q}(a_t, s_t) + b(s_t))
\] where
$a_t \sim \pi(a_t \vert s_t)$. This shows that
\[\textstyle
\exs{\tau}{\textstyle \left (\sum_{t=0}^{\infty} \gamma^t \nabla_\theta  \log \pi(a_t \mid s_t) (\hat{Q}(a_t, s_t) + b(s_t)) \right)^2}
\]
is the total return of the Markov reward process (MRP)\footnote{See Definition \ref{def-MRP} in the Appendix. } $(p, p_0, u_1, \gamma^2)$, where
\begin{gather*}
u_1(s) =
\vars{X_1( x \mid s)}{x} +
\left( \exs{X_1( x \mid s)}{x}\right)^2 +
2 \gamma \exs{X_1( x \mid s)}{x}  \exs{p( s' \mid s)} {\Vv(s)')}.
\end{gather*}
Likewise, applying Lemma \ref{smbe-lemma} again to the right-hand side, instantiating $X$ as a deterministic random variable $X_2(s_t) = I^{\hat{Q}}_\pi(s_t) $, we have that $\exs{\tau}{  \sum_{t=0}^{\infty} \left ( \gamma^t I^{\hat{Q}}_\pi(s_t) \right)^2  }$ is the total return of the MRP $(p, p_0, u_2, \gamma^2)$, where
\[
u_2(s) = \left(\exs{X_2( x \mid s)}{x}\right)^2 + 2 \gamma \exs{X_2( x \mid s)}{x} \exs{p( s' \mid s)}{\Vv(s)')}.
\]
Note that $\exs{X_1( x \mid s)}{x} = \exs{X_2( x \mid s)}{x}$ and therefore $u_1(s) \geq u_2(s)$ for all states $s$. Furthermore, by assumption of the lemma, the inequality is strict.  The lemma then follows by applying Observation \ref{dom-o}. 
\end{proof}

For convenience, Lemma \ref{lem-var} also assumes infinite length trajectories.  However, this is not a practical limitation since all policy gradient methods implicitly assume trajectories are long enough to be modeled as infinite.  Furthermore, a finite trajectory variant also holds, though the proof is messier.

Lemma \ref{lem-var}'s assumption is reasonable since the only way a random variable \[
\nabla_\theta~\log~\pi(a~\mid~s)~\hat{Q}(a, s)
\] 
could have zero variance is if it were the same for all actions in the policy's support (except for sets of measure zero), in which case optimizing the policy would be unnecessary. Since we know that both the estimators of  \eqref{spg-samples} and \eqref{epg-samples-i} are unbiased,\footnote{They provide an unbiased estimate of an integral that includes $\hat{Q}$. Of course the gradient can still be biased if $\hat{Q}$ itself is biased.} the estimator with lower variance has lower MSE. Moreover, we observe that Lemma \ref{lem-var} holds for the case where the computation of $I^{\hat{Q}}_\pi$ is exact. Section \ref{sec-expfamily} shows that this is often possible.

\section{Expected policy gradients for Gaussian Policies}
\label{sec-gaussian}
EPG is particularly useful when we make the common assumption of a Gaussian policy: we can then perform the integration analytically under reasonable conditions. We show below (see Corollary \ref{lem-agpg}) that the update to the policy mean computed by EPG is equivalent to the DPG update. Moreover, we derive a simple formula for the covariance (see Lemma \ref{lem-hexplore}). Algorithms \ref{alg-epg-fixpoint} and \ref{alg-p-quad} show the resulting special case of EPG, which we call \emph{Gaussian policy gradients} (GPG).

\begin{algorithm}[ht]
    \begin{algorithmic}[1]
     \State $s \gets s_0$, $t \gets 0$
 \State initialize optimizer
    \While{not converged}
     \State $g_t \gets \gamma^t$ \textsc{do-integral-Gauss}($\hat{Q}, s, \pi_\theta $)
     \State $\theta \gets  \theta \; + \;  $optimiser.\textsc{update}$(g_t) $ \Comment{policy parameters $\theta$ are updated using gradient}
     \State $\Sigma^{1/2}_s \gets $  \textsc{get-covariance}($\hat{Q}, s, \pi_\theta $) \Comment{$\Sigma^{1/2}_s$ computed from scratch}
     \State $a \sim \pi(\cdot \mid s)$ \Comment {$\pi(\cdot \mid s)= N(\mu_s, \Sigma_s)$}
     \State $s',r \gets $ simulator.\textsc{perform-action}(a)
     \State $\hat{Q}$.\textsc{update}($s,a,r,s'$)
     \State $t \gets t + 1$
     \State $s \gets s'$
    \EndWhile

    \end{algorithmic}
    \caption{Gaussian policy gradients} \label{alg-epg-fixpoint}
\end{algorithm}
\begin{algorithm}[ht]
    \begin{algorithmic}[1]
    \Function{do-integral-Gauss}{$\hat{Q}, s, \pi_\theta $}
    \State $I^{\hat{Q}}_{\pi(s), \mu_s} \leftarrow (\nabla_\theta  \mu_s) \nabla_a \hat{Q}(a = \mu_s,s)$ \Comment{Use Lemma \ref{lem-gpg} }
    \State \Return $I^{\hat{Q}}_{\pi(s), \mu_s}$
    \EndFunction
    \\
    \Function{get-covariance}{$\hat{Q}, s, \pi_\theta $}
    \State $H \leftarrow$ \textsc{compute-Hessian}($\hat{Q}(\mu_s,s)$)
    \State \Return $\sigma_0 e^{cH}$ \Comment{Use Lemma \ref{lem-hexplore}}
    \EndFunction
\end{algorithmic}
\caption{Gaussian integrals} \label{alg-p-quad}
\end{algorithm}

Surprisingly, GPG is on-policy but nonetheless fully equivalent to DPG, an off-policy method, with a particular form of exploration.  Hence, GPG, by specifying the policy's covariance, can be seen as a derivation of an exploration strategy for DPG. In this way, GPG addresses an important open question. As we show in Section \ref{sec-gpg-exp}, this leads to improved performance in practice.

The computational cost of GPG is small: while it must store a Hessian matrix $H(a, s) = \nabla^2_a \hat{Q}(a,s)$, its size is only $d \times d$, where $A=\mathbb{R}^d$, which is typically small, e.g., $d=6$ for HalfCheetah, one of the MuJoCo tasks we use for our experiments in Section \ref{sec-gpg-exp}. This Hessian is the same size as the policy's covariance matrix, which any policy gradient must store anyway, and should not be confused with the Hessian with respect to the parameters of the neural network, as used with Newton's or natural gradient methods \citep{peters2008natural, furmston2016approximate}, which can easily have thousands of entries. Hence, GPG obtains EPG's variance reduction essentially for free.

\subsection{Analytical Quadrature for Gaussian Policies}
We now derive a lemma supporting GPG.
\begin{lemma}[Gaussian Policy Gradients]
    \label{lem-gpg}
    For Gaussian policies, i.e. $\pi(\cdot \vert s) \sim \mathcal{N}(\mu_s,\Sigma_s)$ with $\mu_s$ and $\Sigmasqr_s$ parametrized by $\theta$, where $\Sigmasqr_s$ is symmetric, $\Sigmasqr_s \Sigmasqr_s = \Sigma_s$ and the critic is of the form $\hat{Q}(a,s) = a^\top A(s) a + a^\top B(s) + \text{const}$ where $A(s)$ is symmetric for every $s$, then $I^{\hat{Q}}_\pi(s) =  I^{\hat{Q}}_{\pi(s), \mu_s} + I^{\hat{Q}}_{\pi(s), \Sigmasqr_s}$, where the mean and covariance components are given by 
    \begin{align}
    I^{\hat{Q}}_{\pi(s), \mu_s} &= (\nabla_\theta \mu_s) ( 2 A(s) \mu + B(s)) \quad \text{and} \quad \nonumber \\ I^{\hat{Q}}_{\pi(s), \Sigmasqr_s} &= (\nabla_\theta \Sigmasqr_s) 2 A(s) \Sigmasqr_s.
    \label{eq-cov-gpg} 
    \end{align}
    \end{lemma}
    \begin{proof}
    
    First, we observe that the critic ${\hat{Q}}$ defined in the statement of the lemma does not depend on the policy parameters $\theta$ because ${\hat{Q}}$ is an approximation to the $Q$-function maintained by the algorithm, as opposed to the true $Q$-function, which is defined with respect to the policy and does depend on it.

    We can hence move the differentiation outside of the integral, obtaining
    \[
        I^{\hat{Q}}_\pi(s) = \nabla_\theta \int_a \pi(a \vert s) {\hat{Q}}(a,s) da = \nabla_\theta \exs{\pi}{{\hat{Q}}(a,s)}. 
    \]

    We now expand the expectation using the expression 
    $
        \exs{ \pi }{ {\hat{Q}}(a,s) } = \trace (A(s) \Sigma ) + \mu^\top A(s) \mu + B(s)^\top \mu
    $
    for the expectation of a quadratic form. This yields the derivatives
    \begin{align*}
        \nabla_{\Sigmasqr} \exs{\pi}{Q(a,s)} &= \nabla_{\Sigmasqr} (\trace (A(s) \Sigma ) + \mu^\top A(s) \mu + B(s)^\top \mu) = 2 A(s) \Sigmasqr \;\;\;\; \text{and} \\
        \nabla_\mu \exs{\pi}{{\hat{Q}}(a,s)} &= \nabla_\mu (\trace (A(s) \Sigma ) + \mu^\top A(s) \mu + B(s)^\top \mu) = 2 A(s) \mu + B(s).
    \end{align*}

    We now obtain the result by applying the chain rule, giving
    \[
        I^{\hat{Q}}_\pi(s) = I^{\hat{Q}}_{\pi(s), \mu_s} + I^{\hat{Q}}_{\pi(s), \Sigmasqr_s} = 
        (\nabla_\theta \mu) (2 A(s) \mu + B(s)) + (\nabla_\theta \Sigmasqr) (2 A(s) \Sigmasqr).
    \] 

\end{proof}

While Lemma \ref{lem-gpg} requires the critic to be quadratic in the actions, this assumption is not very restrictive since the coefficients $B(s)$ and $A(s)$ can be arbitrary continuous functions of the state, e.g., a neural network.

\subsection{Exploration using the Hessian}
\label{ss-hessian-exp}

Equation \eqref{eq-cov-gpg} suggests that we can include the covariance in the actor network and learn it along with the mean, using the update rule
\begin{gather*}
    \label{eq-cov-upd}
    \Sigma^{1/2}_s \leftarrow \Sigma^{1/2}_s + \alpha H(s) \Sigma^{1/2}_s .
\end{gather*}
In practice, this update has two disadvantages. First, our policy network must include outputs for the covariance. Second, we must constrain our covariance matrix to be positive semi-definite, which is not trivial in a deep learning setup without influencing the gradient norm in a way that slows learning. Instead, we sidestep the issue by not performing incremental covariance updates at all. Instead, we analytically compute the matrix that would have been obtained if we ran the update \eqref{eq-cov-upd} for long enough. 

The following lemma derives the covariance from scratch at each iteration by analytically computing the result of applying \eqref{eq-cov-upd} infinitely many times.
\begin{lemma}[Exploration Limit]
\label{lem-hexplore}
For any fixed total number of iterations $n$ and the same constant learning rate $\alpha = 1/n$ applied at every iteration, the iterative procedure defined by \eqref{eq-cov-upd} applied $n$ times yields $(\Sigma^{1/2}_s)_{n+1} = U (I + \frac1n \Lambda)^n U^\top \sigma_0$. Moreover, the limit as $n \rightarrow \infty$, is $\Sigma^{1/2}_s \propto e^{H(s)}$.

\end{lemma}
\begin{proof}
First, let $n$ be fixed. Consider the sequence $(\Sigma^{1/2}_s)_1 = \sigma_0 I$, $(\Sigma^{1/2}_s)_n = (\Sigma^{1/2}_s)_{n-1} + \frac1n H(s) (\Sigma^{1/2}_s)_{n-1} $. We diagonalize the Hessian as $H(s) = U \Lambda U^\top$ for some orthonormal matrix $U$ and obtain the following expression for the $n$-th element of the sequence
\[
    (\Sigma^{1/2}_s)_{n+1} = \left(I + \frac1n H(s) \right)^n \sigma_0 = U \left(I + \frac1n \Lambda \right)^n U^\top \sigma_0.
\]
We now let $n \rightarrow \infty$. Since we have $\lim_{n \rightarrow \infty} (1 + \frac1n \lambda)^n = e^\lambda$ for each eigenvalue of the Hessian, we obtain the identity
\[
    \lim_{n \rightarrow \infty} U \left (I + \frac1n \Lambda \right)^n U^\top \sigma_0 = \sigma_0 e^{H(s)}.
\]
\end{proof}

The practical implication of  Lemma \ref{lem-hexplore} is that, in a policy gradient method, it is justified\footnote{Lemma \ref{lem-hexplore} relies crucially on the use of special constant step sizes that diminish as we consider longer and longer trajectories. This step sequence serves as a useful intermediate stage between simply taking \emph{one} PG step of \eqref{eq-cov-upd} and using conventional step sizes, which would mean that the covariance would either converge to zero or diverge to infinity.} to use Gaussian exploration with covariance proportional to $e^{cH}$ for some reward scaling constant $c$, as in
\begin{gather*}
\label{eq=exp-sigma}
\Sigma = \sigma_0 e^H  = \sigma_0 U e^\Lambda U^\top \quad \text{where} \quad H(s) = U \Lambda U^\top.
\end{gather*}
Thus, by exploring with (scaled) covariance $e^{cH}$, we obtain a critic-driven alternative to the Ornstein-Uhlenbeck heuristic of \eqref{ou-noise}. Our results below show that it also performs much better in practice.

In order to show that exponentiating the eigenvalues is really necessary, we also implemented a simpler version, which we call 1-step EPG. It approximates the matrix exponential as
\begin{gather*}
    \label{eq=1step-sigma}
    \sigma_0 e^H \approx \sigma_0 U \max(1 + \Lambda, 0) U^\top \quad \text{where} \quad H(s) = U \Lambda U^\top.
\end{gather*}
Here, the $\max$ operator applies to each entry in the diagonal matrix separately. This process corresponds to truncating the Taylor series of an exponential function after the linear term and then constraining the eigenvalues to be positive. It can also be interpreted as performing just one iteration of \eqref{eq-cov-upd}, starting with initial covariance $\sigma_0 I$. 

Lemma \ref{lem-hexplore} has an intuitive interpretation. If $H(s)$ has a large positive eigenvalue $\lambda$, then $\hat{Q}(\cdot,s)$ has a sharp minimum along the corresponding eigenvector, and the corresponding eigenvalue of $\Sigma^{1/2}$ is $e^\lambda$, i.e., also large.
This is easiest to see with a one-dimensional action space, where the Hessian and its only eigenvalue are scalar. The exploration mechanism in the one-dimensional case is illustrated in Figure \ref{fig-explore-1d}. If $\lambda$ is positive, we have a minimum and apply exploration noise. If, on the other hand, $\lambda$ is negative, then $\hat{Q}(\cdot,s)$ has a maximum and so $e^\lambda$ is small, leading to a policy that does not deviate too much from the existing maximum.

\begin{figure}[t]
    \centering
        \begingroup%
  \makeatletter%
  \providecommand\color[2][]{%
    \errmessage{(Inkscape) Color is used for the text in Inkscape, but the package 'color.sty' is not loaded}%
    \renewcommand\color[2][]{}%
  }%
  \providecommand\transparent[1]{%
    \errmessage{(Inkscape) Transparency is used (non-zero) for the text in Inkscape, but the package 'transparent.sty' is not loaded}%
    \renewcommand\transparent[1]{}%
  }%
  \providecommand\rotatebox[2]{#2}%
  \ifx\svgwidth\undefined%
    \setlength{\unitlength}{396.85040104bp}%
    \ifx\svgscale\undefined%
      \relax%
    \else%
      \setlength{\unitlength}{\unitlength * \real{\svgscale}}%
    \fi%
  \else%
    \setlength{\unitlength}{\svgwidth}%
  \fi%
  \global\let\svgwidth\undefined%
  \global\let\svgscale\undefined%
  \makeatother%
  \begin{picture}(1,0.2353654)%
    \put(0,0.035){\includegraphics[width=\unitlength,page=1]{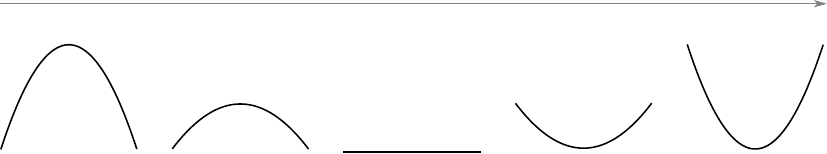}}%
    \tiny
    \put(0.42,0.2253504){\color[rgb]{0,0,0}\makebox(0,0)[lb]{\smash{eigenvalue increases}}}%
    \put(0.02,0.01739747){\color[rgb]{0,0,0}\makebox(0,0)[lb]{\smash{sharp maximum,}}}%
    \put(0.858,0.01772701){\color[rgb]{0,0,0}\makebox(0,0)[lb]{\smash{sharp minimum,}}}%
    \put(0.415,0.01640887){\color[rgb]{0,0,0}\makebox(0,0)[lb]{\smash{moderate exploration}}}%
    \put(0.001,0.0){\color[rgb]{0,0,0}\makebox(0,0)[lb]{\smash{very little exploration}}}%
    \put(0.847,0.0){\color[rgb]{0,0,0}\makebox(0,0)[lb]{\smash{lots of exploration}}}%
  \end{picture}%
\endgroup%

        \caption{The parabolas show different possible curvatures of the critic $\hat{Q}(\cdot, s)$. We set exploration to be the strongest for sharp mimima, on the left side of the figure. The exploration strength then increases as we move towards the right. There is almost no exploration to the far left, where we have a sharp maximum.}
        \label{fig-explore-1d}
\end{figure}

In the multi-dimensional case, the critic can have saddle points, as shown in Figure \ref{fig-saddle-point}. For the case shown in the figure, we explore  little along the blue eigenvector (since the intersection of $Q(\cdot,s)$ with the blue plane shows a maximum) and much more along the red eigenvector (since the intersection of $Q(\cdot,s)$ with the red plane shows a minimum, which we want to escape). In essence, we apply the one-dimensional reasoning shown in Figure \ref{fig-explore-1d} to each plane separately, where the planes are spanned by the corresponding eigenvector and the $z$-axis. This way, we can escape saddle points and minima.\footnote{Of course the optimization is still local and there is no guarantee of finding a global optimum---we can merely increase our chances.} In practice, we show experimentally that good performance is obtained if we consider quadratic functions constrained to have a diagonal Hessian, i.e., where eigenvectors are axis-aligned.

\begin{figure}[t]
    \centering
        \def\svgwidth{6cm} \tiny
        \begingroup%
  \makeatletter%
  \providecommand\color[2][]{%
    \errmessage{(Inkscape) Color is used for the text in Inkscape, but the package 'color.sty' is not loaded}%
    \renewcommand\color[2][]{}%
  }%
  \providecommand\transparent[1]{%
    \errmessage{(Inkscape) Transparency is used (non-zero) for the text in Inkscape, but the package 'transparent.sty' is not loaded}%
    \renewcommand\transparent[1]{}%
  }%
  \providecommand\rotatebox[2]{#2}%
  \ifx\svgwidth\undefined%
    \setlength{\unitlength}{338.89057876bp}%
    \ifx\svgscale\undefined%
      \relax%
    \else%
      \setlength{\unitlength}{\unitlength * \real{\svgscale}}%
    \fi%
  \else%
    \setlength{\unitlength}{\svgwidth}%
  \fi%
  \global\let\svgwidth\undefined%
  \global\let\svgscale\undefined%
  \makeatother%
  \begin{picture}(1,0.83535079)%
    \put(0,0){\includegraphics[width=\unitlength,page=1]{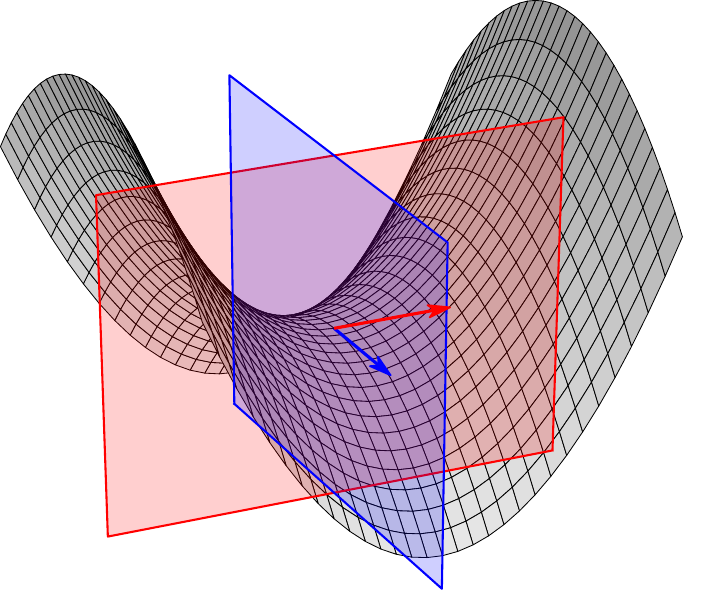}}%
    \put(0.86443659,0.77555611){\color[rgb]{0,0,0}\makebox(0,0)[lb]{\smash{$Q(\cdot,s)$}}}%
    \put(0.51476622,0.3993285){\color[rgb]{0,0,0}\makebox(0,0)[lb]{\smash{$u_1$}}}%
    \put(0.45452074,0.32787558){\color[rgb]{0,0,0}\makebox(0,0)[lb]{\smash{$u_2$}}}%
  \end{picture}%
\endgroup%
        \caption{In multi-dimensional action spaces, the critic $\hat{Q}(\cdot,s)$ can have saddle points. In this case, we define exploration along each eigenvector separately.}
        \label{fig-saddle-point}
\end{figure}

\subsection{Behavior of Policy Gradients across Optimization Time}
\begin{figure}
\centering
\small
\begin{subfigure}{0.3\textwidth}
\centering
\begin{tikzpicture}[line cap=round,line join=round,>=triangle 45,x=0.6cm,y=0.6cm]
\begin{axis}[
x=0.6cm,y=0.6cm,
axis lines=middle,
ymajorgrids=true,
xmajorgrids=true,
xmin=-3.0200000000000045,
xmax=3.2599999999999976,
ymin=-3.430000000000009,
ymax=1.4500000000000035,
xtick={-3,-2,...,3},
ytick={-3,-2,...,1},]
\clip(-3.02,-3.43) rectangle (3.26,1.45);
\draw[line width=2pt,color=qqwuqq,smooth,samples=100,domain=-3.0200000000000045:3.2599999999999976] plot(\x,{0-(\x)^(2)});
\begin{scriptsize}
\draw[color=qqwuqq] (-1.18,-3.1) node {$\hat{Q}$};
\draw[] (3,0.2) node {$a$};
\end{scriptsize}
\end{axis}
\end{tikzpicture}
\caption{ Continuous bandit. }
\label{fig-cb-def}
\end{subfigure}
\begin{subfigure}{0.65\textwidth}
\centering
\begin{tikzpicture}[line cap=round,line join=round,>=triangle 45,x=0.015cm,y=3cm]
\begin{axis}[
x=0.015cm,y=3cm,
axis lines=middle,
ymajorgrids=true,
xmajorgrids=true,
xmin=-0.01,
xmax=500.0000000000001,
ymin=-0.047641509433962344,
ymax=1.202358490566037,
xtick={0,50,...,500},
ytick={0,0.2,...,1.2000000000000002},]
\clip(-0.01,-0.047641509433962344) rectangle (500.0000000000001,1.202358490566037);
\draw[line width=2pt,color=ffqqtt] (-0.01,1.0001000050001667) -- (-0.01,1.0001000050001667);
\draw[line width=2pt,color=ffqqtt] (-0.01,1.0001000050001667) -- (1.2400250000000002,0.9876763162928744);
\draw[line width=2pt,color=ffqqtt] (1.2400250000000002,0.9876763162928744) -- (2.49005,0.9754069601926456);
\draw[line width=2pt,color=ffqqtt] (2.49005,0.9754069601926456) -- (3.740075,0.963290019510941);
\draw[line width=2pt,color=ffqqtt] (3.740075,0.963290019510941) -- (4.9901,0.9513236008753932);
\draw[line width=2pt,color=ffqqtt] (4.9901,0.9513236008753932) -- (6.240125,0.939505834433952);
\draw[line width=2pt,color=ffqqtt] (6.240125,0.939505834433952) -- (7.49015,0.9278348735627037);
\draw[line width=2pt,color=ffqqtt] (7.49015,0.9278348735627037) -- (8.740175,0.9163088945773211);
\draw[line width=2pt,color=ffqqtt] (8.740175,0.9163088945773211) -- (9.990200000000002,0.9049260964480983);
\draw[line width=2pt,color=ffqqtt] (9.990200000000002,0.9049260964480983) -- (11.240225000000002,0.8936847005185239);
\draw[line width=2pt,color=ffqqtt] (11.240225000000002,0.8936847005185239) -- (12.490250000000003,0.8825829502273518);
\draw[line width=2pt,color=ffqqtt] (12.490250000000003,0.8825829502273518) -- (13.740275000000004,0.8716191108341239);
\draw[line width=2pt,color=ffqqtt] (13.740275000000004,0.8716191108341239) -- (14.990300000000005,0.8607914691481027);
\draw[line width=2pt,color=ffqqtt] (14.990300000000005,0.8607914691481027) -- (16.240325000000006,0.85009833326057);
\draw[line width=2pt,color=ffqqtt] (16.240325000000006,0.85009833326057) -- (17.490350000000007,0.8395380322804541);
\draw[line width=2pt,color=ffqqtt] (17.490350000000007,0.8395380322804541) -- (18.740375000000007,0.8291089160732371);
\draw[line width=2pt,color=ffqqtt] (18.740375000000007,0.8291089160732371) -- (19.990400000000008,0.8188093550031094);
\draw[line width=2pt,color=ffqqtt] (19.990400000000008,0.8188093550031094) -- (21.24042500000001,0.8086377396783242);
\draw[line width=2pt,color=ffqqtt] (21.24042500000001,0.8086377396783242) -- (22.49045000000001,0.7985924806997182);
\draw[line width=2pt,color=ffqqtt] (22.49045000000001,0.7985924806997182) -- (23.74047500000001,0.7886720084123534);
\draw[line width=2pt,color=ffqqtt] (23.74047500000001,0.7886720084123534) -- (24.99050000000001,0.7788747726602464);
\draw[line width=2pt,color=ffqqtt] (24.99050000000001,0.7788747726602464) -- (26.240525000000012,0.7691992425441434);
\draw[line width=2pt,color=ffqqtt] (26.240525000000012,0.7691992425441434) -- (27.490550000000013,0.7596439061823045);
\draw[line width=2pt,color=ffqqtt] (27.490550000000013,0.7596439061823045) -- (28.740575000000014,0.7502072704742596);
\draw[line width=2pt,color=ffqqtt] (28.740575000000014,0.7502072704742596) -- (29.990600000000015,0.7408878608674992);
\draw[line width=2pt,color=ffqqtt] (29.990600000000015,0.7408878608674992) -- (31.240625000000016,0.7316842211270637);
\draw[line width=2pt,color=ffqqtt] (31.240625000000016,0.7316842211270637) -- (32.49065000000002,0.7225949131079934);
\draw[line width=2pt,color=ffqqtt] (32.49065000000002,0.7225949131079934) -- (33.74067500000002,0.713618516530608);
\draw[line width=2pt,color=ffqqtt] (33.74067500000002,0.713618516530608) -- (34.99070000000002,0.7047536287585753);
\draw[line width=2pt,color=ffqqtt] (34.99070000000002,0.7047536287585753) -- (36.24072500000002,0.6959988645797374);
\draw[line width=2pt,color=ffqqtt] (36.24072500000002,0.6959988645797374) -- (37.49075000000002,0.6873528559896603);
\draw[line width=2pt,color=ffqqtt] (37.49075000000002,0.6873528559896603) -- (38.74077500000002,0.6788142519778722);
\draw[line width=2pt,color=ffqqtt] (38.74077500000002,0.6788142519778722) -- (39.99080000000002,0.6703817183167559);
\draw[line width=2pt,color=ffqqtt] (39.99080000000002,0.6703817183167559) -- (41.24082500000002,0.6620539373530655);
\draw[line width=2pt,color=ffqqtt] (41.24082500000002,0.6620539373530655) -- (42.49085000000002,0.653829607802032);
\draw[line width=2pt,color=ffqqtt] (42.49085000000002,0.653829607802032) -- (43.740875000000024,0.6457074445440266);
\draw[line width=2pt,color=ffqqtt] (43.740875000000024,0.6457074445440266) -- (44.990900000000025,0.6376861784237502);
\draw[line width=2pt,color=ffqqtt] (44.990900000000025,0.6376861784237502) -- (46.240925000000026,0.6297645560519173);
\draw[line width=2pt,color=ffqqtt] (46.240925000000026,0.6297645560519173) -- (47.490950000000026,0.6219413396094039);
\draw[line width=2pt,color=ffqqtt] (47.490950000000026,0.6219413396094039) -- (48.74097500000003,0.6142153066538274);
\draw[line width=2pt,color=ffqqtt] (48.74097500000003,0.6142153066538274) -- (49.99100000000003,0.6065852499285302);
\draw[line width=2pt,color=ffqqtt] (49.99100000000003,0.6065852499285302) -- (51.24102500000003,0.5990499771739363);
\draw[line width=2pt,color=ffqqtt] (51.24102500000003,0.5990499771739363) -- (52.49105000000003,0.5916083109412497);
\draw[line width=2pt,color=ffqqtt] (52.49105000000003,0.5916083109412497) -- (53.74107500000003,0.5842590884084692);
\draw[line width=2pt,color=ffqqtt] (53.74107500000003,0.5842590884084692) -- (54.99110000000003,0.5770011611986879);
\draw[line width=2pt,color=ffqqtt] (54.99110000000003,0.5770011611986879) -- (56.24112500000003,0.5698333952006491);
\draw[line width=2pt,color=ffqqtt] (56.24112500000003,0.5698333952006491) -- (57.49115000000003,0.5627546703915325);
\draw[line width=2pt,color=ffqqtt] (57.49115000000003,0.5627546703915325) -- (58.741175000000034,0.5557638806619414);
\draw[line width=2pt,color=ffqqtt] (58.741175000000034,0.5557638806619414) -- (59.991200000000035,0.5488599336430635);
\draw[line width=2pt,color=ffqqtt] (59.991200000000035,0.5488599336430635) -- (61.241225000000036,0.5420417505359798);
\draw[line width=2pt,color=ffqqtt] (61.241225000000036,0.5420417505359798) -- (62.491250000000036,0.5353082659430929);
\draw[line width=2pt,color=ffqqtt] (62.491250000000036,0.5353082659430929) -- (63.74127500000004,0.52865842770165);
\draw[line width=2pt,color=ffqqtt] (63.74127500000004,0.52865842770165) -- (64.99130000000004,0.5220911967193336);
\draw[line width=2pt,color=ffqqtt] (64.99130000000004,0.5220911967193336) -- (66.24132500000003,0.5156055468118951);
\draw[line width=2pt,color=ffqqtt] (66.24132500000003,0.5156055468118951) -- (67.49135000000003,0.5092004645428043);
\draw[line width=2pt,color=ffqqtt] (67.49135000000003,0.5092004645428043) -- (68.74137500000002,0.5028749490648924);
\draw[line width=2pt,color=ffqqtt] (68.74137500000002,0.5028749490648924) -- (69.99140000000001,0.4966280119639606);
\draw[line width=2pt,color=ffqqtt] (69.99140000000001,0.4966280119639606) -- (71.241425,0.4904586771043327);
\draw[line width=2pt,color=ffqqtt] (71.241425,0.4904586771043327) -- (72.49145,0.48436598047632545);
\draw[line width=2pt,color=ffqqtt] (72.49145,0.48436598047632545) -- (73.741475,0.4783489700456144);
\draw[line width=2pt,color=ffqqtt] (73.741475,0.4783489700456144) -- (74.99149999999999,0.4724067056044703);
\draw[line width=2pt,color=ffqqtt] (74.99149999999999,0.4724067056044703) -- (76.24152499999998,0.46653825862484416);
\draw[line width=2pt,color=ffqqtt] (76.24152499999998,0.46653825862484416) -- (77.49154999999998,0.4607427121132768);
\draw[line width=2pt,color=ffqqtt] (77.49154999999998,0.4607427121132768) -- (78.74157499999997,0.4550191604676112);
\draw[line width=2pt,color=ffqqtt] (78.74157499999997,0.4550191604676112) -- (79.99159999999996,0.4493667093354846);
\draw[line width=2pt,color=ffqqtt] (79.99159999999996,0.4493667093354846) -- (81.24162499999996,0.4437844754745785);
\draw[line width=2pt,color=ffqqtt] (81.24162499999996,0.4437844754745785) -- (82.49164999999995,0.4382715866146048);
\draw[line width=2pt,color=ffqqtt] (82.49164999999995,0.4382715866146048) -- (83.74167499999994,0.43282718132100617);
\draw[line width=2pt,color=ffqqtt] (83.74167499999994,0.43282718132100617) -- (84.99169999999994,0.4274504088603502);
\draw[line width=2pt,color=ffqqtt] (84.99169999999994,0.4274504088603502) -- (86.24172499999993,0.4221404290673945);
\draw[line width=2pt,color=ffqqtt] (86.24172499999993,0.4221404290673945) -- (87.49174999999993,0.4168964122138048);
\draw[line width=2pt,color=ffqqtt] (87.49174999999993,0.4168964122138048) -- (88.74177499999992,0.41171753887850226);
\draw[line width=2pt,color=ffqqtt] (88.74177499999992,0.41171753887850226) -- (89.99179999999991,0.4066029998196228);
\draw[line width=2pt,color=ffqqtt] (89.99179999999991,0.4066029998196228) -- (91.2418249999999,0.4015519958480656);
\draw[line width=2pt,color=ffqqtt] (91.2418249999999,0.4015519958480656) -- (92.4918499999999,0.3965637377026141);
\draw[line width=2pt,color=ffqqtt] (92.4918499999999,0.3965637377026141) -- (93.7418749999999,0.39163744592660643);
\draw[line width=2pt,color=ffqqtt] (93.7418749999999,0.39163744592660643) -- (94.99189999999989,0.3867723507461397);
\draw[line width=2pt,color=ffqqtt] (94.99189999999989,0.3867723507461397) -- (96.24192499999988,0.38196769194978575);
\draw[line width=2pt,color=ffqqtt] (96.24192499999988,0.38196769194978575) -- (97.49194999999987,0.3772227187698024);
\draw[line width=2pt,color=ffqqtt] (97.49194999999987,0.3772227187698024) -- (98.74197499999987,0.3725366897648194);
\draw[line width=2pt,color=ffqqtt] (98.74197499999987,0.3725366897648194) -- (99.99199999999986,0.3679088727039822);
\draw[line width=2pt,color=ffqqtt] (99.99199999999986,0.3679088727039822) -- (101.24202499999986,0.36333854445253466);
\draw[line width=2pt,color=ffqqtt] (101.24202499999986,0.36333854445253466) -- (102.49204999999985,0.35882499085882363);
\draw[line width=2pt,color=ffqqtt] (102.49204999999985,0.35882499085882363) -- (103.74207499999984,0.3543675066427065);
\draw[line width=2pt,color=ffqqtt] (103.74207499999984,0.3543675066427065) -- (104.99209999999984,0.34996539528534526);
\draw[line width=2pt,color=ffqqtt] (104.99209999999984,0.34996539528534526) -- (106.24212499999983,0.34561796892037);
\draw[line width=2pt,color=ffqqtt] (106.24212499999983,0.34561796892037) -- (107.49214999999982,0.34132454822639396);
\draw[line width=2pt,color=ffqqtt] (107.49214999999982,0.34132454822639396) -- (108.74217499999982,0.33708446232086386);
\draw[line width=2pt,color=ffqqtt] (108.74217499999982,0.33708446232086386) -- (109.99219999999981,0.33289704865522884);
\draw[line width=2pt,color=ffqqtt] (109.99219999999981,0.33289704865522884) -- (111.2422249999998,0.328761652911412);
\draw[line width=2pt,color=ffqqtt] (111.2422249999998,0.328761652911412) -- (112.4922499999998,0.3246776288995678);
\draw[line width=2pt,color=ffqqtt] (112.4922499999998,0.3246776288995678) -- (113.7422749999998,0.3206443384571092);
\draw[line width=2pt,color=ffqqtt] (113.7422749999998,0.3206443384571092) -- (114.99229999999979,0.3166611513489899);
\draw[line width=2pt,color=ffqqtt] (114.99229999999979,0.3166611513489899) -- (116.24232499999978,0.3127274451692245);
\draw[line width=2pt,color=ffqqtt] (116.24232499999978,0.3127274451692245) -- (117.49234999999977,0.30884260524363266);
\draw[line width=2pt,color=ffqqtt] (117.49234999999977,0.30884260524363266) -- (118.74237499999977,0.30500602453379116);
\draw[line width=2pt,color=ffqqtt] (118.74237499999977,0.30500602453379116) -- (119.99239999999976,0.3012171035421791);
\draw[line width=2pt,color=ffqqtt] (119.99239999999976,0.3012171035421791) -- (121.24242499999976,0.29747525021850113);
\draw[line width=2pt,color=ffqqtt] (121.24242499999976,0.29747525021850113) -- (122.49244999999975,0.29377987986717524);
\draw[line width=2pt,color=ffqqtt] (122.49244999999975,0.29377987986717524) -- (123.74247499999974,0.2901304150559688);
\draw[line width=2pt,color=ffqqtt] (123.74247499999974,0.2901304150559688) -- (124.99249999999974,0.2865262855257703);
\draw[line width=2pt,color=ffqqtt] (124.99249999999974,0.2865262855257703) -- (126.24252499999973,0.2829669281014812);
\draw[line width=2pt,color=ffqqtt] (126.24252499999973,0.2829669281014812) -- (127.49254999999972,0.2794517866040157);
\draw[line width=2pt,color=ffqqtt] (127.49254999999972,0.2794517866040157) -- (128.74257499999973,0.2759803117633928);
\draw[line width=2pt,color=ffqqtt] (128.74257499999973,0.2759803117633928) -- (129.99259999999973,0.27255196113290836);
\draw[line width=2pt,color=ffqqtt] (129.99259999999973,0.27255196113290836) -- (131.24262499999972,0.26916619900437333);
\draw[line width=2pt,color=ffqqtt] (131.24262499999972,0.26916619900437333) -- (132.4926499999997,0.265822496324405);
\draw[line width=2pt,color=ffqqtt] (132.4926499999997,0.265822496324405) -- (133.7426749999997,0.26252033061175795);
\draw[line width=2pt,color=ffqqtt] (133.7426749999997,0.26252033061175795) -- (134.9926999999997,0.2592591858756819);
\draw[line width=2pt,color=ffqqtt] (134.9926999999997,0.2592591858756819) -- (136.2427249999997,0.25603855253529423);
\draw[line width=2pt,color=ffqqtt] (136.2427249999997,0.25603855253529423) -- (137.4927499999997,0.2528579273399532);
\draw[line width=2pt,color=ffqqtt] (137.4927499999997,0.2528579273399532) -- (138.74277499999968,0.24971681329062162);
\draw[line width=2pt,color=ffqqtt] (138.74277499999968,0.24971681329062162) -- (139.99279999999968,0.24661471956220576);
\draw[line width=2pt,color=ffqqtt] (139.99279999999968,0.24661471956220576) -- (141.24282499999967,0.2435511614268606);
\draw[line width=2pt,color=ffqqtt] (141.24282499999967,0.2435511614268606) -- (142.49284999999966,0.2405256601782467);
\draw[line width=2pt,color=ffqqtt] (142.49284999999966,0.2405256601782467) -- (143.74287499999966,0.2375377430567285);
\draw[line width=2pt,color=ffqqtt] (143.74287499999966,0.2375377430567285) -- (144.99289999999965,0.23458694317550155);
\draw[line width=2pt,color=ffqqtt] (144.99289999999965,0.23458694317550155) -- (146.24292499999964,0.23167279944763788);
\draw[line width=2pt,color=ffqqtt] (146.24292499999964,0.23167279944763788) -- (147.49294999999964,0.22879485651403705);
\draw[line width=2pt,color=ffqqtt] (147.49294999999964,0.22879485651403705) -- (148.74297499999963,0.22595266467227273);
\draw[line width=2pt,color=ffqqtt] (148.74297499999963,0.22595266467227273) -- (149.99299999999963,0.22314577980632272);
\draw[line width=2pt,color=ffqqtt] (149.99299999999963,0.22314577980632272) -- (151.24302499999962,0.22037376331717248);
\draw[line width=2pt,color=ffqqtt] (151.24302499999962,0.22037376331717248) -- (152.4930499999996,0.2176361820542801);
\draw[line width=2pt,color=ffqqtt] (152.4930499999996,0.2176361820542801) -- (153.7430749999996,0.21493260824789315);
\draw[line width=2pt,color=ffqqtt] (153.7430749999996,0.21493260824789315) -- (154.9930999999996,0.2122626194422059);
\draw[line width=2pt,color=ffqqtt] (154.9930999999996,0.2122626194422059) -- (156.2431249999996,0.20962579842934734);
\draw[line width=2pt,color=ffqqtt] (156.2431249999996,0.20962579842934734) -- (157.4931499999996,0.20702173318418882);
\draw[line width=2pt,color=ffqqtt] (157.4931499999996,0.20702173318418882) -- (158.74317499999958,0.2044500167999618);
\draw[line width=2pt,color=ffqqtt] (158.74317499999958,0.2044500167999618) -- (159.99319999999958,0.20191024742467525);
\draw[line width=2pt,color=ffqqtt] (159.99319999999958,0.20191024742467525) -- (161.24322499999957,0.1994020281983229);
\draw[line width=2pt,color=ffqqtt] (161.24322499999957,0.1994020281983229) -- (162.49324999999956,0.19692496719087071);
\draw[line width=2pt,color=ffqqtt] (162.49324999999956,0.19692496719087071) -- (163.74327499999956,0.19447867734101446);
\draw[line width=2pt,color=ffqqtt] (163.74327499999956,0.19447867734101446) -- (164.99329999999955,0.19206277639569824);
\draw[line width=2pt,color=ffqqtt] (164.99329999999955,0.19206277639569824) -- (166.24332499999954,0.1896768868503842);
\draw[line width=2pt,color=ffqqtt] (166.24332499999954,0.1896768868503842) -- (167.49334999999954,0.18732063589006442);
\draw[line width=2pt,color=ffqqtt] (167.49334999999954,0.18732063589006442) -- (168.74337499999953,0.18499365533100545);
\draw[line width=2pt,color=ffqqtt] (168.74337499999953,0.18499365533100545) -- (169.99339999999953,0.1826955815632165);
\draw[line width=2pt,color=ffqqtt] (169.99339999999953,0.1826955815632165) -- (171.24342499999952,0.18042605549363236);
\draw[line width=2pt,color=ffqqtt] (171.24342499999952,0.18042605549363236) -- (172.4934499999995,0.17818472249000225);
\draw[line width=2pt,color=ffqqtt] (172.4934499999995,0.17818472249000225) -- (173.7434749999995,0.17597123232547549);
\draw[line width=2pt,color=ffqqtt] (173.7434749999995,0.17597123232547549) -- (174.9934999999995,0.17378523912387572);
\draw[line width=2pt,color=ffqqtt] (174.9934999999995,0.17378523912387572) -- (176.2435249999995,0.17162640130565474);
\draw[line width=2pt,color=ffqqtt] (176.2435249999995,0.17162640130565474) -- (177.4935499999995,0.16949438153451804);
\draw[line width=2pt,color=ffqqtt] (177.4935499999995,0.16949438153451804) -- (178.74357499999948,0.16738884666471318);
\draw[line width=2pt,color=ffqqtt] (178.74357499999948,0.16738884666471318) -- (179.99359999999947,0.16530946768897298);
\draw[line width=2pt,color=ffqqtt] (179.99359999999947,0.16530946768897298) -- (181.24362499999947,0.16325591968710526);
\draw[line width=2pt,color=ffqqtt] (181.24362499999947,0.16325591968710526) -- (182.49364999999946,0.16122788177522168);
\draw[line width=2pt,color=ffqqtt] (182.49364999999946,0.16122788177522168) -- (183.74367499999946,0.15922503705559674);
\draw[line width=2pt,color=ffqqtt] (183.74367499999946,0.15922503705559674) -- (184.99369999999945,0.15724707256714998);
\draw[line width=2pt,color=ffqqtt] (184.99369999999945,0.15724707256714998) -- (186.24372499999944,0.15529367923654308);
\draw[line width=2pt,color=ffqqtt] (186.24372499999944,0.15529367923654308) -- (187.49374999999944,0.15336455182988482);
\draw[line width=2pt,color=ffqqtt] (187.49374999999944,0.15336455182988482) -- (188.74377499999943,0.15145938890503555);
\draw[line width=2pt,color=ffqqtt] (188.74377499999943,0.15145938890503555) -- (189.99379999999942,0.14957789276450453);
\draw[line width=2pt,color=ffqqtt] (189.99379999999942,0.14957789276450453) -- (191.24382499999942,0.1477197694089321);
\draw[line width=2pt,color=ffqqtt] (191.24382499999942,0.1477197694089321) -- (192.4938499999994,0.14588472849114986);
\draw[line width=2pt,color=ffqqtt] (192.4938499999994,0.14588472849114986) -- (193.7438749999994,0.14407248327081154);
\draw[line width=2pt,color=ffqqtt] (193.7438749999994,0.14407248327081154) -- (194.9938999999994,0.1422827505695875);
\draw[line width=2pt,color=ffqqtt] (194.9938999999994,0.1422827505695875) -- (196.2439249999994,0.14051525072691573);
\draw[line width=2pt,color=ffqqtt] (196.2439249999994,0.14051525072691573) -- (197.4939499999994,0.1387697075563025);
\draw[line width=2pt,color=ffqqtt] (197.4939499999994,0.1387697075563025) -- (198.74397499999938,0.137045848302166);
\draw[line width=2pt,color=ffqqtt] (198.74397499999938,0.137045848302166) -- (199.99399999999937,0.1353434035972161);
\draw[line width=2pt,color=ffqqtt] (199.99399999999937,0.1353434035972161) -- (201.24402499999937,0.13366210742036336);
\draw[line width=2pt,color=ffqqtt] (201.24402499999937,0.13366210742036336) -- (202.49404999999936,0.13200169705515094);
\draw[line width=2pt,color=ffqqtt] (202.49404999999936,0.13200169705515094) -- (203.74407499999936,0.13036191304870337);
\draw[line width=2pt,color=ffqqtt] (203.74407499999936,0.13036191304870337) -- (204.99409999999935,0.12874249917118427);
\draw[line width=2pt,color=ffqqtt] (204.99409999999935,0.12874249917118427) -- (206.24412499999934,0.12714320237575896);
\draw[line width=2pt,color=ffqqtt] (206.24412499999934,0.12714320237575896) -- (207.49414999999934,0.12556377275905337);
\draw[line width=2pt,color=ffqqtt] (207.49414999999934,0.12556377275905337) -- (208.74417499999933,0.1240039635221047);
\draw[line width=2pt,color=ffqqtt] (208.74417499999933,0.1240039635221047) -- (209.99419999999932,0.12246353093179711);
\draw[line width=2pt,color=ffqqtt] (209.99419999999932,0.12246353093179711) -- (211.24422499999932,0.12094223428277626);
\draw[line width=2pt,color=ffqqtt] (211.24422499999932,0.12094223428277626) -- (212.4942499999993,0.11943983585983718);
\draw[line width=2pt,color=ffqqtt] (212.4942499999993,0.11943983585983718) -- (213.7442749999993,0.11795610090077936);
\draw[line width=2pt,color=ffqqtt] (213.7442749999993,0.11795610090077936) -- (214.9942999999993,0.11649079755972301);
\draw[line width=2pt,color=ffqqtt] (214.9942999999993,0.11649079755972301) -- (216.2443249999993,0.11504369687088145);
\draw[line width=2pt,color=ffqqtt] (216.2443249999993,0.11504369687088145) -- (217.4943499999993,0.11361457271278319);
\draw[line width=2pt,color=ffqqtt] (217.4943499999993,0.11361457271278319) -- (218.74437499999928,0.11220320177293865);
\draw[line width=2pt,color=ffqqtt] (218.74437499999928,0.11220320177293865) -- (219.99439999999927,0.11080936351294558);
\draw[line width=2pt,color=ffqqtt] (219.99439999999927,0.11080936351294558) -- (221.24442499999927,0.1094328401340283);
\draw[line width=2pt,color=ffqqtt] (221.24442499999927,0.1094328401340283) -- (222.49444999999926,0.10807341654300472);
\draw[line width=2pt,color=ffqqtt] (222.49444999999926,0.10807341654300472) -- (223.74447499999926,0.10673088031867624);
\draw[line width=2pt,color=ffqqtt] (223.74447499999926,0.10673088031867624) -- (224.99449999999925,0.10540502167863532);
\draw[line width=2pt,color=ffqqtt] (224.99449999999925,0.10540502167863532) -- (226.24452499999924,0.10409563344648494);
\draw[line width=2pt,color=ffqqtt] (226.24452499999924,0.10409563344648494) -- (227.49454999999924,0.10280251101946594);
\draw[line width=2pt,color=ffqqtt] (227.49454999999924,0.10280251101946594) -- (228.74457499999923,0.10152545233648592);
\draw[line width=2pt,color=ffqqtt] (228.74457499999923,0.10152545233648592) -- (229.99459999999922,0.10026425784654558);
\draw[line width=2pt,color=ffqqtt] (229.99459999999922,0.10026425784654558) -- (231.24462499999922,0.09901873047755716);
\draw[line width=2pt,color=ffqqtt] (231.24462499999922,0.09901873047755716) -- (232.4946499999992,0.09778867560555041);
\draw[line width=2pt,color=ffqqtt] (232.4946499999992,0.09778867560555041) -- (233.7446749999992,0.09657390102426087);
\draw[line width=2pt,color=ffqqtt] (233.7446749999992,0.09657390102426087) -- (234.9946999999992,0.09537421691509619);
\draw[line width=2pt,color=ffqqtt] (234.9946999999992,0.09537421691509619) -- (236.2447249999992,0.09418943581747515);
\draw[line width=2pt,color=ffqqtt] (236.2447249999992,0.09418943581747515) -- (237.4947499999992,0.09301937259953566);
\draw[line width=2pt,color=ffqqtt] (237.4947499999992,0.09301937259953566) -- (238.74477499999918,0.09186384442920627);
\draw[line width=2pt,color=ffqqtt] (238.74477499999918,0.09186384442920627) -- (239.99479999999917,0.09072267074563711);
\draw[line width=2pt,color=ffqqtt] (239.99479999999917,0.09072267074563711) -- (241.24482499999917,0.0895956732309858);
\draw[line width=2pt,color=ffqqtt] (241.24482499999917,0.0895956732309858) -- (242.49484999999916,0.08848267578255377);
\draw[line width=2pt,color=ffqqtt] (242.49484999999916,0.08848267578255377) -- (243.74487499999915,0.08738350448526883);
\draw[line width=2pt,color=ffqqtt] (243.74487499999915,0.08738350448526883) -- (244.99489999999915,0.08629798758450946);
\draw[line width=2pt,color=ffqqtt] (244.99489999999915,0.08629798758450946) -- (246.24492499999914,0.08522595545926662);
\draw[line width=2pt,color=ffqqtt] (246.24492499999914,0.08522595545926662) -- (247.49494999999914,0.08416724059563925);
\draw[line width=2pt,color=ffqqtt] (247.49494999999914,0.08416724059563925) -- (248.74497499999913,0.08312167756065875);
\draw[line width=2pt,color=ffqqtt] (248.74497499999913,0.08312167756065875) -- (249.99499999999912,0.08208910297643868);
\draw[line width=2pt,color=ffqqtt] (249.99499999999912,0.08208910297643868) -- (251.24502499999912,0.08106935549464564);
\draw[line width=2pt,color=ffqqtt] (251.24502499999912,0.08106935549464564) -- (252.4950499999991,0.08006227577128726);
\draw[line width=2pt,color=ffqqtt] (252.4950499999991,0.08006227577128726) -- (253.7450749999991,0.07906770644181337);
\draw[line width=2pt,color=ffqqtt] (253.7450749999991,0.07906770644181337) -- (254.9950999999991,0.07808549209652647);
\draw[line width=2pt,color=ffqqtt] (254.9950999999991,0.07808549209652647) -- (256.2451249999991,0.07711547925629776);
\draw[line width=2pt,color=ffqqtt] (256.2451249999991,0.07711547925629776) -- (257.4951499999991,0.07615751634858461);
\draw[line width=2pt,color=ffqqtt] (257.4951499999991,0.07615751634858461) -- (258.7451749999991,0.07521145368374624);
\draw[line width=2pt,color=ffqqtt] (258.7451749999991,0.07521145368374624) -- (259.9951999999991,0.07427714343165337);
\draw[line width=2pt,color=ffqqtt] (259.9951999999991,0.07427714343165337) -- (261.2452249999991,0.07335443959858859);
\draw[line width=2pt,color=ffqqtt] (261.2452249999991,0.07335443959858859) -- (262.4952499999991,0.07244319800443368);
\draw[line width=2pt,color=ffqqtt] (262.4952499999991,0.07244319800443368) -- (263.7452749999991,0.07154327626014004);
\draw[line width=2pt,color=ffqqtt] (263.7452749999991,0.07154327626014004) -- (264.9952999999991,0.07065453374547959);
\draw[line width=2pt,color=ffqqtt] (264.9952999999991,0.07065453374547959) -- (266.24532499999907,0.06977683158707138);
\draw[line width=2pt,color=ffqqtt] (266.24532499999907,0.06977683158707138) -- (267.49534999999906,0.06891003263668162);
\draw[line width=2pt,color=ffqqtt] (267.49534999999906,0.06891003263668162) -- (268.74537499999906,0.06805400144979314);
\draw[line width=2pt,color=ffqqtt] (268.74537499999906,0.06805400144979314) -- (269.99539999999905,0.06720860426444096);
\draw[line width=2pt,color=ffqqtt] (269.99539999999905,0.06720860426444096) -- (271.24542499999905,0.06637370898031097);
\draw[line width=2pt,color=ffqqtt] (271.24542499999905,0.06637370898031097) -- (272.49544999999904,0.0655491851380982);
\draw[line width=2pt,color=ffqqtt] (272.49544999999904,0.0655491851380982) -- (273.74547499999903,0.06473490389912125);
\draw[line width=2pt,color=ffqqtt] (273.74547499999903,0.06473490389912125) -- (274.995499999999,0.06393073802519045);
\draw[line width=2pt,color=ffqqtt] (274.995499999999,0.06393073802519045) -- (276.245524999999,0.0631365618587257);
\draw[line width=2pt,color=ffqqtt] (276.245524999999,0.0631365618587257) -- (277.495549999999,0.06235225130312143);
\draw[line width=2pt,color=ffqqtt] (277.495549999999,0.06235225130312143) -- (278.745574999999,0.061577683803355536);
\draw[line width=2pt,color=ffqqtt] (278.745574999999,0.061577683803355536) -- (279.995599999999,0.06081273832683909);
\draw[line width=2pt,color=ffqqtt] (279.995599999999,0.06081273832683909) -- (281.245624999999,0.060057295344504);
\draw[line width=2pt,color=ffqqtt] (281.245624999999,0.060057295344504) -- (282.495649999999,0.05931123681212563);
\draw[line width=2pt,color=ffqqtt] (282.495649999999,0.05931123681212563) -- (283.745674999999,0.058574446151877416);
\draw[line width=2pt,color=ffqqtt] (283.745674999999,0.058574446151877416) -- (284.995699999999,0.05784680823411453);
\draw[line width=2pt,color=ffqqtt] (284.995699999999,0.05784680823411453) -- (286.24572499999897,0.05712820935938405);
\draw[line width=2pt,color=ffqqtt] (286.24572499999897,0.05712820935938405) -- (287.49574999999896,0.0564185372406584);
\draw[line width=2pt,color=ffqqtt] (287.49574999999896,0.0564185372406584) -- (288.74577499999896,0.055717680985789574);
\draw[line width=2pt,color=ffqqtt] (288.74577499999896,0.055717680985789574) -- (289.99579999999895,0.055025531080181336);
\draw[line width=2pt,color=ffqqtt] (289.99579999999895,0.055025531080181336) -- (291.24582499999894,0.054341979369676656);
\draw[line width=2pt,color=ffqqtt] (291.24582499999894,0.054341979369676656) -- (292.49584999999894,0.05366691904365772);
\draw[line width=2pt,color=ffqqtt] (292.49584999999894,0.05366691904365772) -- (293.74587499999893,0.0530002446183559);
\draw[line width=2pt,color=ffqqtt] (293.74587499999893,0.0530002446183559) -- (294.9958999999989,0.052341851920368984);
\draw[line width=2pt,color=ffqqtt] (294.9958999999989,0.052341851920368984) -- (296.2459249999989,0.05169163807038331);
\draw[line width=2pt,color=ffqqtt] (296.2459249999989,0.05169163807038331) -- (297.4959499999989,0.05104950146709797);
\draw[line width=2pt,color=ffqqtt] (297.4959499999989,0.05104950146709797) -- (298.7459749999989,0.05041534177134877);
\draw[line width=2pt,color=ffqqtt] (298.7459749999989,0.05041534177134877) -- (299.9959999999989,0.049789059890429394);
\draw[line width=2pt,color=ffqqtt] (299.9959999999989,0.049789059890429394) -- (301.2460249999989,0.049170557962607365);
\draw[line width=2pt,color=ffqqtt] (301.2460249999989,0.049170557962607365) -- (302.4960499999989,0.04855973934183234);
\draw[line width=2pt,color=ffqqtt] (302.4960499999989,0.04855973934183234) -- (303.7460749999989,0.04795650858263436);
\draw[line width=2pt,color=ffqqtt] (303.7460749999989,0.04795650858263436) -- (304.9960999999989,0.047360771425209655);
\draw[line width=2pt,color=ffqqtt] (304.9960999999989,0.047360771425209655) -- (306.24612499999887,0.04677243478069186);
\draw[line width=2pt,color=ffqqtt] (306.24612499999887,0.04677243478069186) -- (307.49614999999886,0.04619140671660607);
\draw[line width=2pt,color=ffqqtt] (307.49614999999886,0.04619140671660607) -- (308.74617499999886,0.045617596442503586);
\draw[line width=2pt,color=ffqqtt] (308.74617499999886,0.045617596442503586) -- (309.99619999999885,0.04505091429577522);
\draw[line width=2pt,color=ffqqtt] (309.99619999999885,0.04505091429577522) -- (311.24622499999884,0.04449127172764073);
\draw[line width=2pt,color=ffqqtt] (311.24622499999884,0.04449127172764073) -- (312.49624999999884,0.04393858128931235);
\draw[line width=2pt,color=ffqqtt] (312.49624999999884,0.04393858128931235) -- (313.74627499999883,0.04339275661833019);
\draw[line width=2pt,color=ffqqtt] (313.74627499999883,0.04339275661833019) -- (314.9962999999988,0.04285371242506742);
\draw[line width=2pt,color=ffqqtt] (314.9962999999988,0.04285371242506742) -- (316.2463249999988,0.04232136447940296);
\draw[line width=2pt,color=ffqqtt] (316.2463249999988,0.04232136447940296) -- (317.4963499999988,0.041795629597559954);
\draw[line width=2pt,color=ffqqtt] (317.4963499999988,0.041795629597559954) -- (318.7463749999988,0.04127642562910753);
\draw[line width=2pt,color=ffqqtt] (318.7463749999988,0.04127642562910753) -- (319.9963999999988,0.040763671444124155);
\draw[line width=2pt,color=ffqqtt] (319.9963999999988,0.040763671444124155) -- (321.2464249999988,0.04025728692052039);
\draw[line width=2pt,color=ffqqtt] (321.2464249999988,0.04025728692052039) -- (322.4964499999988,0.03975719293151915);
\draw[line width=2pt,color=ffqqtt] (322.4964499999988,0.03975719293151915) -- (323.7464749999988,0.039263311333291515);
\draw[line width=2pt,color=ffqqtt] (323.7464749999988,0.039263311333291515) -- (324.9964999999988,0.03877556495274609);
\draw[line width=2pt,color=ffqqtt] (324.9964999999988,0.03877556495274609) -- (326.24652499999877,0.03829387757547003);
\draw[line width=2pt,color=ffqqtt] (326.24652499999877,0.03829387757547003) -- (327.49654999999876,0.03781817393381999);
\draw[line width=2pt,color=ffqqtt] (327.49654999999876,0.03781817393381999) -- (328.74657499999876,0.037348379695160895);
\draw[line width=2pt,color=ffqqtt] (328.74657499999876,0.037348379695160895) -- (329.99659999999875,0.03688442145025083);
\draw[line width=2pt,color=ffqqtt] (329.99659999999875,0.03688442145025083) -- (331.24662499999874,0.03642622670177024);
\draw[line width=2pt,color=ffqqtt] (331.24662499999874,0.03642622670177024) -- (332.49664999999874,0.03597372385299366);
\draw[line width=2pt,color=ffqqtt] (332.49664999999874,0.03597372385299366) -- (333.74667499999873,0.03552684219660209);
\draw[line width=2pt,color=ffqqtt] (333.74667499999873,0.03552684219660209) -- (334.9966999999987,0.03508551190363446);
\draw[line width=2pt,color=ffqqtt] (334.9966999999987,0.03508551190363446) -- (336.2467249999987,0.034649664012576124);
\draw[line width=2pt,color=ffqqtt] (336.2467249999987,0.034649664012576124) -- (337.4967499999987,0.034219230418583255);
\draw[line width=2pt,color=ffqqtt] (337.4967499999987,0.034219230418583255) -- (338.7467749999987,0.03379414386284075);
\draw[line width=2pt,color=ffqqtt] (338.7467749999987,0.03379414386284075) -- (339.9967999999987,0.03337433792205255);
\draw[line width=2pt,color=ffqqtt] (339.9967999999987,0.03337433792205255) -- (341.2468249999987,0.03295974699806241);
\draw[line width=2pt,color=ffqqtt] (341.2468249999987,0.03295974699806241) -- (342.4968499999987,0.032550306307603694);
\draw[line width=2pt,color=ffqqtt] (342.4968499999987,0.032550306307603694) -- (343.7468749999987,0.0321459518721764);
\draw[line width=2pt,color=ffqqtt] (343.7468749999987,0.0321459518721764) -- (344.9968999999987,0.03174662050805002);
\draw[line width=2pt,color=ffqqtt] (344.9968999999987,0.03174662050805002) -- (346.24692499999867,0.031352249816390565);
\draw[line width=2pt,color=ffqqtt] (346.24692499999867,0.031352249816390565) -- (347.49694999999866,0.030962778173510182);
\draw[line width=2pt,color=ffqqtt] (347.49694999999866,0.030962778173510182) -- (348.74697499999866,0.030578144721238012);
\draw[line width=2pt,color=ffqqtt] (348.74697499999866,0.030578144721238012) -- (349.99699999999865,0.030198289357410542);
\draw[line width=2pt,color=ffqqtt] (349.99699999999865,0.030198289357410542) -- (351.24702499999864,0.029823152726480177);
\draw[line width=2pt,color=ffqqtt] (351.24702499999864,0.029823152726480177) -- (352.49704999999864,0.029452676210240426);
\draw[line width=2pt,color=ffqqtt] (352.49704999999864,0.029452676210240426) -- (353.74707499999863,0.02908680191866632);
\draw[line width=2pt,color=ffqqtt] (353.74707499999863,0.02908680191866632) -- (354.9970999999986,0.028725472680868624);
\draw[line width=2pt,color=ffqqtt] (354.9970999999986,0.028725472680868624) -- (356.2471249999986,0.028368632036160434);
\draw[line width=2pt,color=ffqqtt] (356.2471249999986,0.028368632036160434) -- (357.4971499999986,0.02801622422523465);
\draw[line width=2pt,color=ffqqtt] (357.4971499999986,0.02801622422523465) -- (358.7471749999986,0.027668194181451222);
\draw[line width=2pt,color=ffqqtt] (358.7471749999986,0.027668194181451222) -- (359.9971999999986,0.027324487522232474);
\draw[line width=2pt,color=ffqqtt] (359.9971999999986,0.027324487522232474) -- (361.2472249999986,0.026985050540565377);
\draw[line width=2pt,color=ffqqtt] (361.2472249999986,0.026985050540565377) -- (362.4972499999986,0.026649830196609398);
\draw[line width=2pt,color=ffqqtt] (362.4972499999986,0.026649830196609398) -- (363.7472749999986,0.026318774109408583);
\draw[line width=2pt,color=ffqqtt] (363.7472749999986,0.026318774109408583) -- (364.9972999999986,0.02599183054870659);
\draw[line width=2pt,color=ffqqtt] (364.9972999999986,0.02599183054870659) -- (366.24732499999857,0.025668948426863434);
\draw[line width=2pt,color=ffqqtt] (366.24732499999857,0.025668948426863434) -- (367.49734999999856,0.025350077290872554);
\draw[line width=2pt,color=ffqqtt] (367.49734999999856,0.025350077290872554) -- (368.74737499999856,0.0250351673144772);
\draw[line width=2pt,color=ffqqtt] (368.74737499999856,0.0250351673144772) -- (369.99739999999855,0.02472416929038461);
\draw[line width=2pt,color=ffqqtt] (369.99739999999855,0.02472416929038461) -- (371.24742499999854,0.024417034622576984);
\draw[line width=2pt,color=ffqqtt] (371.24742499999854,0.024417034622576984) -- (372.49744999999854,0.024113715318717945);
\draw[line width=2pt,color=ffqqtt] (372.49744999999854,0.024113715318717945) -- (373.74747499999853,0.023814163982653343);
\draw[line width=2pt,color=ffqqtt] (373.74747499999853,0.023814163982653343) -- (374.9974999999985,0.023518333807005215);
\draw[line width=2pt,color=ffqqtt] (374.9974999999985,0.023518333807005215) -- (376.2475249999985,0.023226178565857736);
\draw[line width=2pt,color=ffqqtt] (376.2475249999985,0.023226178565857736) -- (377.4975499999985,0.022937652607534057);
\draw[line width=2pt,color=ffqqtt] (377.4975499999985,0.022937652607534057) -- (378.7475749999985,0.022652710847462808);
\draw[line width=2pt,color=ffqqtt] (378.7475749999985,0.022652710847462808) -- (379.9975999999985,0.02237130876113332);
\draw[line width=2pt,color=ffqqtt] (379.9975999999985,0.02237130876113332) -- (381.2476249999985,0.02209340237713826);
\draw[line width=2pt,color=ffqqtt] (381.2476249999985,0.02209340237713826) -- (382.4976499999985,0.021818948270302743);
\draw[line width=2pt,color=ffqqtt] (382.4976499999985,0.021818948270302743) -- (383.7476749999985,0.0215479035548988);
\draw[line width=2pt,color=ffqqtt] (383.7476749999985,0.0215479035548988) -- (384.9976999999985,0.021280225877944113);
\draw[line width=2pt,color=ffqqtt] (384.9976999999985,0.021280225877944113) -- (386.24772499999847,0.021015873412584013);
\draw[line width=2pt,color=ffqqtt] (386.24772499999847,0.021015873412584013) -- (387.49774999999846,0.02075480485155571);
\draw[line width=2pt,color=ffqqtt] (387.49774999999846,0.02075480485155571) -- (388.74777499999846,0.020496979400733636);
\draw[line width=2pt,color=ffqqtt] (388.74777499999846,0.020496979400733636) -- (389.99779999999845,0.020242356772755096);
\draw[line width=2pt,color=ffqqtt] (389.99779999999845,0.020242356772755096) -- (391.24782499999844,0.01999089718072498);
\draw[line width=2pt,color=ffqqtt] (391.24782499999844,0.01999089718072498) -- (392.49784999999844,0.01974256133199876);
\draw[line width=2pt,color=ffqqtt] (392.49784999999844,0.01974256133199876) -- (393.74787499999843,0.019497310422042675);
\draw[line width=2pt,color=ffqqtt] (393.74787499999843,0.019497310422042675) -- (394.9978999999984,0.019255106128370204);
\draw[line width=2pt,color=ffqqtt] (394.9978999999984,0.019255106128370204) -- (396.2479249999984,0.01901591060455386);
\draw[line width=2pt,color=ffqqtt] (396.2479249999984,0.01901591060455386) -- (397.4979499999984,0.01877968647431138);
\draw[line width=2pt,color=ffqqtt] (397.4979499999984,0.01877968647431138) -- (398.7479749999984,0.018546396825665338);
\draw[line width=2pt,color=ffqqtt] (398.7479749999984,0.018546396825665338) -- (399.9979999999984,0.018316005205175408);
\draw[line width=2pt,color=ffqqtt] (399.9979999999984,0.018316005205175408) -- (401.2480249999984,0.018088475612242135);
\draw[line width=2pt,color=ffqqtt] (401.2480249999984,0.018088475612242135) -- (402.4980499999984,0.01786377249348163);
\draw[line width=2pt,color=ffqqtt] (402.4980499999984,0.01786377249348163) -- (403.7480749999984,0.01764186073716996);
\draw[line width=2pt,color=ffqqtt] (403.7480749999984,0.01764186073716996) -- (404.9980999999984,0.017422705667756704);
\draw[line width=2pt,color=ffqqtt] (404.9980999999984,0.017422705667756704) -- (406.24812499999837,0.017206273040446658);
\draw[line width=2pt,color=ffqqtt] (406.24812499999837,0.017206273040446658) -- (407.49814999999836,0.016992529035848684);
\draw[line width=2pt,color=ffqqtt] (407.49814999999836,0.016992529035848684) -- (408.74817499999835,0.01678144025469128);
\draw[line width=2pt,color=ffqqtt] (408.74817499999835,0.01678144025469128) -- (409.99819999999835,0.01657297371260354);
\draw[line width=2pt,color=ffqqtt] (409.99819999999835,0.01657297371260354) -- (411.24822499999834,0.01636709683496121);
\draw[line width=2pt,color=ffqqtt] (411.24822499999834,0.01636709683496121) -- (412.49824999999834,0.016163777451796483);
\draw[line width=2pt,color=ffqqtt] (412.49824999999834,0.016163777451796483) -- (413.74827499999833,0.015962983792771328);
\draw[line width=2pt,color=ffqqtt] (413.74827499999833,0.015962983792771328) -- (414.9982999999983,0.01576468448221298);
\draw[line width=2pt,color=ffqqtt] (414.9982999999983,0.01576468448221298) -- (416.2483249999983,0.01556884853421129);
\draw[line width=2pt,color=ffqqtt] (416.2483249999983,0.01556884853421129) -- (417.4983499999983,0.015375445347776965);
\draw[line width=2pt,color=ffqqtt] (417.4983499999983,0.015375445347776965) -- (418.7483749999983,0.015184444702059801);
\draw[line width=2pt,color=ffqqtt] (418.7483749999983,0.015184444702059801) -- (419.9983999999983,0.01499581675162654);
\draw[line width=2pt,color=ffqqtt] (419.9983999999983,0.01499581675162654) -- (421.2484249999983,0.014809532021797163);
\draw[line width=2pt,color=ffqqtt] (421.2484249999983,0.014809532021797163) -- (422.4984499999983,0.014625561404039348);
\draw[line width=2pt,color=ffqqtt] (422.4984499999983,0.014625561404039348) -- (423.7484749999983,0.014443876151419899);
\draw[line width=2pt,color=ffqqtt] (423.7484749999983,0.014443876151419899) -- (424.9984999999983,0.014264447874112874);
\draw[line width=2pt,color=ffqqtt] (424.9984999999983,0.014264447874112874) -- (426.24852499999827,0.014087248534963447);
\draw[line width=2pt,color=ffqqtt] (426.24852499999827,0.014087248534963447) -- (427.49854999999826,0.013912250445106801);
\draw[line width=2pt,color=ffqqtt] (427.49854999999826,0.013912250445106801) -- (428.74857499999825,0.013739426259641611);
\draw[line width=2pt,color=ffqqtt] (428.74857499999825,0.013739426259641611) -- (429.99859999999825,0.013568748973357065);
\draw[line width=2pt,color=ffqqtt] (429.99859999999825,0.013568748973357065) -- (431.24862499999824,0.013400191916513189);
\draw[line width=2pt,color=ffqqtt] (431.24862499999824,0.013400191916513189) -- (432.49864999999824,0.013233728750673378);
\draw[line width=2pt,color=ffqqtt] (432.49864999999824,0.013233728750673378) -- (433.74867499999823,0.013069333464588887);
\draw[line width=2pt,color=ffqqtt] (433.74867499999823,0.013069333464588887) -- (434.9986999999982,0.012906980370134283);
\draw[line width=2pt,color=ffqqtt] (434.9986999999982,0.012906980370134283) -- (436.2487249999982,0.01274664409829349);
\draw[line width=2pt,color=ffqqtt] (436.2487249999982,0.01274664409829349) -- (437.4987499999982,0.012588299595195713);
\draw[line width=2pt,color=ffqqtt] (437.4987499999982,0.012588299595195713) -- (438.7487749999982,0.012431922118200477);
\draw[line width=2pt,color=ffqqtt] (438.7487749999982,0.012431922118200477) -- (439.9987999999982,0.012277487232031476);
\draw[line width=2pt,color=ffqqtt] (439.9987999999982,0.012277487232031476) -- (441.2488249999982,0.012124970804958272);
\draw[line width=2pt,color=ffqqtt] (441.2488249999982,0.012124970804958272) -- (442.4988499999982,0.011974349005025598);
\draw[line width=2pt,color=ffqqtt] (442.4988499999982,0.011974349005025598) -- (443.7488749999982,0.011825598296329329);
\draw[line width=2pt,color=ffqqtt] (443.7488749999982,0.011825598296329329) -- (444.9988999999982,0.011678695435338891);
\draw[line width=2pt,color=ffqqtt] (444.9988999999982,0.011678695435338891) -- (446.24892499999817,0.011533617467265193);
\draw[line width=2pt,color=ffqqtt] (446.24892499999817,0.011533617467265193) -- (447.49894999999816,0.011390341722473779);
\draw[line width=2pt,color=ffqqtt] (447.49894999999816,0.011390341722473779) -- (448.74897499999815,0.011248845812942544);
\draw[line width=2pt,color=ffqqtt] (448.74897499999815,0.011248845812942544) -- (449.99899999999815,0.011109107628763347);
\draw[line width=2pt,color=ffqqtt] (449.99899999999815,0.011109107628763347) -- (451.24902499999814,0.010971105334687233);
\draw[line width=2pt,color=ffqqtt] (451.24902499999814,0.010971105334687233) -- (452.49904999999814,0.010834817366712421);
\draw[line width=2pt,color=ffqqtt] (452.49904999999814,0.010834817366712421) -- (453.74907499999813,0.010700222428714825);
\draw[line width=2pt,color=ffqqtt] (453.74907499999813,0.010700222428714825) -- (454.9990999999981,0.010567299489120286);
\draw[line width=2pt,color=ffqqtt] (454.9990999999981,0.010567299489120286) -- (456.2491249999981,0.010436027777618255);
\draw[line width=2pt,color=ffqqtt] (456.2491249999981,0.010436027777618255) -- (457.4991499999981,0.010306386781916272);
\draw[line width=2pt,color=ffqqtt] (457.4991499999981,0.010306386781916272) -- (458.7491749999981,0.010178356244534701);
\draw[line width=2pt,color=ffqqtt] (458.7491749999981,0.010178356244534701) -- (459.9991999999981,0.010051916159641395);
\draw[line width=2pt,color=ffqqtt] (459.9991999999981,0.010051916159641395) -- (461.2492249999981,0.009927046769925545);
\draw[line width=2pt,color=ffqqtt] (461.2492249999981,0.009927046769925545) -- (462.4992499999981,0.009803728563510517);
\draw[line width=2pt,color=ffqqtt] (462.4992499999981,0.009803728563510517) -- (463.7492749999981,0.00968194227090488);
\draw[line width=2pt,color=ffqqtt] (463.7492749999981,0.00968194227090488) -- (464.9992999999981,0.009561668861991468);
\draw[line width=2pt,color=ffqqtt] (464.9992999999981,0.009561668861991468) -- (466.24932499999807,0.009442889543053692);
\draw[line width=2pt,color=ffqqtt] (466.24932499999807,0.009442889543053692) -- (467.49934999999806,0.009325585753838904);
\draw[line width=2pt,color=ffqqtt] (467.49934999999806,0.009325585753838904) -- (468.74937499999805,0.009209739164658226);
\draw[line width=2pt,color=ffqqtt] (468.74937499999805,0.009209739164658226) -- (469.99939999999805,0.009095331673522321);
\draw[line width=2pt,color=ffqqtt] (469.99939999999805,0.009095331673522321) -- (471.24942499999804,0.008982345403312879);
\draw[line width=2pt,color=ffqqtt] (471.24942499999804,0.008982345403312879) -- (472.49944999999803,0.008870762698989103);
\draw[line width=2pt,color=ffqqtt] (472.49944999999803,0.008870762698989103) -- (473.74947499999803,0.008760566124829025);
\draw[line width=2pt,color=ffqqtt] (473.74947499999803,0.008760566124829025) -- (474.999499999998,0.008651738461704965);
\draw[line width=2pt,color=ffqqtt] (474.999499999998,0.008651738461704965) -- (476.249524999998,0.008544262704392968);
\draw[line width=2pt,color=ffqqtt] (476.249524999998,0.008544262704392968) -- (477.499549999998,0.008438122058915535);
\draw[line width=2pt,color=ffqqtt] (477.499549999998,0.008438122058915535) -- (478.749574999998,0.008333299939917459);
\draw[line width=2pt,color=ffqqtt] (478.749574999998,0.008333299939917459) -- (479.999599999998,0.008229779968074231);
\draw[line width=2pt,color=ffqqtt] (479.999599999998,0.008229779968074231) -- (481.249624999998,0.008127545967532606);
\draw[line width=2pt,color=ffqqtt] (481.249624999998,0.008127545967532606) -- (482.499649999998,0.00802658196338303);
\draw[line width=2pt,color=ffqqtt] (482.499649999998,0.00802658196338303) -- (483.749674999998,0.007926872179163377);
\draw[line width=2pt,color=ffqqtt] (483.749674999998,0.007926872179163377) -- (484.999699999998,0.007828401034393808);
\draw[line width=2pt,color=ffqqtt] (484.999699999998,0.007828401034393808) -- (486.24972499999797,0.007731153142142139);
\draw[line width=2pt,color=ffqqtt] (486.24972499999797,0.007731153142142139) -- (487.49974999999796,0.007635113306619526);
\draw[line width=2pt,color=ffqqtt] (487.49974999999796,0.007635113306619526) -- (488.74977499999795,0.007540266520806015);
\draw[line width=2pt,color=ffqqtt] (488.74977499999795,0.007540266520806015) -- (489.99979999999795,0.007446597964105529);
\draw[line width=2pt,color=ffqqtt] (489.99979999999795,0.007446597964105529) -- (491.24982499999794,0.007354093000030077);
\draw[line width=2pt,color=ffqqtt] (491.24982499999794,0.007354093000030077) -- (492.49984999999793,0.00726273717391263);
\draw[line width=2pt,color=ffqqtt] (492.49984999999793,0.00726273717391263) -- (493.7498749999979,0.00717251621064851);
\draw[line width=2pt,color=ffqqtt] (493.7498749999979,0.00717251621064851) -- (494.9998999999979,0.007083416012464737);
\draw[line width=2pt,color=ffqqtt] (494.9998999999979,0.007083416012464737) -- (496.2499249999979,0.006995422656717185);
\draw[line width=2pt,color=ffqqtt] (496.2499249999979,0.006995422656717185) -- (497.4999499999979,0.006908522393715011);
\draw[line width=2pt,color=ffqqtt] (497.4999499999979,0.006908522393715011) -- (498.7499749999979,0.006822701644572169);
\draw[line width=2pt,color=ffqqtt] (498.7499749999979,0.006822701644572169) -- (499.9999999999979,0.006737946999085613);
\begin{scriptsize}
\draw[color=ffqqtt] (42.233548387096775,0.9016509433962259) node {$\sigma^2$};
\end{scriptsize}
\end{axis}
\end{tikzpicture}
\caption{Policy variance as a function of optimization time.}
\label{fig-pg-asymp}
\end{subfigure}
\caption{Policy Gradient on a bandit problem across optimization time.}
\end{figure}

To provide additional intuition about Lemma \ref{lem-hexplore}, we now analyze how Gaussian Policy Gradients behave across optimization time. Consider a simple bandit task shown in Figure \ref{fig-cb-def}, where there is only one state and the critic is defined by $\hat{Q}(a) = -a^2$. Applying equation \eqref{eq-cov-gpg}, we have that, as the policy gradient update is applied again and again, the variance evolves according to the equation $\dot{\sigma} = - \sigma$. Assuming the initial value of $\sigma$ is 1, the standard deviation at time $t$ is given by $e^{-t}$. In this equation, $t$ is optimization time, which is external to the bandit problem. 

Figure \ref{fig-pg-asymp} plots the standard deviation as a function of optimization time. The plot confirms the intuition that policy gradient algorithms try to find the policy with the maximum expected value of $\hat{Q}$. In particular, $\sigma^2$ decays to zero because the critic $\hat{Q}$ has one maximum $a=0$, where the optimal policy puts all the probability mass. This is also consistent with Lemma \ref{lem-hexplore}, which corresponds to fixing $t=1$. 

\subsection{Action Clipping}
\label{sec-gpg-clip}
We now describe how GPG works in environments where the action space has bounded support\footnote{We assume without loss of generality that the support interval is $[0,1]$.}. This setting occurs frequently in practice, since real-world systems often have physical constraints such as a bound on how fast a robot arm can accelerate. The typical solution to this problem is simply to start with a policy $\pi_b$ with unbounded support and then, when an action is to be taken, clip it to the desired range, so that sampling an action 
\begin{gather*}
  a \sim \pi (a \mid s) \quad \text{is equivalent to} \quad a = \max(\min(b,1),0) \;\; \text{with} \;\; b \sim \pi_b(b \mid s).
\end{gather*}
The justification for this process is that we can simply treat the clipping operation $\max(\min(b,1),0)$ as part of the environment specification. Formally, this means that we transform the original MDP $M$ defined as $M = (S,A,R_d,p,p_0,\gamma)$ with $A=[0,1]^d$ into another MDP $M' = (S,A',R_d',p',p_0,\gamma)$, where $A' = \mathbb{R}^d$ and $p'$ and $R_d'$ are defined as
\[
p'(s' \vert b, s) \; = \; p \left (s' \vert \max(\min(b,1),0), s \right ) \quad \text{and} \quad R_d'(r \vert b, s) \; = \; R_d \left (r \vphantom{s'} \vert \max(\min(b,1),0), s \right) \!.
\]
Since $M'$ has an unbounded action space, we can use the RL machinery for unbounded actions to solve it. Since any MDP is guaranteed to have an optimal deterministic policy, we call this deterministic solution $\pi_D^\star : S \rightarrow A$. Now, $\pi_D^\star$ can be transformed into a policy for $M$ of the form $ \max(\min(\pi_D^\star(s),1),0)$. In practice, the MDP $M'$ is never constructed explicitly---the described process in equivalent to using an RL algorithm meant for $A = \mathbb{R}^d$ and then, when the action is generated, simply clipping it (Algorithm \ref{alg-epg-clip}).

\begin{algorithm}[t]
\begin{algorithmic}[1]
 \State $s \gets s_0$, $t \gets 0$
 \State initialize optimizer, initialize policy $\pi$ parameterized by $\theta$
\While{not converged}
 \label{line-rep-pg} \State $g_t \gets \gamma^t$ \textsc{do-integral}($\hat{Q_b}, s, \pi_\theta $)
 \State $\theta \gets  \theta \; + \;  $optimizer.\textsc{update}$(g_t) $
 \State $b \sim \pi(\cdot \mid s)$
 \State $a = c(b)$ \Comment{Clipping function $c(b) = \max(\min(\pi_D^\star(s),1),0)$.}
 \State $s',r \gets $ simulator.\textsc{perform-action}(a)
 \State $\hat{Q_b}$.\textsc{update}($s,b,r,s'$) \Comment{Update using the unclipped action $b$.}
 \State $t \gets t + 1$
 \State $s \gets s'$
\EndWhile

\end{algorithmic}
\caption{ Policy gradients with clipped actions. } \label{alg-epg-clip}
\end{algorithm}

However, while such an algorithm does not introduce new bias in the sense that reward obtained in $M$ and $M'$ will be the same, it can lead to problems with slow convergence in the policy gradient settings. To see why, consider the one-dimensional example in Figure \ref{fig-vanishing-gradient}, where the policy mean is located far away from the clipping boundary. This can arise due to a combination of random initialization of the policy network and generalization error across states. 

\begin{figure}[tb]
    \centering
        \begingroup%
  \makeatletter%
  \providecommand\color[2][]{%
    \errmessage{(Inkscape) Color is used for the text in Inkscape, but the package 'color.sty' is not loaded}%
    \renewcommand\color[2][]{}%
  }%
  \providecommand\transparent[1]{%
    \errmessage{(Inkscape) Transparency is used (non-zero) for the text in Inkscape, but the package 'transparent.sty' is not loaded}%
    \renewcommand\transparent[1]{}%
  }%
  \providecommand\rotatebox[2]{#2}%
  \ifx\svgwidth\undefined%
    \setlength{\unitlength}{361.29314756bp}%
    \ifx\svgscale\undefined%
      \relax%
    \else%
      \setlength{\unitlength}{\unitlength * \real{\svgscale}}%
    \fi%
  \else%
    \setlength{\unitlength}{\svgwidth}%
  \fi%
  \global\let\svgwidth\undefined%
  \global\let\svgscale\undefined%
  \makeatother%
  \begin{picture}(1,0.38305564)%
    \put(0,0){\includegraphics[width=\unitlength,page=1]{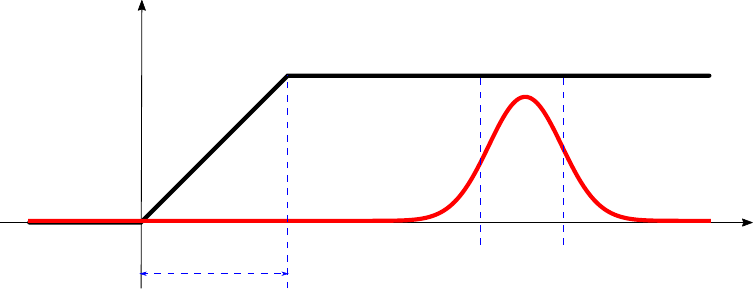}}%
    \put(0.2161753,0.19420357){\color[rgb]{0,0,0}\makebox(0,0)[lb]{\smash{$g(b)$}}}%
    \put(0.94368927,0.06069099){\color[rgb]{0,0,0}\makebox(0,0)[lb]{\smash{$b$}}}%
    \put(0.74873134,0.19200982){\color[rgb]{0,0,0}\makebox(0,0)[lb]{\smash{$\pi$}}}%
    \put(0.19618225,0.02937654){\color[rgb]{0,0,0}\makebox(0,0)[lb]{\smash{$\{b : g'(b) \neq 0\}$}}}%
    \put(0.74723613,0.06203558){\color[rgb]{0,0,0}\makebox(0,0)[lb]{\smash{$b_2$}}}%
    \put(0.63713033,0.06102545){\color[rgb]{0,0,0}\makebox(0,0)[lb]{\smash{$b_1$}}}%
    \put(0.38267881,0.06150414){\color[rgb]{0,0,0}\makebox(0,0)[lb]{\smash{$b_L$}}}%
  \end{picture}%
\endgroup%

        \caption{Vanishing gradients when using hard clipping. The agent cannot determine whether $b$ is too small or too large from $b_1$ and $b_2$ alone. It is necessary to sample from the interval $\{b : g'(b) \neq 0\}$ in order to obtain a meaningful policy update but this is  unlikely for the current policy (shown as the red curve). }
        \label{fig-vanishing-gradient}
\end{figure}

With hard clipping, the agent cannot distinguish between $b_1$ and $b_2$ since squashing reduces them both to the same value, i.e., $g(b_1) = g(b_2)$. Hence, the corresponding $Q$ values are identical and, based on trajectories using $b_1$ and $b_2$, there is no way of knowing how the mean of the policy should be adjusted. In order to get a useful gradient that moves the distribution into the interval $[0, b_L]$, a sample $b_\star < b_L$ has to be chosen. Since the $b$'s are samples from a Gaussian with infinite support, it will eventually happen, yielding a nonzero gradient. However, if this interval falls into a distant part of the tail of $\pi_b$, convergence will be slow.

\begin{figure}[tb]
    \centering
        \begingroup%
  \makeatletter%
  \providecommand\color[2][]{%
    \errmessage{(Inkscape) Color is used for the text in Inkscape, but the package 'color.sty' is not loaded}%
    \renewcommand\color[2][]{}%
  }%
  \providecommand\transparent[1]{%
    \errmessage{(Inkscape) Transparency is used (non-zero) for the text in Inkscape, but the package 'transparent.sty' is not loaded}%
    \renewcommand\transparent[1]{}%
  }%
  \providecommand\rotatebox[2]{#2}%
  \ifx\svgwidth\undefined%
    \setlength{\unitlength}{361.29314756bp}%
    \ifx\svgscale\undefined%
      \relax%
    \else%
      \setlength{\unitlength}{\unitlength * \real{\svgscale}}%
    \fi%
  \else%
    \setlength{\unitlength}{\svgwidth}%
  \fi%
  \global\let\svgwidth\undefined%
  \global\let\svgscale\undefined%
  \makeatother%
  \begin{picture}(1,0.38305564)%
    \put(0,0){\includegraphics[width=\unitlength,page=1]{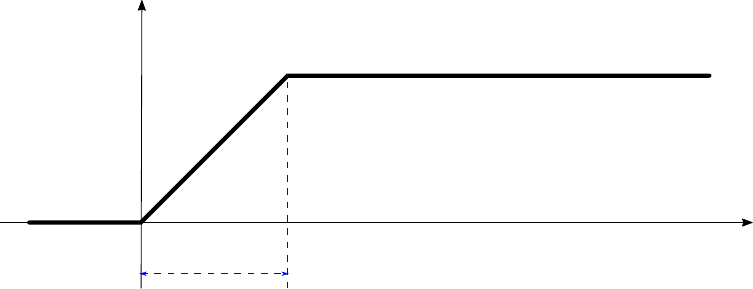}}%
    \put(0.2161753,0.19420357){\color[rgb]{0,0,0}\makebox(0,0)[lb]{\smash{$g(b)$}}}%
    \put(0.94368927,0.06069099){\color[rgb]{0,0,0}\makebox(0,0)[lb]{\smash{$b$}}}%
    \put(0.19618225,0.02937654){\color[rgb]{0,0,0}\makebox(0,0)[lb]{\smash{$\{b : g'(b) \neq 0\}$}}}%
    \put(0.38267881,0.06150414){\color[rgb]{0,0,0}\makebox(0,0)[lb]{\smash{$b_L$}}}%
    \put(0,0){\includegraphics[width=\unitlength,page=2]{gpg-clipping.pdf}}%
  \end{picture}%
\endgroup%

        \caption{GPG avoids the vanishing gradient problem. Once a policy with a small variance (denoted in red) enters the flat area where $ b > b_L$, exploration noise is set to $\sigma_0$ (the new distribution is in blue). }
        \label{fig-clipped-gradient-gpg}
\end{figure}

GPG mitigates this problem because it always explores with standard deviation $\sigma_0$ in the flat region. This is shown in Figure \ref{fig-clipped-gradient-gpg}. For actions $b > b_L$, the critic is constant. Since a constant critic has a zero Hessian, the standard deviation of the policy is set to $\sigma_0 e^0 = \sigma_0$, which makes it likely that a point $b < b_L$ is sampled and a useful gradient is obtained for typical values of $\sigma_0$. In contrast, consider the case where stochastic policy gradients were used to tune the standard deviation. It follows from setting $A(s) = 0$ in \eqref{eq-cov-gpg}) that for SPG, the policy gradient would be zero in expectation. Then, if the policy standard deviation is initialized to a small value, it likely to to remain stuck and take prohibitively long to leave the plateau. Another way of mitigating the hard clipping problem is to use a differentiable squashing function, which we describe in Section \ref{sec-expfamily}.

\subsection{Quadratic Critics and their Approximations}
Gaussian policy gradients require a quadratic critic \emph{given the state}. This assumption, which is different from assuming a quadratic dependency \emph{on the state}, is typically sufficient for two reasons. First, discrete-time linear quadratic regulators (LQR) with time-varying feedback, a class of problems widely studied in classical control\footnote{Indeed, the Hessian discussed in Section \ref{ss-hessian-exp} can be considered a type of reward model. } theory, are known to have a $Q$-function that is quadratic in the action vector given the state 
(\citealt[Equation 1]{bradtkeReinforcementLearningApplied1993};
\citealt[Equation 10]{tenhagens.h.g.LinearQuadraticRegulation1998};
\citealt{petersPolicyGradientMethods2003};
\citealt[Equation 8.81]{crassidis2011optimal}). Second, it is often assumed \citep{li2004iterative} that a quadratic critic (or a quadratic approximation to a general critic) is enough to capture enough local structure to perform a policy optimization step, in much the same way as Newton's method for deterministic unconstrained optimization, which locally approximates a function with a quadratic, can be used to optimize a non-quadratic function across several iterations. In Corollary \ref{lem-agpg} below, we describe such an approximation method applied to GPG where we approximate $Q$ with a quadratic function in the neighborhood of the policy mean.
\begin{corollary}[Approximate Gaussian Policy Gradients with Arbitrary Critic]
\label{lem-agpg} 
\ \\
If the policy is Gaussian, i.e. $\pi(\cdot \vert s) \sim \mathcal{N}(\mu_s,\Sigma^{1/2}_s)$ with $\mu_s$ and $\Sigma^{1/2}_s$ parameterized by $\theta$ as in Lemma \ref{lem-gpg} and any critic $\hat{Q}(a,s)$ doubly differentiable with respect to actions for each state, then $ I^{\hat{Q}}_{\pi(s), \mu_s} \approx (\nabla_\theta  \mu_s) \nabla_a {\hat{Q}}(a = \mu_s, s) $ and $I^{\hat{Q}}_{\pi(s), \Sigma^{1/2}_s} \approx (\nabla_\theta  \Sigma^{1/2}_s) H(\mu_s, s) \Sigma^{1/2}_s$, where $H(\mu_s, s)$ is the Hessian of ${\hat{Q}}$ with respect to $a$, evaluated at $\mu_s$ for a fixed $s$.
\end{corollary}
\begin{proof}
We begin by approximating the critic (for a given $s$) using the first two terms of the Taylor expansion of ${\hat{Q}}$ in $\mu_s$. This allows us to approximate the critic as
\begin{align*}
    \textstyle {\hat{Q}}(a,s) &  \textstyle \approx {\hat{Q}}(\mu_s, s) + (a - \mu_s)^\top \atmus{\nabla_a {\hat{Q}} (a , s)} + \frac{1}{2} (a - \mu_s)^\top H(\mu_s, s) (a - \mu_s) \\
    & \textstyle = \frac{1}{2} a^\top H(\mu_s, s) a + a^\top \left( \atmus{\nabla_a {\hat{Q}} (a, s)} - H(\mu_s, s) \mu_s \right)  + \text{const}_s.
\end{align*}
We used the notation $\text{const}_s$ to denote a term constant in the action (but dependent on the state). Because of the series truncation, the function on the right-hand side is quadratic and we can then use Lemma \ref{lem-gpg}, obtaining
\begin{align*}
    \textstyle I^{\hat{Q}}_{\pi(s), \mu_s} &\textstyle = \nabla_\theta  \mu_s (2 \frac{1}{2}  H(\mu_s, s) \mu_s + \atmus{\nabla_a {\hat{Q}} (a, s)} - H(\mu_s, s) \mu_s) \\
    &\textstyle = \nabla_\theta  \mu_s \atmus{ \nabla_a {\hat{Q}} (a, s)}, \;\;\text{and}  \\
    \textstyle I^{\hat{Q}}_{\pi(s), \Sigma^{1/2}_s} &\textstyle = \nabla_{\Sigma^{1/2}_s} (2 \frac{1}{2} H(\mu_s, s) \Sigma^{1/2}_s) = \nabla_{\Sigma^{1/2}_s}  H(\mu_s, s) \Sigma^{1/2}_s.
\end{align*} 
\end{proof}
To actually obtain the Hessian, we could use automatic differentiation to compute it analytically. Sometimes this may not be possible---for example when ReLU units are used, the Hessian is always zero or undefined. In these cases, we can approximate the Hessian by generating a number of random action-values around $\mu_s$, computing the ${\hat{Q}}$ values, and (locally) fitting a quadratic, akin to sigma-point methods in control \citep{roth2016nonlinear}.

\section{Universal Expected Policy Gradients}
\label{sec-expfamily}
Having covered the most common case of continuous Gaussian policies, we now extend the analysis to other policy classes. We provide two cases of such results in the following sections: exponential family policies with multivariate polynomial critics (of arbitrary order) and arbitrary policies (possessing a mean) with linear critics. Our main claim is that an analytic solution to the EPG integral is possible for almost any system; hence we describe EPG as a \emph{universal} method.\footnote{Of course no method can be truly universal for a \emph{completely arbitrary} problem. Our claim is that EPG is universal \emph{for the class of systems} arising from lemmas in this section. However, this class is so broad that we feel the term `universal` is justified. This is similar to the claim that neural networks based on sigmoid nonlinearities are universal, even though then can only approximate continuous functions, as opposed to completely arbitrary ones. }

\subsection{Exponential Family Policies and Polynomial Critics}
\label{sec-universal}
We now describe a general technique to obtain analytic EPG updates for the case when the policy belongs to a certain exponential family and the critic is an arbitrary polynomial. This result is significant since polynomials can approximate any continuous function on a bounded interval with arbitrary accuracy \citep{weierstrass1885analytische, stone1948generalized}. Since our result holds for a nontrivial class of distributions in the exponential family, it implies that analytic solutions for EPG can almost always be obtained in practice and hence that the Monte Carlo sampling to estimate the inner integral that is typical in SPG is not necessary in many cases.

\begin{restatable}[EPG for Exponential Families w. Polynomial Sufficient Statistics]{lmm}{epgexpfamily}\label{lem-epg-ef}\ \\
Consider the class of policies parameterized by $\theta$ defined by the formula
\begin{gather*}
\label{pc-exp-fam-c}
\pi( a \mid s) = e^{(\est)^\top T^s(a) - U^s_{\est} + W^s(a)},
\end{gather*}
where each entry in the vector $T^s(a)$ is a (possibly multivariate) polynomial in the entries of the vector $a$. Moreover, assume that the critic $\hat{Q}(a, s)$ is (a possibly multivariate) polynomial in the entries of $a$. Then, the policy gradient update is a closed form expression in terms of the uncentered moments of $\pi(\cdot \mid s)$ and can be written as
\begin{gather*}
\label{poly-epg-cf}
I^Q_\pi(s) = (\nabla_\theta  (\est)^\top) ((C^s_{TQ})^\top m_\pi) - (\nabla_\theta  U^s_{\est}) ((C^s_{Q})^\top m_\pi),
\end{gather*}
where $C^s_Q$ is the vector containing the coefficients of the polynomial ${\hat{Q}}(\cdot,s)$, $C^s_{TQ}$ is the vector containing the coefficients of the polynomial $T^s(a) {\hat{Q}}(a, s)$, i.e., a multiplication of $T^s$ and ${\hat{Q}}(a, s)$. Moreover, $m_\pi$ is a vector of uncentered moments of $\pi$ (in the order matching the polynomials).
\end{restatable}

The lemma is proved in the Appendix. The cross-moments themselves can be obtained from the \emph{moment generating function} (MGF) of $\pi$. Indeed, for a distribution of the form of \eqref{pc-exp-fam-c}, the MGF of $T^s(a)$ is guaranteed to exist and has a closed form \citep{Bickel:2006aa}. Hence, the computation of the moments reduces to the computation of derivatives. See details in Appendix \ref{ef-moments}.

In Lemma \ref{lem-epg-ef}, the assumption that $T^s$ and ${\hat{Q}}(a,s)$ are polynomial is with respect to the action $a$. The dependence on the state can be arbitrary, e.g., a multi-layered neural network.

Of course, while polynomials are universal approximators, they may not be the most efficient or stable ones. The importance of Lemma \ref{lem-epg-ef} is currently mainly conceptual---analytic EPG is possible for a universal class of approximators (polynomials) which shows that EPG is analytically tractable in principle for any continuous $Q$-function.\footnote{The universality of polynomials holds only for bounded intervals \citep{weierstrass1885analytische}, while the support of the policy may be unbounded. We do not address the unbounded approximation case here other than by saying that, in practice, the critic is learned from samples and is thus typically only accurate on a bounded interval anyway.} It is an open research question whether more suitable universal approximators admitting analytic EPG solutions can be identified.

\subsection{Reparameterized Exponential Families and Reparameterized Critics}
\label{sec-repar}
In Lemma \ref{lem-epg-ef}, we assumed that the function $T^s(a)$ (called the sufficient statistic of the exponential family) is polynomial. We now relax this assumption. Our approach is to start with a policy $\pi_b$ which \emph{does} have a polynomial sufficient statistic and then introduce a suitable reparameterization function $g: \mathbb{R}^d \rightarrow A$. The policy is then defined so that sampling an action
\[
  a \sim \pi (a \mid s) \quad \text{is equivalent to} \quad a = g(b) \;\; \text{with} \;\; b \sim \pi_b(b \mid s) = e^{(\est)^\top T^s(b) - U^s_{\est} + W^s(b)},
\]
where $b$ is the random variable representing the action before the squashing.  Assuming that $g^{-1}$ exists and the Jacobian\footnote{We avoid the standard Jacobian notation $\mathbf{J}$ because it is too similar to the total RL return $J$.} $\jacobian  g$ is non-singular almost everywhere, the PDF\footnote{We abuse  notation slightly by using $\pi(a \mid s)$ for both the probability distribution and its PDF.} of the policy $\pi$ can be written as
\begin{gather*}
  \label{eq-repar-pdf}
  \pi (a \mid s) = \pi_b(g^{-1}(a) \mid s) \frac{1}{ \det \jacobian  g(g^{-1}(a)) } = \pi_b(b \mid s) \frac{1}{ \det \jacobian  g(b) }.
\end{gather*}
The following lemma develops an EPG method for such policies.
\begin{restatable}[]{lmm}{epgexpfamilyrepar}
  \label{lem-rep-gpg}
  Consider an invertible and differentiable function $g$. Define a policy $\pi$ as in \eqref{eq-repar-pdf}. Assume that the Jacobian of $g$ is non-singular except on a set of $\pi_b$-measure zero. Consider a critic ${\hat{Q}}$. Denote as ${\hat{Q}}_b$ a reparameterized critic such that for all $a$, ${\hat{Q}}_b(g^{-1}(a), s) = {\hat{Q}}(a, s)$. Then the policy gradient update is given by the formula $I^{{\hat{Q}}}_\pi(s) = I^{{\hat{Q}}_b}_{\pi_b}(s)$.
\end{restatable}

We are now ready to state our universality result. The idea is to obtain a reparameterized version of EPG (and Lemma \ref{lem-epg-ef}) by reparametrizing the critic and the policy using the same transformation $g$. We do so in the following corollary, which is the most general constructive result in this article.

\begin{corollary}[EPG for Exponential Families with Reparametrization]\label{cor-epg-rep}\ \\
Consider the class of policies, parameterized by $\theta$, defined as in \eqref{pc-exp-fam-c}. Consider reparametrization function $g$ and define $T^s_b$, $V^s_b$ and ${\hat{Q}}^s_b$ as $T^s_b(g^{-1}(a)) = T^s(a)$, $W^s_b(g^{-1}(a)) = W^s(a)$ and ${\hat{Q}}_b(g^{-1}(a), s) = {\hat{Q}}(a,s)$ for every $a$. Assume the following:
\begin{enumerate}
  \item $g$ is invertible;
  \item The Jacobian of $g$ exists and is non-singular except on a set of $\pi_b$-measure zero, where $\pi_b$ is the reparameterized policy as in \eqref{eq-repar-pdf}; and
  \item $T^s_b$ and ${\hat{Q}}^s_b$ are polynomial as in Lemma \ref{lem-epg-ef}.
\end{enumerate}
Then a closed-form policy gradient update can be obtained as
\begin{gather*}
\label{r-poly-epg-cf}
I^{\hat{Q}}_\pi(s) = (\nabla_\theta  \eta_\theta^\top) ((C^s_{T_bQ_b})^\top m^s_{\pi_b}) - (\nabla_\theta  U^s_{\eta_\theta}) ((C^s_{Q_b})^\top m^s_{\pi_b}).
\end{gather*}
\end{corollary}
\begin{proof}
  Apply Lemmas \ref{lem-rep-gpg} and then \ref{lem-epg-ef}.
\end{proof}

Lemma \ref{lem-rep-gpg} also has a practical application in case we want to deal with bounded action spaces. As we discussed in Section \ref{sec-gpg-clip}, hard clipping can cause the problem of vanishing gradients and the default solution should be to use GPG. In case we can't use GPG, for instance when the dimensionality of the action space is so large that computing the covariance of the policy is too costly, we can alleviate the vanishing gradients problem by using a strictly monotonic squashing function $g$. One implication of Lemma \ref{lem-rep-gpg} is that, if we set $\pi_b$ to be Gaussian, we can invoke Lemma \ref{lem-gpg} to obtain exact analytic updates for useful policy classes such as Log-Normal and Logit-Normal (obtained by setting $g$ to the sigmoid and the exponential function respectively), as long as we choose our critic ${\hat{Q}}^s$ to be quadratic in $g^{-1}(a)$, i.e., ${\hat{Q}}^s_b$ is quadratic in $b$. The reparameterized version of EPG is the same as Algorithm \ref{alg-epg-clip} except it uses a squashing function $g$ instead of the clipping function $c$.

\subsection{Arbitrary Policies and Linear Critics}

Next, we consider the case where the stochastic policy is almost completely arbitrary, i.e., it only has to possess a mean and need not even be in the already general exponential family of policies used in Lemma \ref{lem-epg-ef} and Corollary \ref{cor-epg-rep}, but the critic is constrained to be linear in the actions. We have the following lemma, which is a slight modification of an observation made in connection with the $Q$-Prop algorithm \cite[Eq. 7]{gu2016q}.

\begin{lemma}[EPG for Arbitrary Stochastic Policies and Linear Critics]\label{lem-lin-crit}\ \\
Consider an arbitrary (non-degenerate) probability distribution $\pi(\cdot \mid s)$ which has a mean. Assume that the critic $\hat{Q}(a, s)$ is of the form $A_s^\top a$ for some coefficient vector $A_s$. Then the policy gradient update is given by $I^{\hat{Q}}_\pi(s) = A_s^\top \nabla_\theta  \mu_{\pi(\cdot \mid s)}$ where $\mu_{\pi(\cdot \mid s)}$ denotes the integral $\int_a a d\pi(a \mid s) $ (the mean).
\end{lemma}
\begin{proof}
The lemma is proven by rewriting the policy gradient update as
\begin{align*}
I^Q_\pi(s) &= \int_a \nabla_\theta  \pi(a \mid s) {\hat{Q}}(a, s) da 
= \int_a \nabla_\theta  \pi(a \mid s) A_s^\top a da \\
&= A_s^\top \nabla_\theta  \underbrace{\int_a \pi(a \mid s) a da}_{\mu_{\pi(\cdot \mid s)}} 
= A_s^\top \nabla_\theta  (\mu_{\pi(\cdot \mid s)} ).
\end{align*}
\end{proof}
Since DPG already provides the same result for Dirac-delta policies (see Corollary \ref{cor-dpg}), we conclude that using linear critics means we can have an analytic solution for any reasonable policy class.

To see why the above lemma is useful, first consider systems that arise as a discretization of continuous time systems with a fine enough time scale and differentiable dynamics and rewards. If we assume that the true $Q$ is smooth in the actions and that the magnitude of the allowed action goes to zero as the time step decreases, then a linear critic is sufficient as an approximation of $Q$ because we can approximate any smooth function with a linear function in any sufficiently small neighborhood of a given point and then choose the time step to be small enough so an action does not leave that neighborhood. We can then use Lemma \ref{lem-lin-crit} to perform policy gradients with any policy.\footnote{Of course the update derived in Lemma \ref{lem-lin-crit} only provides a direction in which to change the policy mean (which means that exploration has to be performed using some other mechanism). This is because a linear critic does not contain enough information to determine exploration. }

\subsection{If All Else Fails: EPG with Numerical Quadrature}
\label{ss-numq}
If, despite the broad framework shown in this article, an analytical solution is impossible, we can still perform integration numerically. EPG can still be beneficial in these cases: if the action space is low dimensional, numerical quadrature is cheap; if it is high dimensional, it is still often worthwhile to balance the expense of simulating the system with the cost of quadrature. Actually, even in the extreme case of expensive quadrature but cheap simulation, the limited resources available for quadrature could still be better spent on EPG with smart quadrature than SPG with simple Monte Carlo.

The crucial insight behind numerical EPG is that the integral given as\footnote{The second expression is known as the reparameterized gradient and was introduced by \citet{heess2015learning} in the context of RL.}
\[
I^{\hat{Q}}_\pi = \int_a d \pi(a \mid s) \nabla_\theta  \log \pi(a \mid s) \hat{Q}(a, s) = \int_a d \pi(a \mid s) \nabla_\theta  {\hat{Q}}(\pi(a_i),s)
\]
only depends on two fully known quantities: the current policy $\pi$ and the current approximate critic $\hat{Q}$. Therefore, we can use any standard numerical integration method to compute it. For example, multi-sample Monte-Carlo quadrature for the policy gradient with respect to the mean is given by
\begin{gather*}
    \label{eq-multi-sample-mc}
I^{\hat{Q}}_{\pi(s), \mu_s} = \frac1m \sum_{i=1}^m (\nabla_\theta  {\hat{Q}}(\pi(a_i),s)) \quad \text{where} \quad a_i \sim \pi(\cdot \vert s).
\end{gather*}
The actions at which the integrand is evaluated do not have to be sampled---one can also use a method such as the Gauss-Legendre quadrature where the abscissae are designed.

\subsection{Probability Distributions over Discrete Actions and Softmax Policies}
\label{sec-discrete-epg}
The main idea of EPG can also be applied to the setting of discrete actions. In this case, the integral in $I^{Q}_\pi(s)$ becomes a sum, and the policy gradient update takes the form
\begin{gather*}
I^{Q}_\pi(s) = \sum_a  \pi(a \mid s) \nabla \log \pi(a \mid s) {Q}(a, s) = \sum_a  \nabla \pi(a \mid s) {Q}(a, s).
\label{eq-epg-discrete}
\end{gather*}
In this case, a softmax parametrization of the discrete policy, also known as a Gibbs or Boltzmann policy, is often chosen. The following observation provides a slightly optimized formula for the sum.

\begin{observation}[Expected Policy Gradients for Discrete Softmax Policies]\label{obs-discrete-epg}\ \\
If the action space is discrete, and the policy is a discrete softmax distribution, i.e., $\pi(\cdot \vert s) \propto e^{h_{\theta}(\cdot, s)} $ for some function $h_\theta$ parametrized by $\theta$, then
\[
I^{Q}_\pi(s) = (\nabla h) u.
\]
Here, $\nabla h$ is a Jacobian matrix and $u$ is a vector whose elements are defined as
\begin{gather*} \textstyle
\{u\}_i = \pi(a_i \vert s) \left( \sum_{j} \pi(a_j \vert s) (Q(a_i, s) - Q(a_j, s)) \right) = \pi(a_i \vert s) Q(a_i, s) - V(s). \label{obs-da-d}
\end{gather*}
\end{observation}
\begin{proof}
Denote by $Z(s)$ the normalization factor of the policy, i.e. $Z(s) = \sum_i e^{h_\theta(a_i,s)}$ and
$\pi(a_i \vert s) = \frac{e^{h_\theta(a_i, s)}}{Z(s)}$.
First, we expand the term $\nabla \pi(a_i \vert s)$ as 
\begin{align*}
\nabla \pi(a_i \vert s) &= \frac{e^{h_\theta(a_i, s)}}{Z(s)} \nabla h_\theta(a_i,s) - \frac{e^{h_\theta(a_i, s)} (\sum_j e^{h_\theta(a_j, s)} \nabla h_\theta(a_j,s))}{Z(s)^2} = \\
&= \textstyle \pi(a_i \vert s) \nabla h_\theta(a_i,s) - \pi(a_i \vert s) \sum_j \pi(a_j \vert s) \nabla h_\theta(a_j,s).
\end{align*}
We now plug this into the definition of $I^{Q}_\pi(s)$, obtaining
\begin{align*}
\textstyle
I^{Q}_\pi(s) &= \textstyle \sum_i \nabla \pi(a_i \vert s) Q(a_i, s) =\\
&= \textstyle \left ( \sum_{i\vphantom{j}} \pi(a_i \vert s) \nabla h_\theta(a_i,s) Q(a_i, s) \right ) - \textstyle \left( \sum_j \pi(a_j \vert s) \nabla h_\theta(a_j,s) \right) \textstyle \left ( \sum_{i\vphantom{j}} \pi(a_i \vert s) Q(a_i, s) \right ) = \\
&= \textstyle \left ( \sum_{i\vphantom{j}} \pi(a_i \vert s) \nabla h_\theta(a_i,s) Q(a_i, s) \right ) - \textstyle \left( \sum_{i\vphantom{j}} \pi(a_i \vert s) \nabla h_\theta(a_i,s) \right) \textstyle \left (\sum_j \pi(a_j \vert s) Q(a_j, s) \right ) = \\
&= \textstyle \sum_{ij} \pi(a_i \vert s) \pi(a_j \vert s) \nabla h_\theta(a_i,s) (Q(a_i, s) - Q(a_j, s)) \\
&= \textstyle \sum_i \pi(a_i \vert s) \nabla h_\theta(a_i,s) (Q(a_i, s) - V(s)).
\end{align*}
The last two lines give the desired result. 
\end{proof}

We include this observation because the required simplifications, leading to \eqref{obs-da-d} may not always be performed when using automatic differentiation software. Also, \eqref{obs-da-d} makes clear that only differences between $Q$-values matter, not absolute values. 

To apply Observation \ref{obs-discrete-epg} in practice, $Q$ and $V$ are replaced by their approximations $\hat{Q}$ and $\hat{V}$ respectively, similarly to the continuous case. However, the use of EPG for discrete policies did not improve performance for a task we tried, a result we discuss in Section \ref{sec-exp-disc}.

\section{Experiments}
\label{sec-gpg-exp}

\begin{figure}[ht]
    \centering
         \captionsetup[subfigure]{aboveskip=0.5em,belowskip=0.5em}
        \begin{subfigure}{0.42\textwidth}
            \caption{HalfCheetah-v2}
            \includegraphics[width=\textwidth]{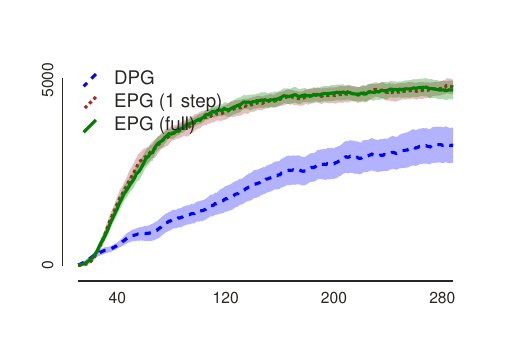}
        \end{subfigure}
        ~
        \begin{subfigure}{0.42\textwidth}
            \caption{InvertedPendulum-v2}
            \includegraphics[width=\textwidth]{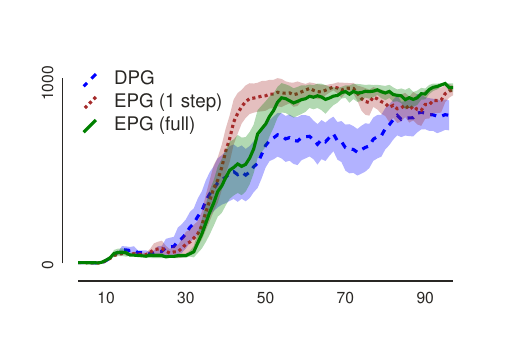}
        \end{subfigure}
        \\
        \begin{subfigure}{0.42\textwidth}
            \caption{Reacher-v2}
            \includegraphics[width=\textwidth]{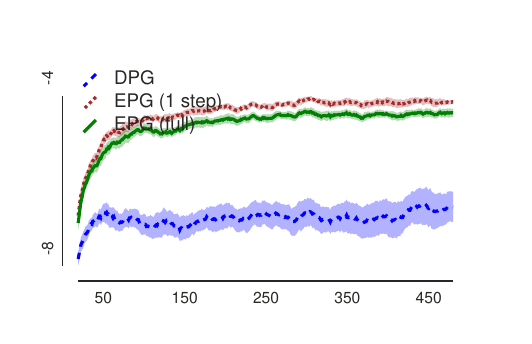}
        \end{subfigure}
        ~
        \begin{subfigure}{0.42\textwidth}
            \caption{Walker2d-v2}
            \includegraphics[width=\textwidth]{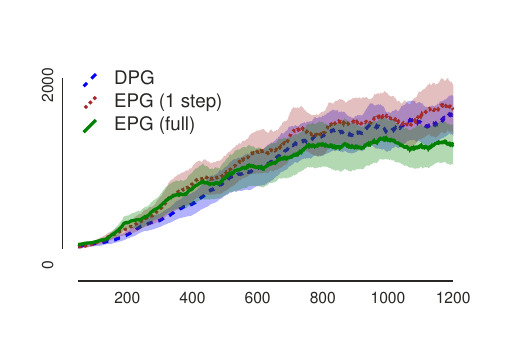}
        \end{subfigure}
        \\
        ~
        \begin{subfigure}{0.42\textwidth}
            \caption{InvertedDoublePendulum-v2}
            \includegraphics[width=\textwidth]{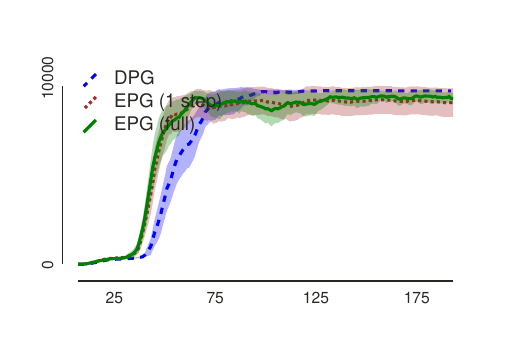}
             
        \end{subfigure}
        \caption{Learning curves (mean and 90\% interval) showing exploration using EPG, 1-step EPG and the OU noise used in DPG. Returns for Reacher-v2 are clipped at -10. All results are obtained from 20 runs. Horizontal axis shows thousands of steps.}
        \label{fig-epg-1step-dpg}
    \end{figure}

    \begin{figure}[ht]
    \centering
            \captionsetup[subfigure]{aboveskip=0.5em,belowskip=0.5em}
        \begin{subfigure}{0.42\textwidth}
            \caption{HalfCheetah-v2}
            \includegraphics[width=\textwidth]{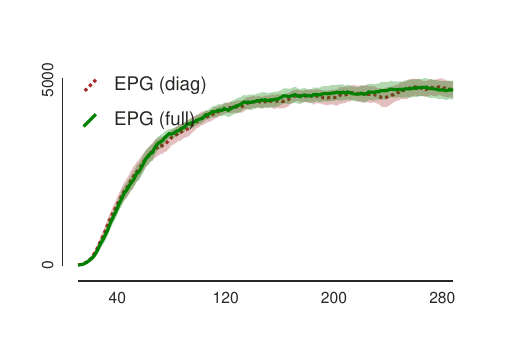}
        \end{subfigure}
        ~
        \begin{subfigure}{0.42\textwidth}
            \caption{Reacher-v2}
            \includegraphics[width=\textwidth]{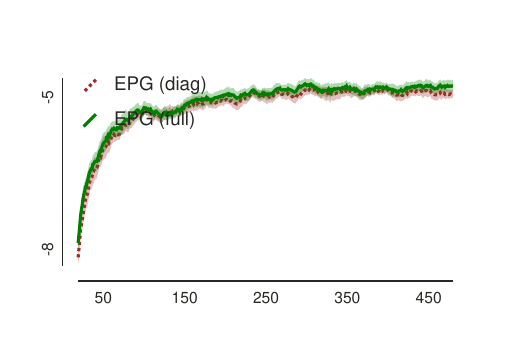}
        \end{subfigure}
        ~ \\
        \begin{subfigure}{0.42\textwidth}
            \caption{Walker2d-v2}
            \includegraphics[width=\textwidth]{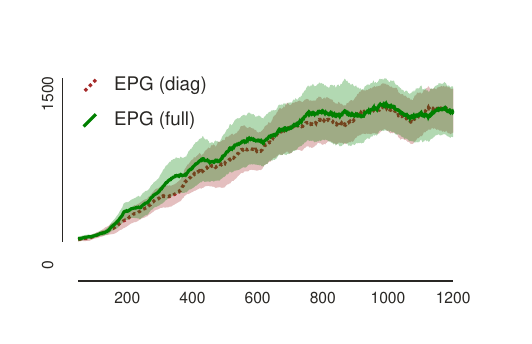}
        \end{subfigure}

        \caption{Learning curves (mean and 90\% interval) comparing EPG with diagonal and full Hessian. Returns for Reacher-v2 are clipped at -10. All results are obtained from 20 runs. Horizontal axis shows thousands of steps. Results for the pendulum domains are not show since they only have one action dimension.}
        \label{fig-exp-diagexp}
    \end{figure}

    \begin{figure}[ht]
        \centering
                \captionsetup[subfigure]{aboveskip=0.5em,belowskip=0.5em}
            \begin{subfigure}{0.42\textwidth}
                \caption{HalfCheetah-v2}
                \includegraphics[width=\textwidth]{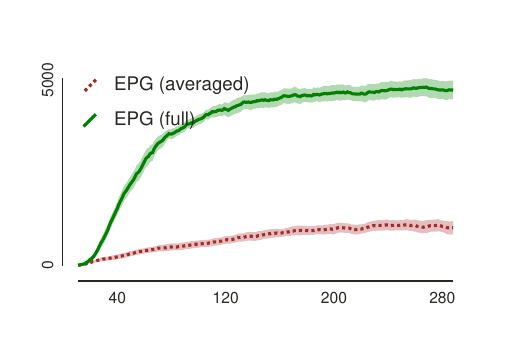}
            \end{subfigure}
            ~
            \begin{subfigure}{0.42\textwidth}
                \caption{InvertedPendulum-v2}
                \includegraphics[width=\textwidth]{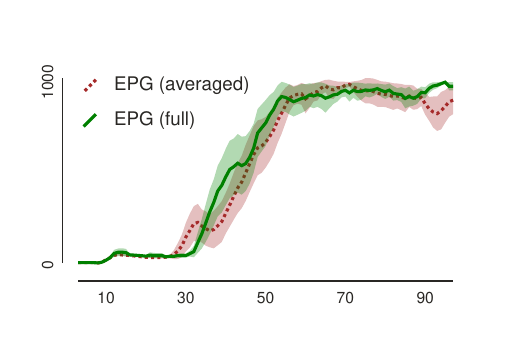}
            \end{subfigure}
            \\
            \begin{subfigure}{0.42\textwidth}
                \caption{Reacher-v2}
                \includegraphics[width=\textwidth]{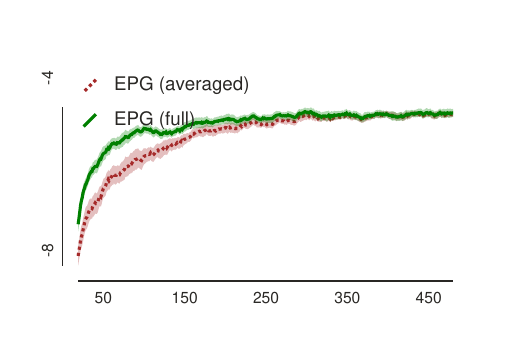}
            \end{subfigure}
            ~
            \begin{subfigure}{0.42\textwidth}
                \caption{Walker2d-v2}
                \includegraphics[width=\textwidth]{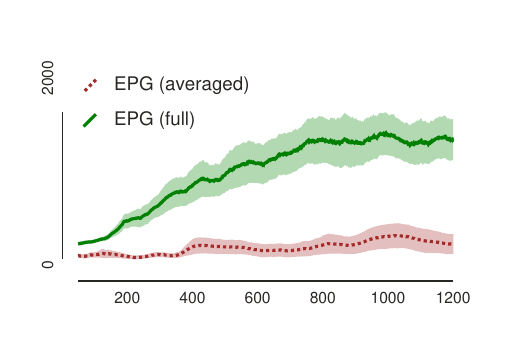}
            \end{subfigure}
            \\
            ~
            \begin{subfigure}{0.42\textwidth}
                \caption{InvertedDoublePendulum-v2}
                \includegraphics[width=\textwidth]{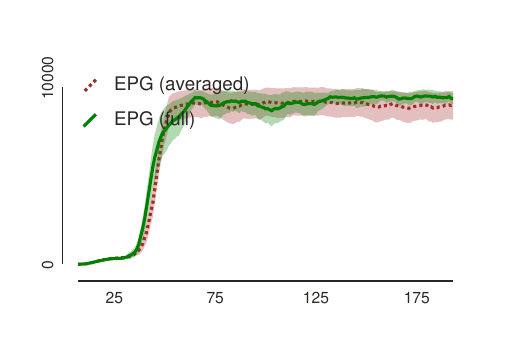}
                 
            \end{subfigure}
            \caption{Learning curves (mean and 90\% interval) showing exploration using vanilla EPG and a variant that uses a averaged covariance (i.e. one that is constant for every state). Returns for Reacher-v2 are clipped at -10. All results are obtained from 20 runs. Horizontal axis shows thousands of steps.}
            \label{fig-exp-expglobal}
        \end{figure}

While EPG has many potential uses, we focus on empirically evaluating its applications to exploration and variance reduction in the actor. To benchmark our algorithms, we use five continuous-action domains, modeled with the MuJoCo physics simulator \citep{todorov2012mujoco}: HalfCheetah-v2, InvertedPendulum-v2, Reacher2d-v2, Walker2d-v2, and InvertedDoublePendulum-v2, as well as one discrete-action domain: Atari Pong. In Section \ref{exp-exploration}, we evaluate the benefits of exploring using a covariance matrix that comes from the Hessian exponential. In Section \ref{exp-quad}, we compare the types of quadrature typically used in connection with policy gradients. In Section 6.3, we discuss hyperparameter tuning. In Section \ref{exp-ppo}, we compare EPG with \emph{proximal policy optimization} (PPO) \citep{schulman2017proximal}. In Section \ref{sec-exp-disc}, we apply EPG to the discrete-action domain, Atari Pong.

\subsection{Continuous Control: Exploration}
\label{exp-exploration}
The EPG framework can be used to derive a new policy gradient algorithm (see Algorithm \ref{alg-epg-fixpoint} and Lemma \ref{lem-hexplore}), where the covariance of the exploration policy is obtained by either taking the matrix exponent of the critic as given in \eqref{eq=exp-sigma} or its approximation, which we call 1-step EPG, as given in \eqref{eq=1step-sigma}. In practice, both versions of EPG differ from deep DPG \citep{lillicrap2015continuous, silver2014deterministic} only in the exploration strategy, though their theoretical underpinnings are also different. 

The Hessian is obtained using a sigma-point method, as follows. At each step, the agent samples 100 action values from $\hat{Q}(\cdot,s)$ and a quadratic is fit to them in the $L_2$ norm. Since this is a least-squares problem, it can be accomplished by solving a linear system. The Hessian computation could be greatly sped up by using an approximate method, or even skipped completely if we use a quadratic critic. However, we did not optimize this part of the algorithm since it is orthogonal to the core insight of GPG that the Hessian is useful for exploration.

We now evaluate both EPG and 1-step EPG, as an alternative to the standard Ornstein-Uhlenbeck (OU) exploration used by Deterministic Policy Gradients. The results in Figure \ref{fig-epg-1step-dpg} show that EPG's exploration strategy yields much better performance than DPG with OU. 1-step EPG performs just well, or even slightly better than the version with the exponent. This is not surprising -- the function $\max(0,1+\lambda)$ is a good approximation of $e^\lambda$ in for the range of eigenvalues seen during training.  

In order to benchmark more accurately which aspects of EPG are important for performance, we performed two additional ablations. Figure \ref{fig-exp-diagexp} compares EPG with a diagonal version, which uses a diagonal Hessian. The diagonal Hessian is fitted using the same process as the full Hessian, except that the local quadratic approximation is constrained so that the off-diagonal entries are zero. The performance of diagonal EPG is similar to the full version, showing that the critic curvature that has practical significance for exploration can be estimated component-wise. This result is consistent with most recent work on the topic \citep{haarnojaSoftActorCriticOffPolicy2018, fujimotoAddressingFunctionApproximation2018}. 

To make an even more minimal version of EPG, we tried to estimate the Hessian globally, i.e., using the same estimate for every state.  The results in Figure \ref{fig-exp-expglobal} show using such a global Hessian is suboptimal.  

\subsection{Continuous Control: Updates for the Policy Mean}
\label{exp-quad}

\begin{figure}[ht]
    \centering
            \captionsetup[subfigure]{aboveskip=0.5em,belowskip=0.5em}
        \begin{subfigure}{0.42\textwidth}
            \caption{HalfCheetah-v2}
            \includegraphics[width=\textwidth]{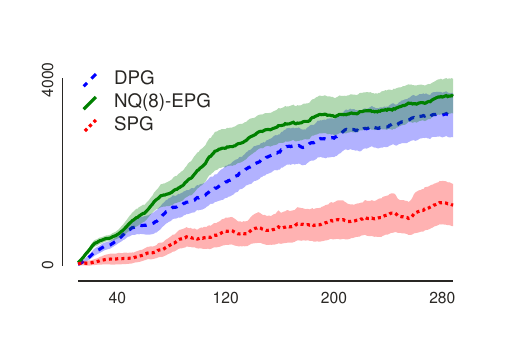}
        \end{subfigure}
        ~
        \begin{subfigure}{0.42\textwidth}
            \caption{InvertedPendulum-v2}
            \includegraphics[width=\textwidth]{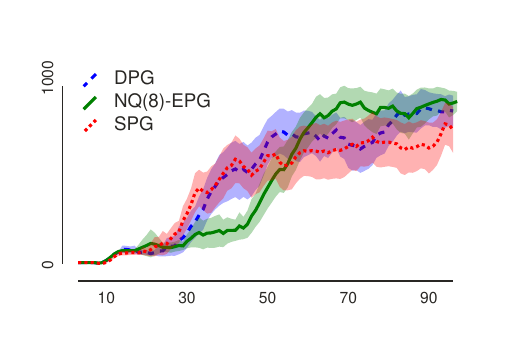}
        \end{subfigure}
        \\
        \begin{subfigure}{0.42\textwidth}
            \caption{Reacher-v2}
            \includegraphics[width=\textwidth]{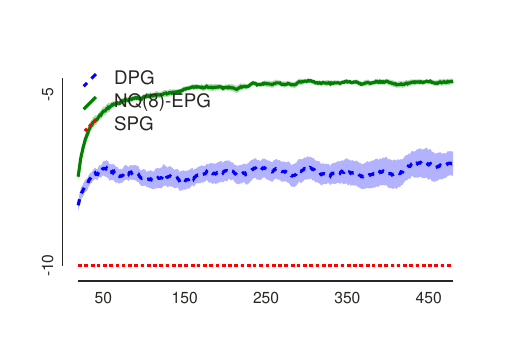}
        \end{subfigure}
        ~
        \begin{subfigure}{0.42\textwidth}
            \caption{Walker2d-v2}
            \includegraphics[width=\textwidth]{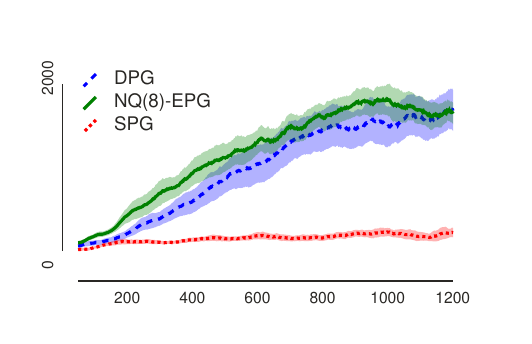}
        \end{subfigure}
        \\
        ~
        \begin{subfigure}{0.42\textwidth}
            \caption{InvertedDoublePendulum-v2}
            \includegraphics[width=\textwidth]{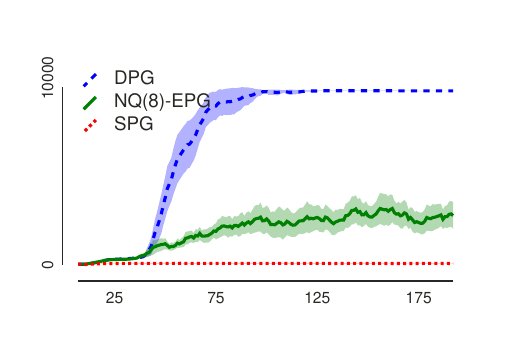}
             
        \end{subfigure}
        \caption{Learning curves (mean and 90\% interval) showing learning performance when updating the policy mean using Expected Policy Gradients with numerical quadrature, as compared to deterministic policy gradients and stochastic policy gradients. Returns for Reacher-v2 are clipped at -10. All results are obtained from 20 runs. Horizontal axis shows thousands of steps.}
        \label{fig-actor-epg-dpg-spg}
    \end{figure}

\begin{figure}[ht]
\centering
        \captionsetup[subfigure]{aboveskip=0.5em,belowskip=0.5em}
    \begin{subfigure}{0.42\textwidth}
        \caption{HalfCheetah-v2}
        \includegraphics[width=\textwidth]{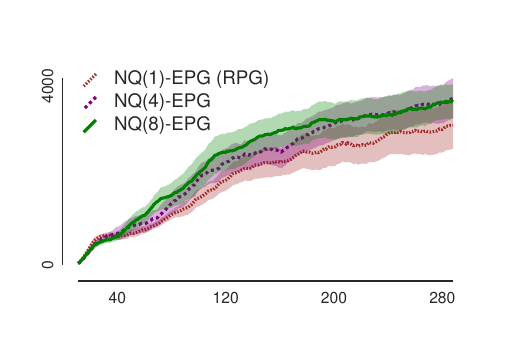}
    \end{subfigure}
    ~
    \begin{subfigure}{0.42\textwidth}
        \caption{InvertedPendulum-v2}
        \includegraphics[width=\textwidth]{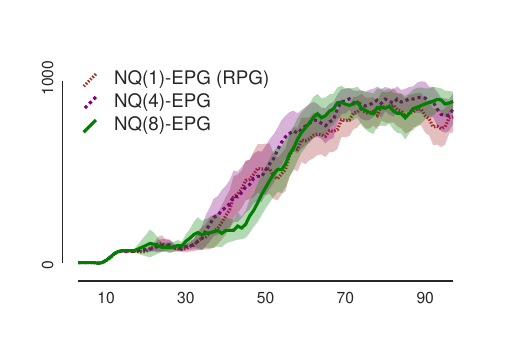}
    \end{subfigure}
    \\
    \begin{subfigure}{0.42\textwidth}
        \caption{Reacher-v2}
        \includegraphics[width=\textwidth]{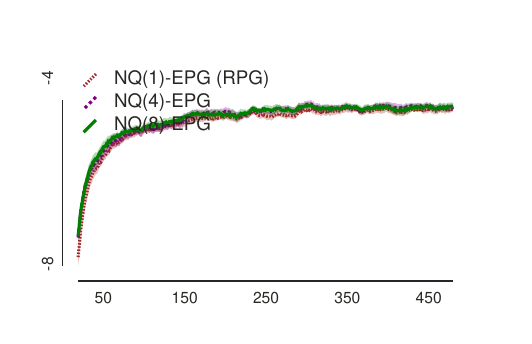}
    \end{subfigure}
    ~
    \begin{subfigure}{0.42\textwidth}
        \caption{Walker2d-v2}
        \includegraphics[width=\textwidth]{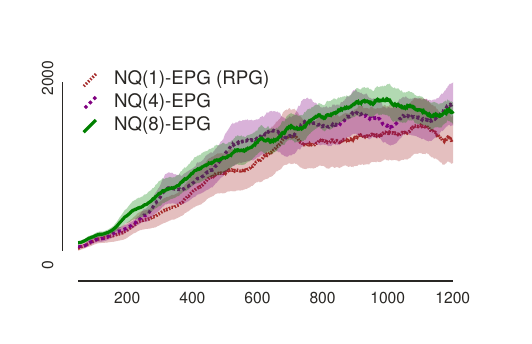}
    \end{subfigure}
    \\
    ~
    \begin{subfigure}{0.42\textwidth}
        \caption{InvertedDoublePendulum-v2}
        \includegraphics[width=\textwidth]{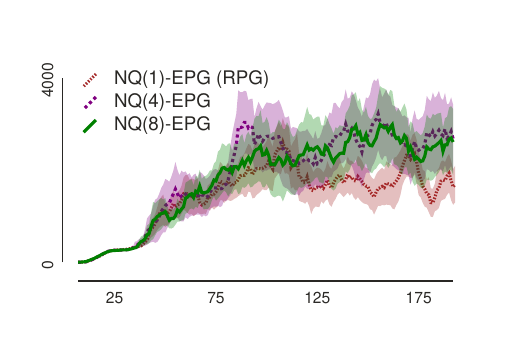}
         
    \end{subfigure}
    \caption{Learning curves (mean and 90\% interval) showing learning performance when updating the policy mean using multi-sample quadrature for variants with 1, 4 and 8 samples. Quadrature with 1 sample is equivalent to reparametrized policy gradients (RPG). Returns for Reacher-v2 are clipped at -10. All results are obtained from 20 runs. Horizontal axis shows thousands of steps.}
    \label{fig-exp-quadrature}
\end{figure}

In this section, we compare different ways of updating the policy mean. We test an expected policy gradient method based on numerical quadrature NQ($m$)-EPG, which computes estimates of the policy gradient using multiple samples of the policy gradient integral and compare it to DPG and SPG. The quadrature is performed as in \eqref{eq-multi-sample-mc}.

In order to capture the effect of the actor update only, we used OU exploration with noise $0.2$ for both NQ($m$)-EPG and DPG. We used Gaussian exploration for SPG. Details of hyperparameters are given in Appendix \ref{sec-exp-details}. Our results in Figure \ref{fig-actor-epg-dpg-spg} show that using numerical quadrature as in \eqref{eq-multi-sample-mc} generally improves performance. The results in Figure \ref{fig-exp-quadrature} show that adding more samples to numerical quadrature does not change performance much, at least for the simple type of quadrature in \eqref{eq-multi-sample-mc}. This is unsurprising given that NQ(1)-EPG (with one sample) corresponds to Reparameterized Policy Gradients, which are known to perform well \citep{haarnojaSoftActorCriticOffPolicy2018}. To obtain a larger improvement on these tasks one would need to use either more sophisticated quadrature or a local parametric approximation similar to the sigma-point method in Section \ref{exp-exploration}.

The behavior of the InvertedDoublePendulum domain, where numerical quadrature performs significantly worse than standard deterministic policy gradient update, is interesting. We believe this behavior is due to the task's inherent instability -- any deviation from the true value leads to sub-optimal updates. One way of addressing this problem would be to reduce the exploration noise. However, in order to get meaningful comparisons, we elected to run our experiments with the same hyperparameters across tasks.

Furthermore, SPG does poorly, solving only the easiest domain (InvertedPendulum-v2) in reasonable time, achieving slow progress on  HalfCheetah-v2, and failing entirely on the other domains. This is not surprising since DPG was introduced precisely to solve the problem of high variance SPG estimates on this type of task. In InvertedPendulum-v2, SPG initially learns quickly, outperforming the other methods, because noisy gradient updates provide a crude, indirect form of exploration that happens to suit this problem. Clearly, this is inadequate for more complex domains: even for this simple domain it leads to sub-par performance late in learning.

\subsection{Sensitivity of EPG to hyperparameters}
\label{exp-hyper}
The hyperparameters for DPG and those of EPG that are not related to exploration were taken from an existing benchmark \citep{islam2017reproducibility, brockman2016openai}. They are detailed in Appendix \ref{sec-exp-details}. Our EPG exploration technique has just one hyperparameter $\sigma_0$ while OU has two (standard deviation and mean reversion constant). We optimized $\sigma_0$ on the HalfCheetah domain (Figure \ref{fig-hyperparameters-EPG}) and settled on the value $\sigma_0 = 0.5$.

\begin{figure}[ht]
    \centering
    \includegraphics[width=0.42\textwidth]{plots-epg-final/hyperparameters-EPG/all.pdf}
        
        \caption{Learning curves (mean and 90\% interval) for HalfCheetah-v2 showing different values of $\sigma_0$ for EPG. All results are obtained from 20 runs. Horizontal axis shows thousands of steps.}
        \label{fig-hyperparameters-EPG}
    \end{figure}

To ensure a fair comparison, we also optimized hyperparameters for DPG; details and learning curves are in Appendix \ref{sec-exp-details}.

\subsection{Comparison with PPO.}
\label{exp-ppo}
One feature of EPG is stability in the learning outcome, i.e., low variance across runs. EPG's stability raises the question whether the instability of an algorithm (i.e., an inverted or oscillating learning curve) is caused primarily by inefficient exploration or by excessively large differences between subsequent policies. To address this question, we compare our results with \emph{proximal policy optimization} (PPO) \citep{schulman2017proximal}, a policy gradient algorithm that explicitly penalizes large differences between successive policies. The results are shown in Figure \ref{fig-ppo-epg}, where we allowed to run PPO until its performance plateaus. On one hand, the results show that EPG is indeed more stable (represented by a narrower confidence interval). On the other hand, PPO performed better overall on the Walker task. This suggests that both the stability relating to exploration and the stability relating to changes in policy space can play a role in policy gradients. In principle, it would be possible to achieve \emph{both} kinds of stability by exploiting the curvature of the critic to obtain the covariance of the policy while at the same time constraining the sequence of policies to be close to one another. Due to the amount of engineering involved in tuning such an algorithm, we leave this idea to future work.

\begin{figure}[ht]
    \centering
            \captionsetup[subfigure]{aboveskip=0.5em,belowskip=0.5em}
        \begin{subfigure}{0.42\textwidth}
            \caption{HalfCheetah-v2}
            \includegraphics[width=\textwidth]{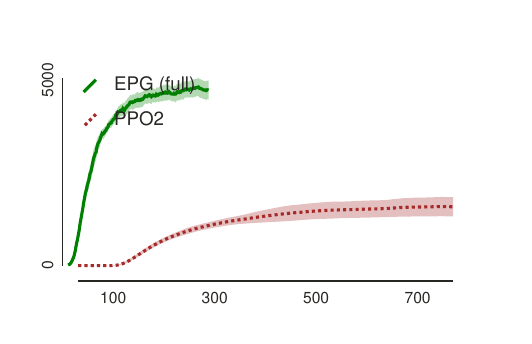}
        \end{subfigure}
        ~
        \begin{subfigure}{0.42\textwidth}
            \caption{Walker2d-v2}
            \includegraphics[width=\textwidth]{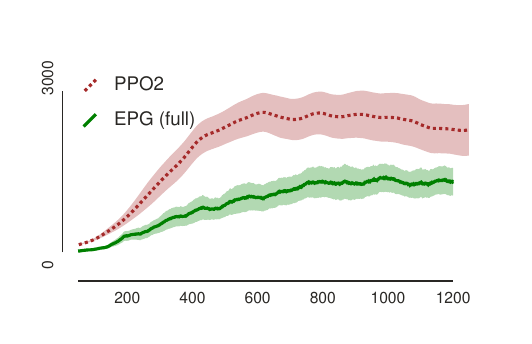}
        \end{subfigure}
        \caption{Learning curves (mean and 90\% interval) showing learning performance of PPO (using the default parameters) as compared to EPG. All results are obtained from 20 runs. Horizontal axis shows thousands of steps.}
        \label{fig-ppo-epg}
    \end{figure}

\subsection{EPG for Discrete Action Spaces}
\label{sec-exp-disc}

To evaluate the usefulness of Expected Policy Gradients for discrete action spaces, we applied it to the Atari version of Pong. The results are presented in Figure \ref{fig-res-discrete}. We used the OpenAI version of A2C \citep{mnih2016asynchronous} for stochastic policy gradients (labeled SPG in the plot) as a baseline, together with the default hyperparameters of the OpenAI implementation \citep{baselines}. To make a discrete version of EPG, we modified the A2C critic to also learn $\hat{Q}$. We also modified the policy gradient update to EPG as in \eqref{eq-epg-discrete} and kept the remaining settings of the algorithm the same. Figure \ref{fig-res-discrete} demonstrates that for Pong, there is no measurable benefit from introducing a sum over actions to the policy gradient estimate. In discrete domains like Pong, the learned critic is often inaccurate, which can be a much greater problem than the variance of the policy estimator. Hence, even reducing the variance in the policy gradient estimator to zero, as EPG does, does not help performance.

\begin{figure}[htb]
    \centering
    \includegraphics[width=0.45\textwidth]{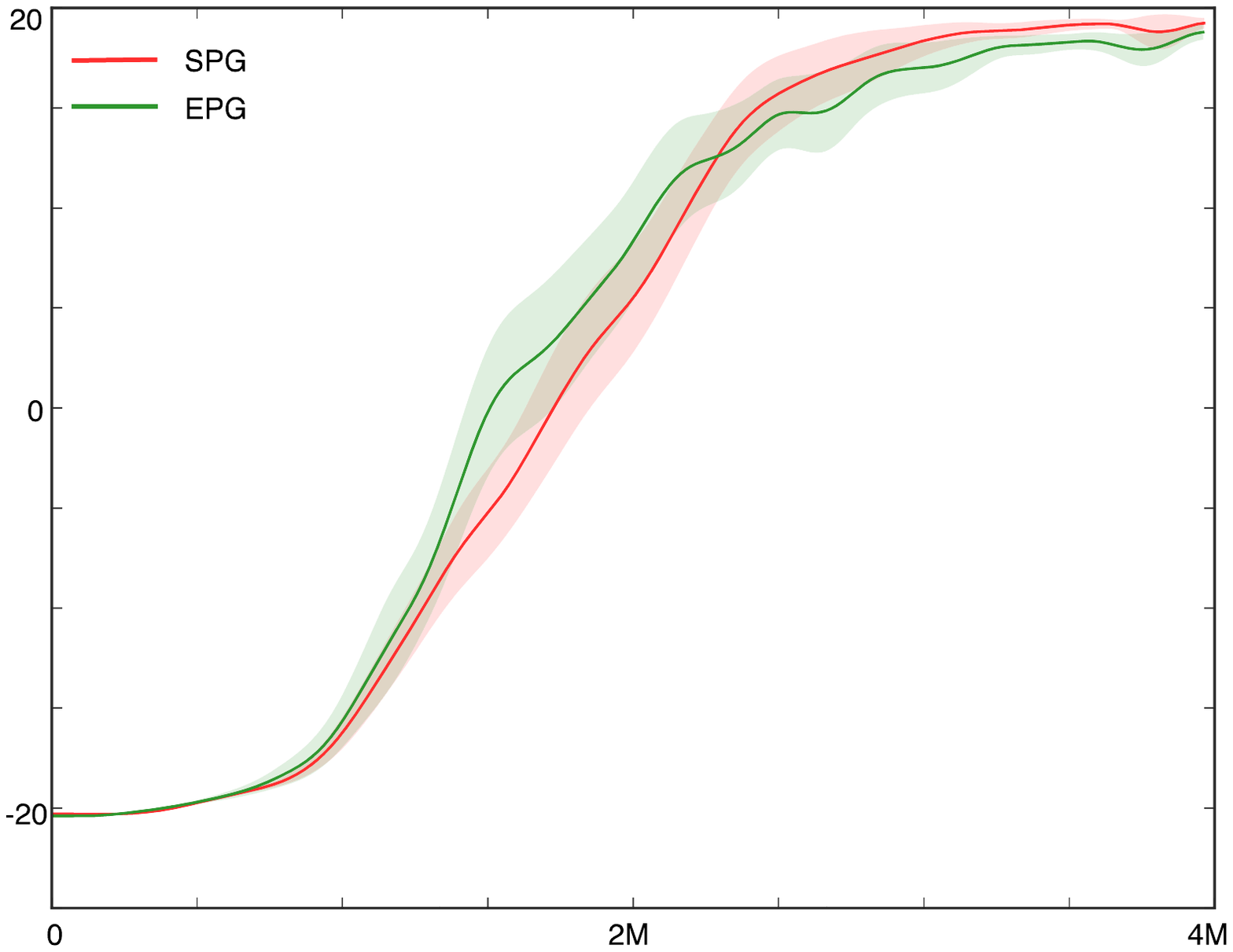}
    \caption{EPG vs SPG on a discrete action domain (Atari Pong). The curves show the mean and 90\% confidence interval and were obtained by 24 runs of each algorithm. Horizontal scale is in millions of steps.}
    \label{fig-res-discrete}
\end{figure}

\section{Related Work}
\label{sec-rel}
In this section, we discuss the relationship between EPG and several other methods.
\subsection{Sampling Methods for SPG}
EPG has some similarities with VINE sampling \citep{schulman2015trust}, which uses an (intrinsically noisy) Monte Carlo quadrature with many samples. However, there are important differences. First, VINE relies entirely on reward rollouts and does not use an explicit critic. This means that VINE has to perform many independent rollouts of $Q(\cdot, s)$ for each $s$, requiring a simulator with reset. A second, related difference is that VINE uses the \emph{same} actions in the estimation of $I^{\hat{Q}}_\pi$ that it executes in the environment. While this is necessary with purely Monte Carlo rollouts, Section \ref{ss-numq} shows that there is no such need in general if we have an explicit critic. Ultimately, the main weakness of VINE is that it is a purely Monte Carlo method. However, the example in Figure \ref{fig-mc-p} (Section \ref{sec-epg-main}) shows that even with a computationally expensive many-sample Monte Carlo method, the problem of variance in the gradient estimator remains, regardless of the baseline.

EPG is also related to variance minimization techniques that interpolate between two estimators \citep{gu2016q}. However, EPG uses a quadratic (not linear) critic, which is crucial for exploration. Furthermore, it completely eliminates variance in the inner integral, as opposed to just reducing it.

A more direct way of coping with variance in policy gradients is to simply reduce the learning rate when the variance of the gradient would otherwise explode, using, e.g., \emph{Adam} \citep{kingma2014adam} or the adaptive step size method \citep{pirotta2013adaptive}. However, this results in slow learning when the variance is high. Another way of reducing variance is by defining new estimators. \citet{parmasTotalStochasticGradient2018} introduces a framework for doing that based on a graphical representation that also allows for including model-based approaches. This framework provides another way of generalizing the deterministic and the stochastic policy gradient theorem by looking them as different ways of estimating the same quantity. Unlike EPG, it does not perform analytic integration.

\subsection{Sarsa and Q-Learning}
It has been known since the introduction of policy gradient methods \citep{sutton2000policy} that they represent a kind of slow-motion policy improvement as opposed to a greedy improvement performed by methods such as (expected) sarsa with action maximization or Q-learning. The two main reasons for the slow-motion improvement are that a greedy maximization operator may not be available (e.g., for continuous or large discrete action spaces) and that a greedy step may be too large because the critic only approximates the value function locally. The argument for a sarsa-like method is that it may converge faster and does not need an additional optimization for the actor. Recently, approaches combining the features of both methods have been investigated. One-step Newton's method for $Q$-functions that are quadratic in the actions has been used to produce a sarsa-like algorithm for continuous domains \citep{gu2016continuous}, previously only tractable with policy gradient methods. For discrete action spaces, softmax $Q$-learning, a family of methods with a hybrid loss combining sarsa and $Q$-learning, has recently been linked to policy gradients via an entropy term \citep{o2016combining}. In this paper, GPG with Hessian-based exploration (Section \ref{ss-hessian-exp}) can be seen as another kind of hybrid. Specifically, it changes the mean of the policy slowly, similar to a vanilla policy gradient method, and computes the covariance greedily, similar to sarsa.

\subsection{DPG}
The update for the policy mean obtained in Corollary \ref{lem-agpg} is  the same as the DPG update, linking the two methods
\[
I^Q_\pi(s) = \atmus{\nabla_a \Qv(a,s)} \nabla_\theta  \mu_s.
\]

We now formalize the equivalences between EPG and DPG. First, any EPG method with a linear critic (or an arbitrary critic approximated by the first term in the Taylor expansion) is equivalent to DPG with actions from a given state $s$ drawn from an exploration policy of the form
\[
a \sim \pi(s) + n(a \vert s), \quad \text{where} \quad \excs{a \sim n}{a}{s} = 0.
\]
Here, the PDF of the zero-mean exploration noise $n(\cdot \vert s)$ must not depend on the policy parameters. This fact follows directly from Lemma \ref{lem-lin-crit}, which says that, in essence, a linear critic only gives information on how to shift the mean of the policy and no information about other moments. Second, on-policy GPG with a quadratic critic (or an arbitrary critic approximated by the first two terms in the Taylor expansion) is equivalent to DPG with a Gaussian exploration policy where the covariance is computed as in Section \ref{ss-hessian-exp}. This follows from Corollary \ref{lem-agpg}. Third, and most generally, for any critic at all (not necessarily quadratic), DPG is a kind of EPG for a particular choice of quadrature (using a Dirac measure). This follows from Theorem \ref{th-gpgt}.

Surprisingly, this means that DPG, normally considered to be off-policy, can also be seen as on-policy when exploring with Gaussian noise defined as above for the quadratic critic or any noise for the linear critic. Furthermore, the compatible critic for DPG \citep{silver2014deterministic} is indeed linear in the actions. Hence, this relationship holds whenever DPG uses a compatible critic.\footnote{The notion of compatibility of a critic is different for stochastic and deterministic policy gradients.} Furthermore, Lemma \ref{lem-gpg} lends new legitimacy to the common practice of replacing the critic required by the DPG theory, which approximates $\nabla_a Q$, with one that approximates $Q$ itself, as done in SPG and EPG.

\subsection{Entropy-Based Methods}
On-policy SPG sometimes includes an entropy term \citep{peters2010relative} in the gradient in order to aid exploration by making the policy more stochastic. The gradient of the differential entropy $\mathcal{H}(s)$ of the policy at state $s$ is defined as\footnote{For discrete action spaces, the same derivation with integrals replaced by sums holds for the entropy.}
\begin{align*}
    \textstyle  -  \textstyle \nabla_\theta  \mathcal{H}(s) & =  \textstyle \nabla_\theta  \int_a d \pi(a \vert s) \log \pi (a \vert s) \\ & \textstyle = \textstyle \int_a  da \nabla_\theta  \pi(a \vert s) \log \pi (a \vert s) +  \int_a  d \pi(a \vert s) \nabla_\theta  \log \pi (a \vert s) \\
 & \textstyle = \textstyle \int_a da \nabla_\theta  \pi(a \vert s) \log \pi (a \vert s) +  \int_a  d \pi(a \vert s) \frac{1}{ \pi (a \vert s)} \nabla_\theta  \pi(a \vert s) \\ & \textstyle = \textstyle \int_a da \nabla_\theta  \pi(a \vert s) \log \pi (a \vert s) +  \nabla_\theta  \textstyle \underbrace{\textstyle \int_a  d \pi(a \vert s)}_{1} \\
 & \textstyle =  \textstyle \int_a  da \nabla_\theta  \pi(a \vert s) \log \pi (a \vert s) = \textstyle  \int_a  d \pi(a \vert s) \nabla_\theta  \log \pi(a \vert s) \log \pi (a \vert s).
\end{align*}
Typically, we add the entropy update to the policy gradient update with a weight $\alpha$, obtaining
\begin{align}
\label{entropy-critic}
I_G^E(s) \textstyle &= I_G(s) + \alpha \nabla_\theta  \mathcal{H}(s) \nonumber \\
&= \textstyle \int_a  d \pi(a \vert s) \nabla_\theta  \log \pi(a \vert s) (\Qv(a,s) - \alpha \log \pi (a \vert s)).
\end{align}
This equation makes clear that performing entropy regularization is equivalent to using a different critic with $Q$-values shifted by $\alpha \log \pi (a \vert s)$. This holds for both EPG and SPG, including SPG with discrete actions where the integral over actions is replaced with a sum. This follows because adding entropy regularization to the objective of optimizing the total discounted reward in an RL setting corresponds to shifting the reward function by a term proportional to $\log \pi (a \vert s)$ \citep{neu2017unified, nachum2017bridging}. Indeed, the \emph{path consistency learning algorithm} \citep{nachum2017bridging} contains a formula similar to \eqref{entropy-critic}, though we obtained ours independently.

Next, we derive a further specialization of \eqref{entropy-critic} for the case where the parameters $\theta$ are shared between the actor and the critic. We start with the policy gradient identity given by \eqref{epg-main-eq} and replace the true critic $Q$ with the approximate critic $\hat{Q}$. Since this holds for any stochastic policy, we choose one of the form
\begin{gather*}
\label{policy-exp}
\pi(a \vert s) = \frac{1}{Z(s)} e^{\hat{Q}(a, s)}, \quad \text{where} \quad Z(s) = \int_a e^{\hat{Q}(a, s)} da.
\end{gather*}
For the continuous case, we assume that the integral in \eqref{policy-exp} converges for each state. Here, we assume that the approximate critic is parameterized by $\theta$. Because of the form of \eqref{policy-exp}, the policy is parameterized by $\theta$ as well. Now, for the policy class given by \eqref{policy-exp}, we can simplify the gradient update even further, obtaining
\begin{align*}
I_G^E(s) &=
\textstyle \int_a  d \pi(a \vert s) \nabla_\theta  \log \pi(a \vert s) (\hat{Q}(a,s) - \alpha \log \pi (a \vert s)) \\
&= \textstyle \int_a  \pi(a \vert s) \nabla_\theta  \log \pi(a \vert s) (\hat{Q}(a,s) - \alpha \underbrace{\log e^{\hat{Q}(a, s)}}_{\hat{Q}(a,s)} - \alpha \log Z(s)) \\
 &= (1 - \alpha) \textstyle \int_a  \pi(a \vert s) \nabla_\theta  \log \pi(a \vert s) \hat{Q}(a,s) \\
&=  - (1 - \alpha) \nabla_\theta  \mathcal{H}(s).
\end{align*}
In the above derivation, we could drop the term $\log Z(s)$ since it does not depend on $a$, as with a baseline. This shows that, in the case of sharing parameters between the critic and the policy as above, methods such as A3C \citep{mnih2016asynchronous}, which have both an entropy loss and a policy gradient loss, are redundant since entropy regularization does nothing except scale the learning rate.\footnote{In this argument, we ignore the effects of sampling on exploration.} Alternatively, for this shared parameterization, a policy gradient method simply subtracts entropy from the policy. In practice, this means that a policy gradient method with this kind of parameter sharing is quite similar to learning the critic alone and simply acting according to the argmax of the $Q$ values rather than representing the policy explicitly, producing a method similar to sarsa.

Another family of entropy-based methods are soft actor-critic methods \citep{haarnojaSoftActorCriticOffPolicy2018}. They combine a policy learned with an entropy bonus similar to \eqref{entropy-critic}, but with $\alpha = 1$ and a different critic. In particular, for stochastic policies the gradient of the soft actor loss \citep[Equation 12]{haarnojaSoftActorCriticOffPolicy2018}, can be written as
\begin{align*}
    I_G^\text{soft} \textstyle &= \textstyle \int_a  d \pi(a \vert s) \nabla_\theta  \log \pi(a \vert s) (\hat{Q}^\text{soft}(a,s) - \log \pi (a \vert s)) = \\
    & = \nabla_\theta  \left[ \int_a d \pi(a \vert s)  \log \pi(a \vert s) \hat{Q}^\text{soft}(a,s)  -  \mathcal{H}(s) \right].
\end{align*}
Here, $\hat{Q}^\text{soft}$ is a separate critic, learned in a way analogous to the regular critic $\hat{Q}$, but based off a reward function $R(a,s)^\text{soft} = R(a,s) + \gamma \mathcal{H}(s)$, where the term $\gamma \mathcal{H}(s)$ is an intrinsic entropy bonus. The formula above assumes that $\hat{Q}^\text{soft}$ is a fixed learned approximation, i.e., it does not depend on the policy parameters. The rationale for using such a critic is largely orthogonal to this paper and given by \citet{haarnojaSoftActorCriticOffPolicy2018}. Transforming the integral above by a change of variables, we obtain the formula
\[
    I_G^\text{soft} \textstyle = \textstyle \int_\epsilon d N(\epsilon) \nabla_\theta  \log \pi(a \vert s) + (\nabla_a \log \pi(a \vert s) - \nabla_a Q^\text{soft}(a,s) ) \nabla_\theta  f(\epsilon, s).
\] 
Here, the reparameterization function $f$ is chosen such that the pdf of $\pi$ matches the pdf of $f(\epsilon)$ and $\epsilon \sim N$ is sampled from the standard normal. Soft actor-critic uses a one-sample Monte Carlo estimate of the integral above, which can be written as
\[
    I_G^\text{soft} \textstyle \approx \textstyle \nabla_\theta  \log \pi(a \vert s) + (\nabla_a \log \pi(a \vert s) - \nabla_a Q^\text{soft}(a,s) ) \nabla_\theta  f(\epsilon, s), 
\] 
where $\epsilon \sim N(\epsilon)$ and $a = f(\epsilon)$.

\subsection{Off-Policy Actor-Critic}
Off-policy learning with policy gradients typically follows the framework of \emph{off-policy actor-critic} \citep{degris2012off}. Denote the behavior policy as $b(a \mid s)$ and the corresponding discounted-ergodic measure as $\rho_b$. The method uses the reweighing approximation
\begin{align}
  \label{eq-offpac-rew}
  \nabla_\theta  J &= \int_s d \rho(s) \int_a d \pi(a \mid s) \nabla_\theta  \log \pi(a \mid s) (\Qv(a,s)) \nonumber \\
  &\approx \int_s d \rho_b(s) \int_a d \pi(a \mid s) \nabla_\theta  \log \pi(a \mid s) (\Qv(a,s)).
\end{align}
The approximation is necessary since, as the samples are generated using the policy $b$, it is not known how to approximate the integral with $\rho$ from samples, while it is easy to do so for an integral with $\rho_b$. A natural off-policy version of EPG emerges from this approximation (see Algorithm \ref{alg-op-epg}), which simply replaces the inner integral with $I^Q_\pi$, giving
\begin{gather*}
  \label{eq-offpac-epg}
\int_s d \rho_b(s) \int_a d \pi(a \mid s) \nabla_\theta  \log \pi(a \mid s) (\Qv(a,s)) = \int_s d \rho_b(s) I^Q_\pi(s).
\end{gather*}
Here, we use an analytic solution to $I^Q_\pi(s)$ as before. The importance sampling term $\frac{\pi(a \mid s)}{b(a\mid s)}$ does not appear because, as the integral is computed analytically, there is no sampling in $I^Q_\pi(s)$, much less sampling with an importance correction. Of course, the algorithm also requires an off-policy critic for which an importance sampling correction is typically necessary. Indeed, \eqref{eq-offpac-epg} makes clear that off-policy actor-critic differs from SPG in two places: the use of $\rho_b$ as in \eqref{eq-offpac-rew} and the use of an importance-sampled Monte Carlo estimator, rather than regular Monte Carlo, for the inner integral.

\begin{algorithm}[ht]
\begin{algorithmic}[1]
 \State $s \gets s_0$, $t \gets 0$
 \State initialize optimizer, initialize policy $\pi$ parameterized by $\theta$
\While{not converged}
\State $g_t \gets \gamma^t$ \textsc{do-integral}($\hat{Q}, s, \pi_\theta $) \Comment{$g_t$ is the estimated policy gradient as per \eqref{epg-samples-i}}
 \State $\theta \gets  \theta \; + \;  $optimizer.\textsc{update}$(g_t) $
 \State $a \sim b(\cdot, s)$
 \State $s',r \gets $ simulator.\textsc{perform-action}(a)
 \State $\hat{Q}$.\textsc{update}($s,a,r,s',\pi,b$) \Comment{Off-policy critic algorithm}
 \State $t \gets t + 1$
 \State $s \gets s'$
\EndWhile

\end{algorithmic}
\caption{Off-policy expected policy gradients with reweighing approximation}
\label{alg-op-epg}
\end{algorithm}


\subsection{Value Gradient Methods}
\emph{Value gradient methods} \citep{fairbank2014value, fairbank2012value, heess2015learning} assume the same parametrization of the policy as policy gradients, i.e., $\pi$ is parameterized by $\theta$, and maximize $J$ by recursively computing the gradient of the value function. In our notation, the policy gradient has the following connection with the value gradient of the initial state
\begin{gather*}
\label{eq-init-valgrad}
\nabla_\theta  J = \int_{s_0} dp_0(s_0) \nabla_\theta  \Vv(s)_0.
\end{gather*}
Value gradient methods use a recursive equation that computes $\nabla_\theta  \Vv(s)$ using $\nabla_\theta  \Vv(s)'$ where $s'$ is the successor state. In practice, this means that a trajectory is truncated and the computation goes backward from the last state all the way to $s_0$, where \eqref{eq-init-valgrad} is applied, so that the resulting estimate of $\nabla_\theta  J$ can be used to update the policy. The recursive formulae for $\nabla_\theta  \Vv(s)$ are based on the differentiated Bellman equation
\begin{gather*}
\label{eq-rec-valgrad}
\nabla_\theta  V = \nabla_\theta   \int_a d\pi(a \vert s) \left ( R(a, s)  + \gamma \int_{s'} p(s' \vert a,s ) \Vv(s)' \right).
\end{gather*}

Different value gradient methods differ in the form of the recursive update for the value gradient obtained from \eqref{eq-rec-valgrad}. For example, \emph{stochastic value gradients} (SVG) introduce a reparameterization both of $\pi$ and $p(s' \vert a,s )$ i.e.
\begin{gather*}
s' \sim p(\cdot \vert a,s ) \quad \Leftrightarrow \quad s' = f(a,s,\xi) \;\; \text{with} \;\; \xi \sim \mathcal{B}_1, \\
a \sim \pi(\cdot \vert s) \quad \Leftrightarrow \quad a = h(s,\eta) \;\; \text{with} \;\; \eta \sim \mathcal{B}_2.
\end{gather*}
Here, we denote the base noise distributions as $\mathcal{B}_1$ and $\mathcal{B}_2$, while $f$ and $h$ are deterministic functions. The function $f$ can be thought of as an MDP transition model. SVG rewrites \eqref{eq-rec-valgrad} using this reparametrization, to obtain

\begin{align}
\textstyle \nabla_\theta  V &= \textstyle \nabla_\theta   \int_\eta  d\mathcal{B}_2(\eta) \left ( R(s, h(s,\eta))  + \gamma  \int_{s'} d\mathcal{B}_1(\xi) V(f(h(s,\eta),s,\xi )) \right) = \nonumber \\
 &=  \textstyle \int_\eta  d\mathcal{B}_2(\eta) \left (\textstyle \nabla_\theta  R(s, h(s,\eta))  + \gamma  \int_{\xi} d\mathcal{B}_1(\xi) \nabla_\theta  V(\underbrace{f(h(s,\eta),s,\xi )}_{s'}) \right). \label{eq-svg-u}
\end{align}
Here, the quantities $\nabla_\theta  R(s, h(s,\eta))$ and $\nabla_\theta  V(f(h(s,\eta),s,\xi ))$ can be computed by the chain rule from the known reward model $R$ and a transition model $f$. SVG learns the approximate model $\hat{f}, \hat{R}, \hat{\xi}, \hat{\eta}$ from samples uses a sample-based approximation to \eqref{eq-svg-u} to obtain the value gradient recursion.

By contrast, we now derive a related but simpler value gradient method that does not require a model or a reparameterized policy,\footnote{SVG($\infty$) and SVG(1) require a model and a policy reparameterization while SVG(0) requires only a policy reparameterization. In fact, SVG(0) can be thought of as a direct analog of DPG or reparametrized gradient methods in the value gradient setting.} starting with \eqref{eq-rec-valgrad}. The value gradient is given by
\begin{align}
\nabla_\theta  \Vv(s) &= \textstyle \nabla_\theta   \int_a d\pi(a \vert s) \left ( R(a, s)  + \gamma \int_{s'} p(s' \vert a,s )  \right) \nonumber \\
&=  \textstyle \int_a da \nabla_\theta  \pi(a \vert s) R(a, s)  + \gamma \nabla_\theta  \pi(a \vert s) \left ( \int_{s'}  p(s' \vert a,s) \Vv(s)'  \right) \\  &\quad\quad\quad\quad\quad \quad\quad\quad\quad\quad \quad\quad\quad\quad\quad \textstyle + \pi(a \vert s) \nabla_\theta  \left( \int_{s'} p(s' \vert a,s ) \Vv(s)' \right) \nonumber \\
&= \textstyle \int_{a,s'} \pi(a \vert s) p(s' \vert a,s) \big( \nabla_\theta  \log \pi(a \vert s) R(a,s)  \\  &\quad\quad\quad\quad\quad \quad\quad\quad\quad\quad \quad\quad\quad\quad\quad \textstyle +  \nabla_\theta  \Vv(s)' + \nabla_\theta  \log \pi(a \vert s) \Vv(s)' \big).
\label{eq-cr-valgrad}
\end{align}
We can use random sampling to approximate \eqref{eq-cr-valgrad} 
\[
\hat{\nabla}_\theta   \Vv(s) \approx \left( \nabla_\theta  \log \pi(a \vert s) R(a,s)  +  \hat{ \nabla_\theta }   \Vv(s)' ) + \nabla_\theta  \log \pi(a \vert s) \hat{V}(s') \right).
\]
Here, the pair $a$ corresponds to the action taken at $s$ and $s'$ to the successor state. This method requires learning a critic, while SVG requires a model.

An  additional connection between value gradient methods and policy gradients is that, since the quantity $I_G(s)$ in Theorem \ref{th-gpgt} can be written as $I_G(s) = \nabla_\theta  \Vv(s) - \gamma \int_{s'} d p_\pi(s' \mid s) \nabla_\theta  \Vv(s)')$, we can think of this theorem as showing how to obtain a policy gradient from a value gradient without backwards iteration.

\subsection{Methods using Sums or Integrals over Actions}
\label{me-sum-i}
Several methods that involve summations over actions pre-date EPG. For discrete actions, it was first explicitly introduced in an unrefereed draft by \citet{sutton2000comparing} and independently developed by \citet{bahdanau2016actor}. The same idea was also independently developed as Mean Actor Critic \citep{2017arXiv170900503A}, concurrently with EPG. 
Our work differs from these other efforts in that we treat both continuous and discrete action spaces and analyze the algorithm in a rigorous theoretical framework. We believe this contribution is significant because, experimentally, the performance benefit of EPG occurs for systems with continuous action spaces. 

For continuous action spaces, the idea of numerical integration over actions is first mentioned by \citet{kakade2002natural}, where a continuous Linear-Quadratic-Gaussian (LQC) control with positive-definite quadratic critic is solved. However, this method is computationally inefficient.\footnote{\citet{kakade2002natural} uses a numerical quadrature ran to convergence to evaluate the policy gradient exactly. It does not describe the exact type of quadrature used.} \citet{petersPolicyGradientMethods2003} suggested extending  \citep{sutton2000comparing} to continuous action spaces for linear critics. They apply  quadrature to pre-compute a matrix, which can then be used to obtain the policy gradient by matrix-vector multiplication. However, their method crucially depends on using a linear critic. The idea of performing integration over actions is further explored in Interpolated Policy Gradients \citep{guInterpolatedPolicyGradient2017}. The IPG interpolation scheme between deterministic and stochastic gradients can be viewed as a specific type of Monte-Carlo quadrature that could be used in EPG. In addition, both the concept of the Normalized Advantage Function control variate \cite[Appendix, Section 11.3]{guInterpolatedPolicyGradient2017} and the quadratic assumption in MORE/MOTO \citep{abdolmalekiModelbasedRelativeEntropy2015, akrourModelFreeTrajectoryOptimization2016} can be viewed as the special case of EPG with quadratic critics where the Hessian is positive/negative definite.\footnote{MORE \citep{abdolmalekiModelbasedRelativeEntropy2015} does cover some cases of indefinite Hessians, but at the cost of heuristically setting the value of a Lagrange multiplier, which may lead to constraint violations.} Bayesian Actor-Critic \citep{ghavamzadeh2016bayesian} also employs a similar idea in the context of learning a critic as a Gaussian Process, where the kernel is chosen such that the mean estimate of the actor GP roughly corresponds to choosing a specific subset of quadratic critics with positive-definite Hessians. The advantage of EPG is that it works with \emph{any} Hessian, and we derive an explicit formula for the exploration covariance, for which indefinite Hessians are crucial for good performance. EPG also allows families of critics more general than quadratic functions. 

After EPG was originally published, \citet{nachumSmoothedActionValue2018} introduced a similar approach to policy gradients, called the smoothed action-value function. Since the authors present it as novel, we provide an explicit comparison here. First, the concept of the smoothed action function is, up to differentiation, essentially the same as our definition of the inner integral $I_G$; in particular, we have the identity
\[
    I^Q_\pi(s) = \nabla_\theta \tilde{Q}^\pi(s, \mu).
\]
Here, the expression to the right is the ``smoothed action value" \cite[Eq. 9]{nachumSmoothedActionValue2018}, while the term $I^Q_\pi(s)$ is defined in \eqref{epg-main-eq} of this article as well as in \cite[Equation 9]{epg-aaai}. \citet[Section 5]{nachumSmoothedActionValue2018} claim that they can handle any critic while avoiding quadrature, as opposed to EPG which requires either numerical quadrature or a tractable analytical form of the critic. \citet[Section 5]{nachumSmoothedActionValue2018} describe the process of what they call ``estimating an integral" as different from approximating an integral, which is what EPG with numerical quadrature does. However, the method they introduce for computing $\tilde{Q}$ given by \citet[Equation 19]{nachumSmoothedActionValue2018} in fact is a kind of numerical quadrature. Indeed, it requires computing averages over function evaluations at a number of what they call ``phantom actions", i.e., actions used only to compute the value of the smoothed action-value function. These phantom actions are normally called \emph{abscissae} in the numerical integration literature. Therefore, the only substantial difference between smoothed actor-critic \citep{nachumSmoothedActionValue2018} and our original work on EPG is the use of a different kind of quadrature and the fact that performing quadrature is integrated into the process of learning the critic -- \citet{nachumSmoothedActionValue2018} learn the \emph{integrated critic} directly, rather than learning a regular critic and then integrating over it.

\subsection{Methods Using the Geometry of the Policy Space}
Policy gradient methods can be improved by adjusting the policy update in a way that respects distances in the space of probability distributions \citep{kakade2002natural, amari1998natural, peters2008natural, furmston2012unifying, furmston2016approximate, schulman2015trust, NachumTrustpclOffpolicy}, leading to a \emph{natural gradient} method. Trust Region Policy Optimization \citep{schulman2015trust}, a representative recent method in this family, is based on an optimization problems of either the form
\begin{gather*}
\theta^\star = \argmax_{ \{ \theta: \; {\scriptscriptstyle \text{KL} \left( \pi_\theta, \pi_{\text{old}} \right) \, < \,  \delta } \} }  J_\theta,   \label{eq-trpo-constr} \quad \text{or} \\
\theta^\star = \argmax_{ \theta } \left( J_\theta  - \text{KL} \left( \pi_\theta, \pi_{\text{old}} \right) \right). \label{eq-trpo-sum}
\end{gather*}
Here, the total discounted return $J$ is defined as in \eqref{epg-main-eq} and $\pi_{\text{old}} $ is the policy from the previous time-step. \eqref{eq-trpo-constr} gives the version of TRPO as implemented, while \eqref{eq-trpo-sum} gives a version that comes form the theoretical analysis \citep{schulman2015trust}. TRPO relies on Monte-Carlo approximation to the policy gradient integral in the gradient of $J$, while performing analytic integration in the KL term. In principle, one can combine EPG with TRPO by computing \emph{both} integrals analytically. This is done by Model-Free Trajectory Optimization (MOTO) \citep{akrourModelFreeTrajectoryOptimization2016}, albeit only for critics that are both quadratic and negative-definite. Another example is the recent work of \citet{abdolmalekiMaximumPosterioriPolicy2018}, where the KL constraint is introduced using the RL-as-inference framework. Because the idea of using a KL constraint is orthogonal to the main thrust of this article, we leave the study of such hybrid algorithms to future work. 

In Section \ref{exp-ppo}, we instead compare EPG with PPO \citep{schulman2017proximal}, an existing established approximation to natural gradients with good empirical performance.

\section{Conclusions}
This paper proposed a new framework for reasoning about policy gradient methods called \emph{expected policy gradients} (EPG) that integrates across the action selected by the stochastic policy, thus reducing variance compared to existing stochastic policy gradient methods. We proved a new general policy gradient theorem subsuming the stochastic and deterministic policy gradient theorems, which covers any reasonable class of policies. We showed that analytical results for the policy update exist and, in the most common cases, lead to a practical algorithm (the analytic updates are summarized in Table \ref{tab-an-results}). We also gave universality results which state that, under certain broad conditions, the quadrature required by EPG can be performed analytically. For Gaussian policies, we also developed a novel approach to exploration that infers the exploration covariance from the Hessian of the critic. The analysis of EPG yielded new insights about DPG and delineated the links between the two methods. We have also described a version of EPG that works with discrete policies. Finally, we discussed the connections between EPG and other common RL techniques, notably sarsa, $Q$-learning, entropy regularization and methods taking account of the geometry of the policy space. We performed an extensive empirical evaluation of versions of EPG based on analytic and numerical quadrature.

\begin{table}[tb]
    \centering
    \small
        \begin{tabular}{ p{3.8cm} p{2cm} p{2.5cm} p{4.3cm} }
         \bf Policy Class & \bf Squashing & \bm{$\hat{Q}$} & \bf Analytic Update \\
         Normal, $a \in \mathbb{R}^d$ & none & $a^\top A_s a + a^\top B_s$ & \multirow{3}{*}{
         $\!\!\!\!\!\!\!\!\!\Bigg\}$
         \begin{minipage}{4.2cm} \tiny
         $ I^Q_{\pi(s), \mu_s} = (\nabla_\theta  \mu_s) (2 A_s \mu_s + B_s) $,\newline $I^Q_{\pi(s), \Sigma^{1/2}_s} = (\nabla_\theta  \Sigma^{1/2}_s)  A_s \Sigma^{1/2}_s$
         \end{minipage}
         } \\
         Logit-Normal; $a \in [0,1]^d$ & $a = \text{expit}(b)$ & $b^\top A_s b + b^\top B_s$ &  \\
         Log-Normal; $a \in [0,\infty]^d$ & $a = e^b$ & $b^\top A_s b + b^\top B_s$ &  \\
         any policy & none & $B_s^\top a$ & $I^Q_\pi(s) =  \nabla_\theta  \mu_{\pi(\cdot \mid s)}  B_s$ \\
        \end{tabular}
    \caption{A summary of the most useful analytic results for expected policy gradients. For bounded action spaces, we assume that the bounding interval is $[0,1]$ or $[0,\infty]$.}
    \label{tab-an-results}
\end{table}

\nocite{parisi2016multi}

\acks{This project has received funding from the European Research Council (ERC) under the European Union's Horizon 2020 research and innovation programme (grant agreement number 637713).  Experiments performed at Oxford were made possible by a generous equipment grant from NVIDIA. Moreover, we appreciate the comments from the JMLR reviewers, which led to improvements in our original submission. We also  thank Paavo Parmas and Rika Antonova for helpful feedback as well as Kaiqing Zhang and Zac Chen for pointing out typos in the original draft. }

\newpage

\appendix
\section{ }

\subsection{Proofs and Detailed Definitions}
\label{sec-proofs}
First, we prove two lemmas concerning the discounted-ergodic measure $\rho(s)$, of the Markov chain induced by fixing a policy in an MDP. They have been implicitly realized for some time but as far as we could find, never proved explicitly in RL literature.

\begin{definition}[ Time-dependent occupancy] The time-dependent state occupancy is defined as
\label{def-tdo}
\begin{gather*}
p(s \mid t=0 ) = p_0(s) \;\; \text{and}\\
p(s' \mid t=i+1 ) = \int_s p(s' \mid s) p(s \mid t=i ) \quad \text{for} \quad i \geq 0.
\end{gather*}
\end{definition}

\begin{definition}[Truncated trajectory]
We define a trajectory truncated after $N$ steps as $\tau_N = (s_0,a_0,r_0,s_1,a_1,r_1,\dots,s_N)$.
\end{definition}

\begin{observation}[Expectation with respect to truncated trajectory]
Since $\tau_N = (s_0,s_1,s_2,\dots,s_N)$ is associated with the density $ \prod_{i=0}^{N-1} p(s_{i+1} \mid s_i) p_0(s_0)$, we have
\begin{align*} \textstyle
& \textstyle \exs{\tau_N} {\sum_{i=0}^{N} \gamma^i f(s_i)} = \\
 & \quad = \textstyle \int_{s_0, s_1, \dots, s_N} \left (\prod_{i=0}^{N-1} p(s_{i+1} \mid s_i) \right) p_0(s_0) \left( \sum_{i=0}^{N} \gamma^i f(s_i) \right) ds_0 ds_1 \dots ds_N= \\
 & \quad= \textstyle \sum_{i=0}^{N} \int_{s_0, s_1, \dots, s_N} \left ( p_0(s_0) \prod_{i=0}^{N-1} p(s_{i+1} \mid s_i) \right) \gamma^i f(s_i) ds_0 ds_1 \dots ds_N = \\
 & \quad= \textstyle \sum_{i=0}^{N} \int_s p(s \mid t=i) \gamma^i f(s) ds
\end{align*}
for any function $f$.

\end{observation}

\begin{definition}[Expectation with respect to infinite trajectory]\ \\
For any bounded function $f$, we have
\begin{gather*}
\exs{\tau} {\sum_{i=0}^{\infty} \gamma^i f(s_i)} \triangleq \lim_{N \rightarrow \infty} \exs{\tau_N} {\sum_{i=0}^{N} \gamma^i f(s_i)}.
\end{gather*}
Here, the sum on the left-hand side is part of the symbol being defined.
\end{definition}

\begin{observation}[Property of expectation with respect to infinite trajectory]\ \\
We have 
\label{obs-eprop}
\begin{align*} \textstyle
\exs{\tau} {\sum_{i=0}^{\infty} \gamma^i f(s_i)}
 &= \textstyle \lim_{N \rightarrow \infty} \exs{\tau_N} {\sum_{i=0}^{N} \gamma^i f(s_i)} = \\
 &= \textstyle \lim_{N \rightarrow \infty} \sum_{i=0}^{N} \int_s p(s \mid t=i) \gamma^i f(s) d s  = \\
 &= \sum_{i=0}^\infty \int_s d p(s \mid t=i) \gamma^i f(s)
\end{align*}
for any bounded function $f$.
\end{observation}

\begin{definition}[Discounted-ergodic occupancy measure $\rho$]
\label{def-deom} We define the discounted-ergodic occupancy measure as
\[
\rho(s) = \sum_{i=0}^{\infty} \gamma^i p(s \mid t=i).
\]
\end{definition}
The measure $\rho$ is not normalized in general. Intuitively, it can be thought of as `marginalizing out' the time in the system dynamics.

\begin{lemma}[Discounted-ergodic property]
\label{lem-d-ergoidic}
For any bounded function $f$, we have
\[
\int_s \rho(s) f(s) = \exs{\tau}{\sum_{i=0}^{\infty}  \gamma^i f(s_i)}.
\]
\end{lemma}
\begin{proof}
We re-write the expectation as
\[
\exs{\tau}{\sum_{i=0}^{\infty} \gamma^i f(s_i) } = \sum_{i=0}^{\infty} \gamma^i \int_s p(s \mid t=i) f(s) ds = \int_s \underbrace{ \left[ \sum_{i=0}^{\infty} \gamma^i p(s \mid t=i) \right ]}_{\rho(s)} f(s) ds.
\]
Here, the first equality follows from Observation \ref{obs-eprop}.
\end{proof}
This property is useful since the expression on the left can be easily manipulated while the expression on the right can be estimated from samples using Monte Carlo.

\begin{lemma}[Generalized eigenfunction property]
\label{gep-lemma}
For any bounded function $f$, we have
\[
\gamma \int_s d \rho(s) \int_{s'} d p(s' \mid s) f(s') = \left( \int_s d \rho(s) f(s) \right) - \left ( \int_s d p_0(s) f(s) \right).
\]
\end{lemma}
\begin{proof}
We rewrite the expression in the left as
\begin{align*}
    \textstyle
\gamma \int_s d\rho(s) \int_{s'} d p(s' \mid s) f(s') &= \textstyle \gamma \sum_{i=0}^{\infty} \gamma^i \int_{s,s'} p(s \mid t=i) p(s' \mid s) f(s') ds ds'= \\
&= \textstyle \sum_{i=0}^{\infty} \gamma^{i+1} \int_{s'} dp(s' \mid t=i+1) f(s') \\
&= \textstyle \sum_{i=1}^{\infty} \gamma^{i} \int_{s'} dp(s' \mid t=i) f(s') \\
&= \textstyle \left ( \sum_{i=0}^{\infty} \gamma^{i} \int_{s'} dp(s' \mid t=i) f(s') \right) - \left ( \int_s dp_0(s) f(s) \right) \\
&= \textstyle \left( \int_s d\rho(s) f(s) \right) - \left ( \int_s dp_0(s) f(s) \right).
\end{align*}
Here, the first equality follows form definition \ref{def-deom}, the second one from definition \ref{def-tdo}. The last equality follows again from definition \ref{def-deom}.
\end{proof}

\begin{definition}[Markov Reward Process]\label{def-MRP}\ \\
A Markov Reward Process is a tuple $(p,p_0,R,\gamma)$, where $p(s' \vert s)$ is a transition kernel, $p_0$ is the distribution over initial states, $R(\cdot \vert s)$ is a reward distribution conditioned on the state and $\gamma$ is the discount constant.
\end{definition}

An MRP can be thought of as an MDP with a fixed policy and dynamics given by marginalizing out the actions $p_\pi(s' \mid s) = \int_a d \pi (a \mid s) p(s' \mid a, s)$. Since this paper considers the case of one policy, we abuse notation slightly by using the same symbol $\tau$ to denote trajectories including actions, i.e. $(s_0,a_0,r_0,s_1,a_1,r_1,\dots)$ and without them $(s_0,r_0,s_1,r_1,\dots)$.

\begin{lemma}[Second Moment Bellman Equation]
\label{smbe-lemma}
Consider a Markov Reward Process $(p,p_0,X,\gamma)$ where $p(s' \mid s)$ is a Markov process and $X( \cdot \mid s)$ is some probability density function.\footnote{Note that while $X$ occupies a place in the definition of the MRP usually called `reward distribution', we are using the symbol $X$, not $R$ since we shall apply the lemma to $X$es which are constructions distinct from the reward of the MDP we are solving.}  Denote the value function of the MRP as $V$. Denote the second moment function $S$ as
\[
S(s) = \excs{\tau}{ \left( \sum_{t=0}^{\infty} \gamma^t x_t \right)^2 }{s_0 = s}, \quad \text{where} \quad x_t \sim X(\cdot \mid s_t).
\]
Then $S$ is the value function of the MRP: $(p, p_0, u, \gamma^2)$, where $u(s)$ is a deterministic random variable given by
\[
u(s) = \vars{X( x \mid s)}{x} + \left(\exs{X( x \mid s)}{x}\right)^2 + 2 \gamma \exs{X( x \mid s)}{x} \exs{p( s' \mid s)}{\Vv(s)')} .
\]
\end{lemma}
\begin{proof}
We rewrite $S(s)$ as
\begin{align*} \textstyle
S(s) &= \textstyle \excs{\tau}{ \left( x_0 + \sum_{t=1}^{\infty} \gamma^t x_t \right)^2 }{s_0 = s} \\
&= \textstyle \excs{\tau}{ x_0^2 + 2 x_0 \left( \sum_{t=1}^{\infty} \gamma^t x_t \right) +  \left( \sum_{t=1}^{\infty} \gamma^t x_t \right)^2}{s_0 = s} \\
&= \textstyle \underbrace{ \textstyle \excs{\tau}{ x_0^2 }{s_0 = s} + \excs{\tau}{ 2 x_0 \left( \sum_{t=1}^{\infty} \gamma^t x_t \right)}{s_0 = s}}_{u(s)}
+ \underbrace{\excs{\tau}{ \textstyle \left( \sum_{t=1}^{\infty} \gamma^t x_t \right)^2}{s_0 = s}}_{\gamma^2  \exs{p( s' \mid s)}{S(s')}}.
\end{align*}
This is exactly the Bellman equation of the MRP $(p, p_0, u, \gamma^2)$. The theorem follows since the Bellman equation uniquely determines the value function.
\end{proof}

\begin{observation}[Dominated Value Functions]
\label{dom-o}\ \\
Consider two Markov Reward Processes $(p,p_0,X_1,\gamma)$ and $(p,p_0,X_2,\gamma)$, where $p(s' \mid s)$ is a Markov process (common to both MRPs) and $X_1(s)$, $X_2(s)$ are some deterministic random variables meeting the condition $X_1(s) \leq X_2(s)$ for every $s$. Then the value functions $V_1$ and $V_2$ of the respective MRPs satisfy $V_1(s) \leq V_2(s)$ for every $s$. Moreover, if we have that $X_1(s) < X_2(s)$ for all states, then the inequality between value functions is strict.
\end{observation}
\begin{proof}
Follows trivially by expanding the value function as a series and comparing series elementwise.
\end{proof}

\subsection{Proofs about Exponential Families}
\epgexpfamily*

    \begin{proof}
    We first rewrite the inner integral as an expectation, obtaining
    \begin{align*}
      I^{{\hat{Q}}}_\pi(s) &= \int_A d \pi(a \mid s) \nabla_\theta  \log \pi(a \mid s) {\hat{Q}}(a,s) \\
      &= \exs{a \sim \pi}{\nabla_\theta  \log \pi(a \mid s) {\hat{Q}}(a,s)} \\
      &= \exs{a \sim \pi}{ (\nabla_\theta  (\eta_\theta^\top T^s(a) - U^s_{\eta_\theta} + W^s(a)) ) {\hat{Q}}(a,s)} \\
      &= \exs{a \sim \pi}{ (\nabla_\theta  \eta_\theta^\top) T^s(a) {\hat{Q}}(a,s) - (\nabla_\theta  U^s_{\eta_\theta}) {\hat{Q}}(a,s)} \\
      &= (\nabla_\theta  \eta_\theta^\top)  \exs{a \sim \pi}{ T^s(a) {\hat{Q}}(a,s)} - (\nabla_\theta  U^s_{\eta_\theta}) \exs{a \sim \pi}{ {\hat{Q}}(a,s)}.
    \end{align*}
    Since $T^s(a)$ and ${\hat{Q}}(a,s)$ are polynomials, and the multiplication of polynomials is still polynomial, both expectations are expectations of polynomials. To compute the second expectation, we exploit the fact that, since ${\hat{Q}}(a,s)$ is a polynomial on $a$, it equals a sum of monomial terms
    \[
    \exs{a \sim \pi}{ {\hat{Q}}(a,s)} =
    \exs{a \sim \pi}{ \sum_{i=1}^D c_i \prod_{j=1}^d a_j^{p_i(j)}} =
    \sum_{i=1}^D c_i \underbrace{ \exs{a \sim \pi}{ \prod_{j=1}^d a_j^{p_i(j)} }}_{\text{cross-moment of} \;\; \pi}.
    \]
    On the right, the terms $\exs{a \sim \pi}{ \prod_{j=1}^d a_j^{p_i(j)} }$ (we do not explicitly denote the dependence on $s$ for clarity), are the uncentered $(p_i(1),p_i(2),\dots,p_i(d))$-cross-moments of $\pi$. If we arrange the coefficients $c_i$ into the vector $C^s_Q$ and the cross-moments into the vector $m^s_\pi$, we obtain the right term in \eqref{poly-epg-cf}. We can apply the same reasoning to the product of $T^s$ and ${\hat{Q}}(\cdot,s)$ to obtain the left term.
    \end{proof}

\epgexpfamilyrepar*
\begin{proof}
We rewrite the policy gradient update
    \begin{align*}
      I^{Q}_\pi(s) & = \int_A d \pi(a \mid s) \nabla_\theta  \log \pi(a \mid s) {\hat{Q}}(a, s)\\
      & = \int_{\mathbb{R}^d} d \pi(g(b) \mid s) \nabla_\theta  \log \pi(g(b) \mid s) {\hat{Q}}(g(b), s) \det \jacobian  g(b) \\
      & = \int_{\mathbb{R}^d} d \pi_b(b \mid s) \nabla_\theta  \log \pi(g(b) \mid s) {\hat{Q}}_b(b, s) \\
      & = \int_{\mathbb{R}^d} d \pi_b(b \mid s) (\nabla_\theta  \log \pi_b(b \mid s) - \underbrace{\nabla_\theta  \log \det \jacobian  g(b)}_{0}) {\hat{Q}}_b(b, s) = I^{{\hat{Q}}_b}_{\pi_b}(s).
    \end{align*}
    In the second equality, we perform the variable substitution $a = g(b)$. In the third equality we use \eqref{eq-repar-pdf} and the fact that ${\hat{Q}}_b(g^{-1}(a), s) = {\hat{Q}}(a, s)$. In the fourth equality we again use \eqref{eq-repar-pdf} and the fact that $\log \det \jacobian  g(b) = 0$ since $g$ is not parameterized by $\theta$.
    \end{proof}

\subsection{Computation of Moments for an Exponential Family}
\label{ef-moments}
Consider the moment generating function of $T^s(a)$, which we denote as $M_T$, for the exponential family of the form given in equation \eqref{pc-exp-fam-c}, that is
\[
M_T(v) = e^{U^s_{v + \est} - U^s_{\est}}.
\]
It is well-known that $M_T$ is finite in a neighborhood of the origin \citep{Bickel:2006aa}, and hence the cross moments can be obtained as
\[
 \exs{a \sim \pi}{ \prod_{j=1}^{K} {T^s(a)}_j^{p(j)} } =
 \left. \frac{\partial} { \partial^{p(1)} v_1, \partial^{p(2)} v_2, \dots, \partial^{p(K)} v_K } M_T(v) \right | _{v =  0 }.
\]
Here, we denoted as $K$ the size of the sufficient statistic (i.e. the length of the vector $T^s(a)$). However, we seek the cross-moments of $a$, not $T^s(a)$. If $T^s(a)$ contains a subset of indices which correspond to the vector $a$, then we can simply use the corresponding indices in the above equation. On the other hand, if this is not the case, we can introduce an extended distribution $\pi'( a \mid s) = e^{{(\eta)'}^\top T'(a) - U^s_{\est} + W^s(a)}$, where $T'$ is a vector concatenation of $T^s$ and $a$. We can then use the MGF of $T'(a)$, restricted to a suitable set of indices to get the moments.

\subsection{Experimental Details}
\label{sec-exp-details}
The exploration hyperparameters for EPG were $\sigma_0 = 0.5$ and $c = 1.0$ where the exploration covariance is $\sigma_0 e^{cH}$. These values were obtained using a grid search from the set $\{0.2,0.5,1\}$ for $\sigma_0$ over the HalfCheetah-v2 domain (see Figure \ref{fig-hyperparameters-EPG}). The remaining  parameters were based on previous attempts to obtain good sample efficiency for deterministic policy gradients \citep{islam2017reproducibility, brockman2016openai}. We provide a full list in Table \ref{table-hyperparamters}.

\begin{table}[ht]
    \small
    \begin{tabular}{ll}
    \multicolumn{2}{l}{\emph{Shared parameters}}                                                                                             \\
    Target network update constant $\tau$                               & 0.01                                                        \\
    Size of replay buffer                                              & 1000000                                                     \\
    Method of sampling from buffer                                     & uniform                                                     \\
    Ignored steps at beginning of training                   & 10000                                                       \\
    Reward scale for \\ \hspace{1cm} InvertedPendulum-v2 and \\ \hspace{1cm} InvertedDoublePendulum-v2 & 0.1                                                         \\
    Reward scale for other tasks                                       & 1 (no scaling)                                              \\
    Optimiser used for both actor and critic                          & Adam                                                        \\
    Learning rate                                                      & 1e-3                                                        \\
    Batch size                                                        & 64                                                        \\
    Structure of critic network                                        & hidden layers of 100, 100 neurons respectively,  \\
                                                                        & \hspace{1cm} ReLU nonlinearities       \\
    Structure of actor network                                         & hidden layers of 100, 50, 25 neurons respectively,      \\
                                                                        & \hspace{1cm} ReLU nonlinearities       \\
    Target network update constant $\tau$                               & 0.01                                                        \\ 
    \\
    \multicolumn{2}{l}{\emph{Experiment-specific parameters}}                                                                         \\
    DPG                               & $\sigma = 0.2, \psi=0.15$                                                        \\
    EPG                               & $\sigma_0 = 0.5$                                                        \\ 
    SPG                               & $\sigma = 0.2$, fixed                                                        \\ 
    NQ($m$)-EPG                       & $\sigma = 0.2, \psi=0.15$ (OU exploration)
    \end{tabular}
    \caption{List of hyperparameters}
    \label{table-hyperparamters}
\end{table}

The experiments described here extend our previous conference work \citep{epg-aaai}. The experiments here are not directly comparable because they were performed on different versions of MuJoCo environment (version 2 instead of 1) and using PyTorch \citep{paszkeAutomaticDifferentiationPyTorch2017} rather than TensorFlow \citep{abadiTensorFlowLargescaleMachine2015}, which have minor differences in the implementation of optimisation algorithms. Despite these differences, our new results are very similar.

For the PPO algorithm, we simply ran the PPO2 version published by OpenAI \citep{baselines} with its default parameters. For the experiments with discrete action spaces we used the hyperparameters of the A2C algorithm \citep{baselines}. 

\begin{figure}[t]
    \centering
        \captionsetup[subfigure]{aboveskip=0.5em,belowskip=0.5em}
        \begin{subfigure}{0.42\textwidth}
            \caption{HalfCheetah-v2}
            \includegraphics[width=\textwidth]{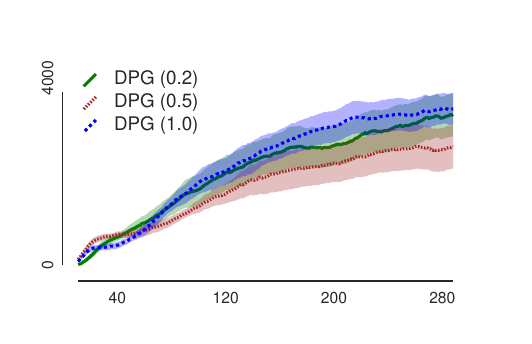}
        \end{subfigure}
        ~
        \vspace{1cm}
        \begin{subfigure}{0.42\textwidth}
            \caption{InvertedDoublePendulum-v2}
            \includegraphics[width=\textwidth]{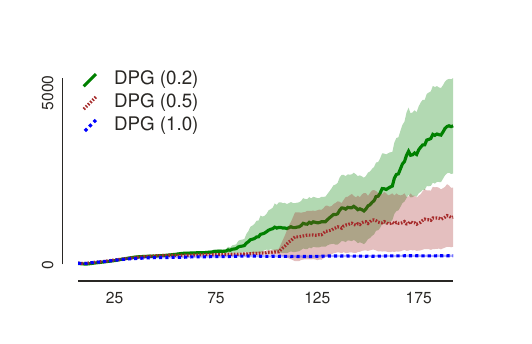}
        \end{subfigure}
        \caption{Learning curves (mean and 90\% interval) showing the performance of Deterministic Policy Gradients for different values of the exploration noise.  All results are obtained from 20 runs. Horizontal axis shows thousands of steps.}
        \label{fig-hyperparameters-DPG}
    \end{figure}

\subsection{Parameter Tuning for Deterministic Policy Gradients}
The original paper the combined Deterministic Policy Gradients with deep networks used $\sigma=0.2$ \citep{lillicrap2015continuous}. To ensure a fair comparison, we tested the algorithm using three different settings for the exploration noise $\sigma$: 0.2, 0.5 and 1. The parameter for the memory of the OU noise was set to $\frac{15}{20}\sigma$, following \citet{lillicrap2015continuous}. On the HalfCheetah domain (Figure \ref{fig-hyperparameters-DPG}), the performances were comparable, with the value 1 leading to the best performance (although the difference wasn't statistically significant). We initially wanted to use that value. However, exploring with such a large noise leads to catastrophic failure in the InvertedDoublePendulum task. We therefore settled on the value of 0.2, which performed reasonably on both tasks and was already widely used \citep{lillicrap2015continuous}. 

\vskip 0.2in
\bibliography{epg-journal-nonote}

\end{document}